\documentclass[letterpaper]{article} % DO NOT CHANGE THIS
\usepackage[preprint]{aaai2027}  % DO NOT CHANGE THIS
\usepackage[hyphens]{url}  % DO NOT CHANGE THIS
\usepackage{graphicx} % DO NOT CHANGE THIS
\usepackage{natbib}  % DO NOT CHANGE THIS AND DO NOT ADD ANY OPTIONS TO IT
\usepackage{caption} % DO NOT CHANGE THIS AND DO NOT ADD ANY OPTIONS TO IT
\usepackage[ruled,vlined,linesnumbered]{algorithm2e}
\usepackage{amsthm}
\usepackage{amsmath,amssymb}
\usepackage{graphicx}
\usepackage{algorithmic}
\usepackage{tikz-cd}
\newtheorem{Proposition}{Proposition}
\usepackage[table]{xcolor}
\usepackage{multicol}
\usetikzlibrary{arrows.meta,calc,fit,backgrounds,shapes.geometric}

\usepackage{algorithm}
\usepackage{algorithmic}
\usepackage{algorithm}
\usepackage{algorithmic}
\usepackage{enumitem}
\usepackage{overpic}
\usepackage{subcaption}
\usepackage{amsmath,amssymb}
\usepackage{comment}
\usepackage{tikz-cd}
\usepackage{multicol}
\usetikzlibrary{arrows.meta,calc,fit,backgrounds,shapes.geometric}

\usepackage{newfloat}
\usepackage{amsfonts}
\usepackage{listings}
\DeclareCaptionStyle{ruled}{labelfont=normalfont,labelsep=colon,strut=off} % DO NOT CHANGE THIS
\floatstyle{ruled}
\newfloat{listing}{tb}{lst}{}
\floatname{listing}{Listing}

\usepackage{booktabs}

\title{Bootstrap Stochastic Interpolant Tree for  Inference-time  \\ Search Under Sparse Reward}
\author {
    Binglin Ji\textsuperscript{\rm 1}\equalcontrib,
    Anindya Sarkar\textsuperscript{\rm 1}\equalcontrib,
    Yevgeniy Vorobeychik\textsuperscript{\rm 1}
}
\affiliations {
    \textsuperscript{\rm 1}Department of Computer Science \& Engineering, Washington University in St. Louis, USA\\
    \{binglin.j,~anindya,~yvorobeychik\}@wustl.edu
}

\begin{document}

%\begin{comment}
\maketitle

\begin{abstract}
Tree-based samplers offer a compelling alternative for scalable inference-time search, bypassing the per-step reward and gradient dependencies that introduce systematic bias into SMC-style frameworks. Despite this advantage, contemporary tree samplers fail to achieve broad exploration. 
Their reliance on learned, DDPM-like transitions inherently couples neighbor states in a trajectory, thereby restricting child-parent feature divergence in the tree sampler.
%DDPM-like transitions couples child nodes to their parents, thereby restricting child-parent feature divergence.
%Because they rely on learned, DDPM-like transitions for child node sampling, child nodes are tightly coupled to their parents, restricting child-parent feature divergence. 
%Moreover, the denoiser's inductive bias stifles exploratory diversity. 
Consequently, these samplers suffer from restricted coverage of the target marginal distribution, severely hindering the exploration critical for efficient search under sparse reward. 
To overcome these limitations, we introduce a simple yet effective \textbf{B}ootstrap \textbf{S}tochastic \textbf{I}nterpolant \textbf{T}ree (BSIT) sampler that decouples child node generation from explicit parent dependencies. By expanding the degrees of freedom during child node sampling, BSIT encourages diverse trajectory rollouts and unlocks broader path-space diversity—the exact catalyst needed for efficient exploration.
%By anchoring the terminal node via a Flow-map trajectory, we synthesize intermediate nodes through stochastic interpolation between the endpoint and a random source sample. 
%Crucially, we show that our sampler converges to the true target distribution faster than sampler relying on ddpm-like transition for child node sampling due to its enhance exploration capability, while naturally broadening its coverage to enhance sample diversity under a fixed sampling budget, without requiring additional function evaluations, thus this framework establishes a highly efficient foundation for online, reward-driven search. 
%Crucially, we demonstrate that the enhanced exploration capability of the sampler accelerates convergence to the target distribution with enhanced distribution coverage that maximizes sample diversity compared to samplers relying on DDPM-like transitions for child node generation-all without requiring extra function evaluations or reward feedback—establishing a highly efficient foundation for online, reward-driven search.
Crucially, we show that the BSIT’s improved exploration accelerates convergence to the target distribution while achieving broader distributional coverage than samplers that rely on DDPM-like transitions for child-node generation, all without incurring additional function evaluations or reward feedback, thereby providing a highly efficient basis for rare target search.
Extensive experiments across diverse search settings, including rare-species discovery and target-aligned molecule generation tasks, demonstrate BSIT's superior search efficiency over the baselines.

\end{abstract}


\section{Introduction}
Diffusion models, flow matching, and stochastic interpolants have emerged as the most powerful frameworks for generative modeling~\citep{song2020score,lipman2022flow,albergo2022building},
achieving state-of-the-art results across a wide range of modalities, including image synthesis~\citep{rombach2022high,ho2020denoising}, molecular conformer generation~\citep{hoogeboom2022equivariant,xu2022geodiff}, and text generation~\citep{lou2023discrete}. A growing frontier in this space is adapting these models to align with a reward function to enforce desirable sample properties~\citep{bansal2024universal,kim2022refining,ma2025inference}. %However, a fundamental bottleneck in many applications is the reward sparsity and the prohibitive cost of acquiring ground-truth reward feedback. Whether it's the rare species discovery task or generating image with a rare attribute, true reward signals are expensive to evaluate and sparse. 
%However, real-world deployment faces a fundamental bottleneck: ground-truth reward feedback is typically sparse and expensive to evaluate. In domains like rare species discovery or rare-attribute image generation, evaluating true reward signals incurs prohibitive costs under extreme sparsity.
A central challenge in many applications stems from extreme reward sparsity combined with the high cost of acquiring ground-truth feedback. In tasks such as rare species discovery or generating images with rare attributes, true reward signals remain both scarce and expensive to compute.
Consequently, constructing samplers capable of yielding high-quality, target-aligned samples while radically reducing reliance on ground-truth reward feedback is essential. The cornerstone of an effective, feedback-efficient sampler lies in its ability to conduct highly informed and strategically targeted exploration. 

While numerous feedback-efficient search strategies have been introduced~\citep{uehara2024feedback,kim2025test,jain2026diffusion,ji2026bootstrap}, they tend to underperform precisely in regimes that demand broad exploration—specifically when reward signals are extremely sparse and subject to a strict evaluation budget. %In such settings, the sampler must conduct robust exploration to discover high-utility regions while operating under a tightly constrained feedback budget, a requirement that existing methods often fail to meet. 
For instance,~\citep{uehara2024feedback} employs online RL fine-tuning of diffusion models to approximate a target KL divergence minimization. This inherently mode-seeking objective fails to capture the full diversity of high-utility regions within the target distribution. %Furthermore, because these frameworks rely on an online-trained reward model, systematic bias in the early phases of training can misdirect the fine-tuning process and permanently degrade sample quality.
Recent Sequential Monte Carlo–based inference-time scaling methods~\citep{kim2025test,singhal2025general,skreta2025feynman} circumvent direct model fine-tuning by operating on fixed generative priors. However, these approaches suffer from severe weight degeneracy, which collapses the proposal distribution, cripples feedback efficiency, and severely suppresses sample diversity~\citep{lee2025debiasing}. Furthermore, their reliance on reward values or gradients at every resampling step poses a major bottleneck. In practice, these signals are unavailable, and relying on approximations introduces systematic bias that ultimately degrades the sample quality.

\begin{figure*}[t]
\centering
\resizebox{\textwidth}{!}{%
\begin{tikzpicture}[
    x=0.86cm,y=0.86cm,
    >=Latex,
    font=\large,
    panel/.style={draw=black!12, rounded corners=2mm, fill=black!1, line width=0.8pt},
    state/.style={circle, draw=black!35, fill=black!5, minimum size=8.5mm, inner sep=0pt, font=\bfseries\normalsize},
    base/.style={state, draw=green!55!black, fill=green!9, text=green!45!black},
    bnode/.style={state, draw=teal!55!black, fill=teal!10, text=teal!60!black},
    noise/.style={state, draw=purple!45!black, fill=purple!8, text=purple!65!black},
    card/.style={draw=black!14, fill=white, rounded corners=1.6mm, align=left, inner sep=4pt, font=\normalsize},
    title/.style={font=\bfseries\Large, text=black!78},
    note/.style={font=\small, text=black!58},
    advantage/.style={draw=orange!75!black, fill=orange!12, rounded corners=1.6mm, line width=0.8pt, align=left, inner sep=4pt, text=orange!35!black, font=\small},
    arr/.style={-{Latex[length=2.5mm]}, line width=1pt},
    bflow/.style={arr, teal!65!black},
    pflow/.style={arr, violet!70!black},
    guide/.style={draw=black!18, rounded corners=10mm, line width=0.9pt},
    treeedge/.style={draw=blue!18, line width=0.75pt},
    treenode/.style={circle, draw=blue!20, fill=blue!5, minimum size=6mm, inner sep=0pt},
    trajpipe/.style={orange!70!black, line width=16pt, opacity=0.28, line cap=round, line join=round},
    trajpipeedge/.style={orange!80!black, line width=1.1pt, line cap=round, line join=round}
]

\begin{scope}[xshift=-1.15cm]
\node[title] at (10.55,6.72) {BSIT};
\node[state] (xt) at (7.05,5.50) {$x_t$};
\node[base] (xbase) at (11.00,-1.35) {$x_0$};
\node[noise] (eps1) at (9.05,5.85) {$\varepsilon_1$};
\node[noise] (eps2) at (10.95,5.85) {$\varepsilon_2$};
\node[noise, minimum size=6.5mm] (epsd) at (12.75,5.85) {$\cdots$};
\node[noise] (epsm) at (14.45,5.85) {$\varepsilon_m$};
\node[bnode] (c1) at (8.20,4.85) {$x_{t'_1}$};
\node[bnode] (c2) at (9.45,3.55) {$x_{t'_2}$};
\node[bnode, minimum size=6.5mm] (cd) at (10.75,2.10) {$\cdots$};
\node[bnode] (cm) at (12.45,-0.10) {$x_{t'_m}$};

\begin{scope}[on background layer]
% Faint branches indicate that the highlighted trajectory belongs to a
% larger search tree rooted at x_t.
\node[treenode] (t11) at (5.75,4.55) {};
\node[treenode] (t12) at (6.55,4.05) {};
\node[treenode] (t14) at (7.65,4.15) {};
\node[treenode] (t21) at (5.20,3.15) {};
\node[treenode] (t22) at (6.05,2.75) {};
\node[treenode] (t23) at (6.85,2.55) {};
\node[treenode] (t31) at (5.55,1.45) {};
\node[treenode] (t32) at (6.50,1.15) {};
\node[treenode] (t33) at (7.45,1.35) {};

\node[treenode] (t24) at (8.20,2.55) {};
\node[treenode] (t34) at (8.35,0.95) {};
\node[treenode] (t35) at (9.30,0.70) {};

% Final leaf row aligned vertically with x_0.
\node[treenode] (leaf1) at (5.15,-1.35) {};
\node[treenode] (leaf2) at (6.15,-1.35) {};
\node[treenode] (leaf3) at (7.15,-1.35) {};
\node[treenode] (leaf4) at (8.05,-1.35) {};
\node[treenode] (leaf5) at (8.95,-1.35) {};

\draw[treeedge] (xt) -- (t11);
\draw[treeedge] (xt) -- (t12);
\draw[treeedge] (t11) -- (t21);
\draw[treeedge] (t11) -- (t22);
\draw[treeedge] (t12) -- (t23);
\draw[treeedge] (t21) -- (t31);
\draw[treeedge] (t22) -- (t32);
\draw[treeedge] (t23) -- (t33);
\draw[treeedge] (xt) -- (t14);
\draw[treeedge] (t14) -- (t24);
\draw[treeedge] (t24) -- (t34);
\draw[treeedge] (t24) -- (t35);
\draw[treeedge] (t31) -- (leaf1);
\draw[treeedge] (t32) -- (leaf2);
\draw[treeedge] (t33) -- (leaf3);
\draw[treeedge] (t34) -- (leaf4);
\draw[treeedge] (t35) -- (leaf5);
\draw[trajpipe] (xt) -- (c1) -- (c2) -- (cd) -- (cm) -- (xbase);
\draw[trajpipeedge] (xt) -- (c1) -- (c2) -- (cd) -- (cm) -- (xbase);
\end{scope}

\draw[pflow] (xt) .. controls (6.15,3.25) and (7.95,-0.85) .. node[pos=0.45, left=2pt, text=violet!70!black, font=\bfseries\small] {1 NFE} (xbase);
\draw[bflow] (xbase) -- (c1);
\draw[bflow] (xbase) -- (c2);
\draw[bflow, dashed] (xbase) -- (cd);
\draw[bflow] (xbase) -- (cm);
\draw[bflow] (eps1) -- (c1);
\draw[bflow] (eps2) -- (c2);
\draw[bflow, dashed] (epsd) -- (cd);
\draw[bflow] (epsm) -- (cm);
\node[note, draw=teal!55!black, fill=teal!6, rounded corners=1.2mm, inner sep=2.2pt, text=teal!65!black] at (11.95,4.05) {$x_{t'_i}=I_{t'_i}(x_0,\varepsilon_i)$};
\end{scope}

\begin{scope}[xshift=12.35cm, yshift=0.15cm]
\node[font=\bfseries\Large, text=black!78] at (6.10,6.20) {Comparison with Current Tree Samplers.};
\draw[rounded corners=1.2mm, line width=0.8pt, draw=black!35, fill=black!1] (0,1.20) rectangle (12.20,5.75);
\foreach \x in {2.44,4.88,7.32,9.76} {\draw[black!28, line width=0.55pt] (\x,1.20) -- (\x,5.75);}
\foreach \y in {2.40,3.60,4.55} {\draw[black!28, line width=0.55pt] (0,\y) -- (12.20,\y);}
\node[font=\bfseries\small, align=center, text=black!72] at (1.22,5.13) {Method};
\node[font=\bfseries\scriptsize, align=center, text=black!72] at (3.66,5.13) {No Reward\\Gradient};
\node[font=\bfseries\scriptsize, align=center, text=black!72] at (6.10,5.13) {Dynamic\\Transition\\Schedule};
\node[font=\bfseries\scriptsize, align=center, text=black!72] at (8.54,5.13) {NFE\\Efficient};
\node[font=\bfseries\scriptsize, align=center, text=black!72] at (10.98,5.13) {No Parent State\\Dependency};
\node[font=\bfseries\normalsize, text=black!75] at (1.22,4.08) {DTS};
\node[font=\bfseries\normalsize, text=black!75] at (1.22,2.90) {BFMT};
\node[font=\bfseries\normalsize, text=black!75] at (1.22,1.70) {BSIT};
\node[font=\LARGE, text=green!55!black] at (3.66,4.08) {$\checkmark$};
\node[font=\LARGE, text=red!70!black] at (6.10,4.08) {$\times$};
\node[font=\LARGE, text=red!70!black] at (8.54,4.08) {$\times$};
\node[font=\LARGE, text=red!70!black] at (10.98,4.08) {$\times$};
\node[font=\LARGE, text=green!55!black] at (3.66,2.90) {$\checkmark$};
\node[font=\LARGE, text=green!55!black] at (6.10,2.90) {$\checkmark$};
\node[font=\LARGE, text=green!55!black] at (8.54,2.90) {$\checkmark$};
\node[font=\LARGE, text=red!70!black] at (10.98,2.90) {$\times$};
\node[font=\LARGE, text=green!55!black] at (3.66,1.70) {$\checkmark$};
\node[font=\LARGE, text=green!55!black] at (6.10,1.70) {$\checkmark$};
\node[font=\LARGE, text=green!55!black] at (8.54,1.70) {$\checkmark$};
\node[font=\LARGE, text=green!55!black] at (10.98,1.70) {$\checkmark$};
\node[font=\bfseries\scriptsize, text=green!55!black] at (5.60,0.72) {$\checkmark$ supported};
\node[font=\bfseries\scriptsize, text=red!70!black] at (8.45,0.72) {$\times$ not supported};
\node[advantage, text width=10.2cm] at (6.10,-0.95) {\textbf{No Explicit Parent State Dependency within BSIT enables:} \\[1pt] %for Child Node Sampling enables: \\[1pt]
$1.$ Superior Search Performance under Sparse-Reward.\\[1pt]
$2.$ Superior Performance under Limited Feedback-Budget and Compute Constraints.};
\end{scope}
\end{tikzpicture}%
}
\caption{\large{Conceptual Overview of BSIT. $I_{t^{\prime}_i}$ denotes Stochastic Interpolant operator.}}
\label{fig:bsit_trajectory}
\vspace{-3pt}
\end{figure*}

To mitigate these issues, tree-based inference-time samplers~\citep{jain2026diffusion,ji2026bootstrap,guo2026training} have been introduced that backpropagate only terminal rewards, eliminating the need to evaluate rewards at every intermediate denoising step. 
Despite their potential, current tree-based samplers~\citep{jain2026diffusion,ji2026bootstrap,li2025dynamic} are constrained by DDPM-like transitions for child node sampling. This creates tightly coupled intermediate nodes in a trajectory that yield highly correlated marginal samples across the denoising time steps, ultimately suppressing the broad path-space exploration vital for budget-constrained search. 
To overcome this bottleneck, we introduce a deceptively simple yet powerful trajectory construction strategy that bypasses the need for step-by-step, DDPM-like transitions for child node sampling. By anchoring the terminal node from any denoising time step powered by a pretrained Flow-map model, we synthesize all intermediate states of a trajectory by stochastically interpolating between this terminal bootstrap target and a random source sample. This design eliminates explicit parent-child dependencies, structurally minimizing correlation between intermediate nodes to unlock expansive path-space coverage.
By decoupling these marginal distributions across the denoising steps, our approach enables enhanced path-space exploration, ultimately ensuring broader target space coverage necessary to succeed under sparse reward settings. We depict a conceptual overview of the BSIT sampler in Figure 1. We summarize our key contributions below:
\begin{itemize}[noitemsep,topsep=0pt, leftmargin=*]
 \item We introduce the Bootstrap Stochastic Interpolant Tree, a novel framework that enhances the exploration capability of a tree sampler to enable feedback-efficient fine-grained search.
 \item We demonstrate that BSIT significantly outperforms existing methods in exploration-heavy regimes characterized by sparse rewards and strict feedback budgets.
 \item Through evaluation across diverse search settings, we show that BSIT achieves faster convergence and broader target distribution coverage compared to prior tree samplers.
 %\item We show that BSIT achieves faster convergence and superior target coverage under fixed budgets. %compared to prior tree-based samplers. % relying on DDPM-style transitions for child node sampling.
\end{itemize}

\vspace{-7pt}
\section{Related Work}
\vspace{-6pt}
\paragraph{Tree-Based Sampler For Search}
Tree-based samplers have been introduced for effective inference-time search and alignment. Diffusion Tree Sampler (DTS)~\citep{jain2026diffusion}  casts the denoising process as a finite-horizon tree, where, similar to MCTS, rollouts are used to continuously improve value estimates for intermediate noisy states. DTS thus enables a principled search in the space of denoising trajectories to identify the modes within high-volume regions of the target density. Unlike SMC-style samplers, tree-based frameworks backpropagate only terminal rewards, eliminating the need to evaluate rewards at every intermediate denoising step. While promising, current tree-based samplers~\citep{jain2026diffusion,guo2026training} are severely bottlenecked by the massive number of function evaluations (NFEs) needed for node valuation, rendering them impractical for exploration-heavy tasks. Moreover, their dependence on rigid, uniform transitions restricts adaptive search, and heuristic workarounds like beam search~\citep{li2025dynamic} fail to outperform standard UCT selection~\citep{kocsis2006bandit}. %Exacerbating these bottlenecks, existing exploration strategies are fundamentally budget-agnostic, rendering them infeasible for resource-constrained tasks. 
To resolve these drawbacks, \citep{ji2026bootstrap} recently proposed the Flow-map tree. By utilizing a flow-map to construct trajectories in a single NFE, this approach drastically reduces computational overhead while exploiting sufficient statistics for child-node sampling to enable dynamic transitions and flexible search.
\vspace{-6pt}
\paragraph{SMC-based Sampler for Search}
Inference-time scaling has seen rapid adoption, with Sequential Monte Carlo (SMC) serving as a primary framework. SMC methods approximate the target distribution by maintaining an evolving population of sample particles. 
SMC frameworks~\citep{phillips2024particle,del2006sequential,kim2025test,singhal2025general} utilize potential functions—which typically approximate the soft value function—to iteratively weight and resample particles at each step. While numerous variants have recently been adapted for diffusion model alignment~\citep{cardoso2023montecarloguideddiffusion,kim2025test,wu2023practical}, classical SMC theoretically guarantees both exact target sampling and optimal value estimation only in the limit of infinite particles.
In practice, repeated resampling inherently diminishes sample diversity due to high weight variance and imprecise value estimation. %Moreover, this process depends on a learned reward model that is typically poorly calibrated during initial online interactions. Consequently, this early-stage bias misdirects the reverse denoising trajectory, severely degrading final sample quality.
To combat particle weight degeneracy and debias particle-based guidance, several correction techniques~\citep{skreta2025feynman,lee2025debiasing} have recently been proposed. However, these methods falter in high-dimensional domains, such as image synthesis, due to their reliance on dense reward gradients at every resampling step. Because exact gradients are rarely available in practice, relying on approximations inevitably injects systematic bias into the sampling trajectory.
Although numerous gradient-free alignment methods circumvent the need for reward gradients via best-of-N sampling~\citep{stiennon2020learning,beirami2024theoretical}, this generate-and-filter paradigm incurs severe computational overhead. Furthermore, because intervention occurs exclusively post-hoc at the final output state, these approaches fundamentally fail to actively steer the intermediate generative trajectory—a capability essential for feedback-efficient sampling.
%\vspace{-6pt}
\begin{comment}
\paragraph{Direct Fine-tuning of Generative Models for Efficient Search}
Fine-tuning generative models using human feedback~\citep{ouyang2022training,touvron2023llama}, has become increasingly common. Existing approaches largely focus on optimizing predefined reward functions through supervised learning~\citep{lee2023aligning,wu2023better}, control-based strategies~\citep{ajay2022conditional,janner2022planning}, or policy-gradient methods~\citep{black2024training}. A key limitation of these frameworks is their reliance on a static reward model, where rewards are treated as fixed ground truth, precluding adaptive or online interaction. In contrast, we consider an online, interactive setting in which the target properties are not specified a priori but are progressively uncovered through sequential feedback. This formulation enables continual target discovery rather than one-shot alignment to a fixed reward proxy. Recent work introduces a feedback-efficient online fine-tuning approach for diffusion models~\citep{uehara2024feedback}; however, it depends on learning a separate online reward model to guide sampling. This design can exacerbate bias during early training stages, and the KL-regularized objective further exposes the method to mode collapse.
\vspace{-14pt}
\end{comment}
\section{Problem Formulation}
%\vspace{-3pt}
%Real-world data distribution is inherently long-tailed in nature. Thus, we assume we are given a long tail dataset, represented as $
%D_{lt}$,  consisting of \( m \) independent and identically distributed (i.i.d.) samples \( \{x_i\}_{i=1}^m \) drawn from a distribution \( \mathcal{P}_X \) over the data space \( X \). As human preferences are unique and heterogeneous, we assume the user preference are diverse and come from distinct modes of the underlying distribution and also from the long-tail region of the distribution. The user preferences are not known a priori, but revealed through interactive feedback. Obtaining feedback is costly, as it might require real world wet lab experiment, for example -in drug discovery. Thus our goal is to design a sampler that maximizes user preference within a pre-defined fixed query budget $B$.
%Real-world data are typically long-tailed, with most probability mass concentrated on a few frequent patterns and a vast number of rare ones. 
%Assuming access to a pretrained generative model (e.g., a Flow Matching model), trained on a dataset $\mathcal{D}=\{x_i\}_{i=1}^m$, consisting of $m$ i.i.d. samples drawn from an unknown distribution $\mathcal{P}_X$ over a data space $\mathcal{X}$. 
We assume access to a pretrained generative model, such as a Flow Matching model, optimized over a dataset $\mathcal{D}=\{x_i\}_{i=1}^m$, consisting of $m$ i.i.d. observations sampled from an unknown true distribution $\mathcal{P}_X$ defined over the space $\mathcal{X}$.
%In practice, desired target properties are inherently multimodal; thus, a given downstream objective may naturally favor several distinct modes of $\mathcal{P}_X$. These high-value target properties are not known a priori and are revealed only through costly interactive feedback—for example, via physical experiments or through human engagement—so feedback queries must be used sparingly. 
%In practice, downstream objectives are inherently sparse and multimodal, naturally favoring a sparse set of distinct regions of $\mathcal{P}_X$. Because often the case in practice the obtaining reward feedback is costly—such as physical experiments or human evaluations—the search process demands strict sample efficiency to minimize expensive queries.
In practice, downstream objectives are inherently sparse and multimodal, concentrating target properties within a few distinct regions of $\mathcal{P}_X$. Because acquiring reward feedback—such as through physical experiments or human evaluation—is often highly costly, the search process requires strict sample efficiency to minimize expensive queries.
%Our goal is therefore to design a sampler that, given a fixed feedback budget $\mathcal{B}$, efficiently allocates queries to learn the underlying multi-modal target structure and generates samples that maximally align with the discovered target properties under the data distribution. We formulate the objective as:

Therefore, we seek to derive a sampler $g_{\theta_\mathcal{T}}$, constrained by a fixed feedback budget $\mathcal{T}$, that, at inference time, generates samples that maximize alignment with desired high-utility properties. We formalize this objective as: $$\frac{1}{\mathcal{B}}\sum_{i=1}^{i=\mathcal{B}}\mathbb{E}_{x^i \sim g_{\theta_\mathcal{T}}}[r(x^i)]$$

%Therefore, our goal is to design a sampler that operates within a strict feedback budget $\mathcal{B}$ to uncover the underlying multimodal target structure efficiently and generate samples that maximally align with the discovered high-utility properties while remaining consistent with the data distribution. We formally define this objective as: $\max_{\{g_{\theta^{\mathcal{D}}_t}\}}\sum_{t=1}^{t=\mathcal{B}}\mathbb{E}_{x_t \sim g_{\theta^{\mathcal{D}}_t}}[r(x_t)]$
%$g_{\theta^{\mathcal{D}}_t}$ denotes the underlying generative model, parameterized by $\theta^{\mathcal{D}}_t$, used to draw samples at the $t$-th interaction step. $\theta^{\mathcal{D}}_0$ denotes the parameters of the base generative model, pretrained exclusively on $\mathcal{D}$ to generate samples before observing any ground-truth feedback. Unlike direct fine-tuning methods that update parameters after each feedback step, our framework operates entirely at inference time. Specifically, we keep the base model strictly frozen throughout the search, ensuring that $\theta^{\mathcal{D}}_0$ remains constant across all interaction steps $t \in \{1, \ldots, \mathcal{B}\}$.
Here, $\mathcal{B}$ represents inference-time evaluation budget, $r: \mathcal{X} \to [0, 1]$ denotes the probabilistic sparse reward function, where $r(x^i) = 1$ signifies that the $i$-th generated sample successfully exhibits the desired high-value target properties, and $r(x^i) = 0$ indicates a failure to do so. Leveraging the generative prior pre-trained on $\mathcal{D}$, our objective is twofold: first, to efficiently utilize the budget $\mathcal{T}$ to derive a sampler $g_{\theta_\mathcal{T}}$ that explores the vast space $\mathcal{X}$ to discover diverse, high-utility search trajectories; and second, to systematically exploit this sampler $g_{\theta_\mathcal{T}}$ at inference time to generate samples maximally aligned with these rare target properties.
%Leveraging the generative prior trained on $\mathcal{D}$, our objective is twofold: to efficiently utilize the budget T to derive a sampler that explore the vast space $\mathcal{X}$ to discover diverse, high-utility regions, and during inference to systematically exploit the derived sampler to generate samples that maximally align with these rare high-utility properties.
%\vspace{-14pt}
\section{Method: Background}
Effective, feedback-efficient search requires an exploration policy that strategically navigates the target landscape to sample efficiently from the target distribution:
\vspace{-3pt}
\begin{equation}
    \small{\pi^{*}(x) = \frac{1}{Z} p^{pre}_{\theta}(x) \exp(\tau \> r(x))}
    \label{eq:reward_tilted}
\end{equation}
Here, $p^{\text{pre}}_{\theta}(x)$ is the pretrained generative prior, $r(x)$ is the reward function, $\tau$ is the inverse temperature, and $Z$ is the normalizing constant. The target distribution can be seen as the optimal policy for the following objective~\citep{jain2026diffusion}:
\vspace{-3pt}
\begin{equation}
    \small{\pi^*(x) := \text{arg\,max}_{\pi} \mathbb{E}_{x \sim \pi(\cdot)} [r(x)] - \frac{1}{\tau} D_{\text{KL}}(\pi \parallel p^{pre}_\theta)}
    \label{eq:optimal_policy}
\end{equation}

Efficiently optimizing $\pi^*(x)$ requires the estimation of the soft value function at time $t$, defined as:
\begin{equation}
    \small{V_t(x_t) := \frac{1}{\tau} \log \mathbb{E}_{p_{\theta}(x_{0:t-1}|x_t)} \left[ \exp \left( \tau r(x_0) \right) \right]}
    \label{eq:soft_value_func}
\end{equation}
Estimating $x_t$'s value via Eq.~\ref{eq:soft_value_func} introduces a computational bottleneck by requiring full denoising trajectory simulations, which incur high NFEs. Standard Tweedie-based terminal estimators fail here due to inaccuracies at high noise levels. Furthermore, a fundamental tension exists: while highly efficient, ODE sampling lacks the stochastic transitions strictly required for search. Minimizing the discretization error of these stochastic steps requires them to be infinitesimal, demanding a prohibitively large number of transitions.
To address this conflict, prior work~\citep{potaptchik2026meta} derives an ODE-based flow model capable of sampling from $x_0 \sim p(x_0|x_t)$ in a single step. 
Theoretically, given $x_t$, an ODE transports the prior $p_1$ to the conditional posterior $p_{0\vert{}t}(\cdot\vert{}x_t)$. Its associated drift $b_s(\cdot;x_t)$ solves the standard flow matching problem~\cite{lipman2022flow}, explicitly targeting the conditional posterior $p_{0\vert{}t}(\cdot\vert{}x_t)$ instead of the marginal data distribution $p_0$:
\begin{equation}
\small{
\begin{aligned}
    {b}_s(x; x_t) &= \mathbb{E}\left[\dot{\alpha}_sI_1 + \dot{\beta}_sI_0 \mid I_s = x\right] \\
    I_s &= \alpha_sI_1 + \beta_sI_0; \>\>
    \text{where } I_1 \sim \mathcal{N}(0, I_d), I_0 \sim p_{0|t}(\cdot|x_t)
\end{aligned}
}
\end{equation}
Consequently, the probability flow associated with 
$b_s(\cdot; x_t)$ satisfies~\citep{holderrieth2025glass}:
\vspace{-3pt}
\begin{equation}
\small{\frac{d}{ds}x_s = b_s(x_s; x_t), \quad x_1 \sim \mathcal{N}(0, I_d) \implies \text{Law}(x_0) = p_{0|t}(\cdot|x_t)}
\label{eq:ode_cond}
\end{equation}
$\alpha_t$ and $\beta_t$ are time-dependent coefficients satisfying $\alpha_0 = \beta_1 = 0$ and $\alpha_1 = \beta_0 = 1$. Assuming a Gaussian prior $p_1$, the conditional drift $b_s$ can be derived analytically by reparameterizing the unconditional drift $b_t$~\cite{holderrieth2025glass}:
$$b_s(x_s; x_t) = w_1x_s + w_2 b_{t^*}(S(x_s, x_t)) + w_3 x_t$$
Where, $w_1, w_2,$ and $w_3$ represent scalar weights. $S$ serves as a linear \emph{sufficient statistic}, and $t^*$ denotes a reparametrized time parameter:
\vspace{-3pt}
\begin{equation}
\small{
\begin{aligned}
    S_{s,t}(x_s, x_t) &= \frac{\alpha_s \sigma^2_t x_s + \alpha_t \sigma^2_s x_t}{\sigma^2_t \alpha^2_s + \alpha^2_t \sigma^2_s} \\
    t^*(s, t) &= g^{-1}\left( \frac{\sigma^2_t \sigma^2_s}{\sigma^2_t \alpha^2_s + \alpha^2_t \sigma^2_s} \right), \quad g(t) = \frac{\sigma^2_t}{\alpha^2_t}
\end{aligned}
}
\label{eq:combined_reparam}
\end{equation}
While this reparameterization makes $b_s$ accessible, generating posterior samples from $p_{0\vert{}t}(\cdot\vert{}x_t)$ via ODE integration remains computationally intensive. To bypass this bottleneck, standard teacher-distillation utilizes the exact analytical field $b_s(x_s; x_t)$ as a teacher. The student model's velocity $\hat{v}$ is optimized using a combined objective, featuring an instantaneous loss that minimizes the mean squared error between the student's instantaneous velocity $\hat{v}_{s,s}$ and the teacher field:
\vspace{-3pt}
\begin{equation}
\small{\mathcal{L}_{inst}^{distill}(\hat{v}) := \int_0^1 \int_0^1 \mathbb{E} ||\hat{v}_{s,s}(x; x_t) - b_s(x; x_t)||^2 ds dt}
\end{equation}
Pairing this with a consistency loss~\cite{boffi2025build} ($\mathcal{L}^{\text{distill}}_{\text{cons}}$) ensuring the student accurately integrates this velocity over varying time intervals, defined as follows:
\vspace{-3pt}
\[
\small{
\begin{aligned}
    \int_{0}^{1} \int_{0}^{u} \mathbb{E} \bigg\| \hat{v}_{s,u}(I_{s}; I'_t) &- \text{sg}\big( b_u (\hat{X}_{s,u}(I_{s}; I'_t)); I'_t \big) \\
    &- (u - s)\partial_u \hat{v}_{s,u}(I_{s}; I'_t) \bigg\|^2 ds \, du
\end{aligned}
}
\]
Here, $X_{s,u}(\cdot \, ; x_t) : \mathbb{R}^d \to \mathbb{R}^d$ denotes the solution operator for the context-dependent ODE in Eq.~\ref{eq:ode_cond}. This distillation compresses the multi-step ODE trajectories into a single amortized map, enabling efficient, one-step sampling directly from the conditional posterior $p_{0\vert{}t}(\cdot\vert{}x_t)$.
%\vspace{-14pt}
\section{Method: BSIT Framework}
%Armed with the distilled flow-map model $\hat{v}(x;x_t)$, we first sample the terminal node in a single step from any given node $x_t$. In order to sample a child of $x_t$, instead of following the strict guidance of the denoiser conditioned on $x_t$, we compute a stochastic interpolant between the bootstrapped terminal node $x_0$ and a randomly sampled noise from the source distribution. Specifically, we perform the following steps:
\paragraph{Trajectory Construction}
Leveraging the distilled flow-map model $\hat{v}(x;x_t)$, we first generate the terminal node $x_0$ in a single step from any given parent $x_t$. \emph{To sample its child node ($x_{t-1}$)}, rather than relying on DDPM-style standard reverse-time denoising, we construct a stochastic interpolant between the terminal sample $x_0$ and a random noise vector drawn from the source distribution via the following steps:
\begin{equation}
    \small{x_{t-1} = \alpha_{t-1} x_0 + \beta_{t-1} \epsilon^{\prime}, \quad \text{where } x_0 \sim p_{\hat{v}(\epsilon;x_t)}(x_0 \mid x_t), \;} %\epsilon, \epsilon^{\prime} \sim \mathcal{N}(0, \mathbf{I})} %\quad \forall s \in (0, t)}
    \label{eq:flow_ss}
\end{equation}
Here, $\epsilon, \epsilon^{\prime} \sim \mathcal{N}(0, I)$. 
%Rather than stepping recursively backward through a denoiser, we construct the complete trajectory starting at $x_t$ (i.e., $x_s$ $\forall{s \in (0,t)}$) by iteratively evaluating the stochastic interpolant between a fixed bootstrapped terminal node $x_0$ and independent noise samples $\epsilon_s \sim \mathcal{N}(0, I)$ across the desired intermediate time steps $s \in (0,t)$. Specifically, we compute, $x_s =\alpha_{s} x_0 + \beta_{s} \epsilon^{s}, \forall{s \in (o,t)}$. This process populates every intermediate node between the parent and the terminal state, synthesizing the entire tree trajectory in a highly parallelizable and non-sequential manner. We term the proposed trajectory construction scheme Bootstrap Stochastic Interpolant Tree (a.k.a BSIT). 
Rather than stepping recursively backward through a denoiser, we construct the complete trajectory originating at $x_t$ by evaluating the stochastic interpolant across intermediate time steps $s \in (0, t)$. Specifically, using a bootstrapped terminal node $x_0$ alongside independent noise vectors $\epsilon_s \sim \mathcal{N}(0, \mathbf{I})$, each intermediate state is computed directly as: $x_s = \alpha_s x_0 + \beta_s \epsilon_s, \quad \forall s \in (0, t)$, where $\epsilon_s \sim \mathcal{N}(0, I)$. We denote $x_s \sim p_{\text{BSIT}}(x_s \mid x_t)$. This formulation populates all intermediate nodes between the parent and terminal states simultaneously, synthesizing the entire trajectory in a highly parallelizable, non-sequential manner. We term this trajectory construction scheme the Bootstrap Stochastic Interpolant Tree (BSIT).
Beyond eliminating the extra function evaluations (NFEs) required by traditional denoisers, BSIT provides complete flexibility over transition step sizes. Unlike DDPM-style transitions that constrain $t'$ to be infinitesimally close to $t$ to control discretization error, our stochastic interpolant formulation supports arbitrary step sizes. This enables dynamic search strategies that progressively scale transition steps as tree depth increases, offering precise control over exploration–exploitation dynamics.
%Beyond eliminating additional function evaluations (NFEs) incurred by traditional denoisers, this trajectory construction strategy offers complete flexibility over the transition step. Unlike DDPM-style transitions that constrain $t'$ to be infinitesimally close to $t$ to minimize discretization error, our BSIT based trajectory construction strategy allows arbitrary step sizes. This enables dynamic search strategies that can progressively scale transition steps as the tree deepens, allowing efficient control over exploration-exploitation dynamics. 
\vspace{-10pt}
\paragraph{Value Assignment and its Benefits}
In tree-based search frameworks, the sampler must accurately estimate the soft value function $V_t(x_t)$ at an intermediate parent node $x_t$ to guide the tree policy toward high-reward trajectories under a strict feedback budget. The soft value function of BSIT can be represented as: $$V_t(x_t) = \frac{1}{\tau} \log \mathbb{E}_{x_0 \sim p_{\hat{v}(\epsilon;x_t)}(\cdot \mid x_t)} \left[ \exp\left( \tau \cdot r(x_0) \right) \right]$$ 
Applying the law of total expectation and by conditioning on a child state $x_s$, we express $\mathbb{E}_{x_0 \sim p_{\hat{v}(\epsilon, x_t)}} \left[ \exp(\tau r(x_0)) \right]$ as:
$$ \mathbb{E}_{x_s \sim p_{\text{BSIT}}(\cdot \mid x_t)} \left[ \mathbb{E}_{x_0 \mid x_s, x_t} \left[ \exp(\tau r(x_0)) \right] \right]$$ 
Taking the log-soft max and using markovian transition property yields the BSIT Continuous Bellman Backup Identity:
\begin{equation}
V_t(x_t) = \frac{1}{\tau} \log \int p_{\text{BSIT}}(x_s \mid x_t) \exp\left( \tau \cdot V(x_s ) \right) \mathrm{d}x_s
\label{eq:bellman}
\end{equation}
where $V(x_s ) = \frac{1}{\tau} \log \mathbb{E}_{x_0 \mid x_s} [\exp(\tau r(x_0))]$ is the soft value estimate of the child node of $x_t$.
Although \emph{iterative Bellman updates} (Eq.~\ref{eq:bellman}) guarantee that the node value estimate $V_t(x_t)$ converges to the unique optimal value function $V^*$, high estimator variance severely impairs feedback efficiency. Specifically, elevated variance exponentially increases the number of rollouts needed to resolve optimal paths from suboptimal ones. Crucially, our trajectory construction strategy significantly reduces value estimation variance. This lower variance allows BSIT to achieve the same value function accuracy with strictly fewer rollouts, making it fundamentally more sample efficient under tight feedback budgets, as stated in the following Proposition.
%the soft Bellman backup operator $\mathcal{T}_{\text{BSIT}} V(x) = \frac{1}{\tau} \log \mathbb{E}_{x_s \mid x} [\exp(\tau V(x_s))]$ is a $\gamma$-contraction under the supremum norm $\Vert{}\cdot\Vert{}_\infty$, guaranteeing unique convergence to the optimal value function $V^*$, here we aim to derive a sampler that is capable of estimating optimal value with as minimal feedback sample as possible. To this end it is important to observe that  high estimator variance directly wastes the sample budget, as it requires an exponentially larger number of rollouts to distinguish optimal paths from suboptimal ones. Crucially, our trajectory construction strategy significantly reduces value estimation variance. This lower variance allows BSIT to achieve the same value function accuracy with strictly fewer rollouts, making it fundamentally more sample-efficient under tight feedback budgets, as stated in the following Proposition.
\begin{Proposition}
    Let $x_t$ be an intermediate parent node at time $t$. Let $\hat{V}_{\text{DDPM}}(x_t)$ and $\hat{V}_{\text{BSIT}}(x_t)$ be the empirical Monte Carlo estimators of the soft value function $V_t(x_t)$ constructed from a single rollout trajectory under DDPM-style child node sampling and BSIT, respectively. The variance of the BSIT value estimator is strictly lower than that of the DDPM-denoiser-based trajectory rollout:$\text{Var}(\hat{V}_{\text{BSIT}}(x_t)) < \text{Var}(\hat{V}_{\text{DDPM}}(x_t))$
\end{Proposition}
\vspace{-10pt}
\paragraph{Remark 1}
By the Hoeffding bounds for Monte Carlo convergence, the number of samples $M$ required to achieve an estimation error within a tolerance $\epsilon$ with high probability scales linearly with the variance: $M \propto \frac{\text{Var}(\hat{V})}{\epsilon^2}$. Since $\text{Var}(\hat{V}_{\text{BSIT}}) < \text{Var}(\hat{V}_{\text{DDPM}})$, BSIT requires strictly fewer rollout trajectories to achieve a given value function accuracy, demonstrating superior sample efficiency under constrained feedback budgets. 

Importantly, by breaking strict parent-child dependencies along the trajectory, BSIT decorrelates the induced marginal distributions at different tree depths. This structural decoupling expands path-space exploration and yields broader distributional coverage, ultimately generating highly diverse, target-aligned samples. 
The following proposition demonstrates that BSIT yields higher trajectory entropy, underpinning its enhanced exploratory capacity.
\vspace{-6pt}
\begin{Proposition}
    Let $\Omega = \mathcal{C}([0, t], \mathbb{R}^d)$ be the space of continuous path segments from time $t$ to $0$ starting at $x_t$. %Let $\mu_{\text{Diff}}$ and $\mu_{\text{BSIT}}$ be the probability measures induced on $\Omega$ by the sequential denoiser based tree sampler and the BSIT framework, respectively.
    For any intermediate timestamp $s \in (0, t)$: %The spatial support of the BSIT path density strictly contains the support of the reverse-time sequential diffusion path density: $\text{supp}(p_{\text{DDPM}}(x_s \mid x_t)) \subseteq \text{supp}(p_{\text{BSIT}}(x_s \mid x_t))$. Moreover, 
    The conditional path entropy of BSIT is strictly greater than that of traditional reverse-time sequential diffusion: $H_{\text{BSIT}}(x_s \mid x_t) > H_{\text{DDPM}}(x_s \mid x_t)$
\end{Proposition}
\vspace{-10pt}
\paragraph{Adaptability with Deterministic Flow-map Prior} However, deploying the BSIT framework still relies on the availability of a distill meta Flow-Map model ($\hat{v}(\epsilon;x_t)$) capable of sampling from $p_{\hat{v}(\epsilon;x_t)}(x_0 | x_t)$~\citep{potaptchik2026meta,holderrieth2025glass}. 
The distilled model $\hat{v}(\epsilon; x_t)$ allows us to repeatedly sample distinct terminal nodes $x_0$ from a fixed parent $x_t$ simply by passing new noise vectors $\epsilon \sim \mathcal{N}(0, \mathbf{I})$ through the map. Consequently, revisiting high-value nodes during exploitation never leads to sample duplication. In practice, however, such a distilled stochastic model may be unavailable. In its absence, the deterministic nature of a standard flow map severely restricts repeated node visits: while the first visit (during the construction) to $x_t$ can yield a new terminal sample via Equation~\ref{eq:flow_ss}, all subsequent visits deterministically generate identical $x_0$ outcomes. This redundancy inefficiently exhausts the sampling budget, ultimately hurting downstream search performance.  
\vspace{-10pt}
\paragraph{Dynamic Node State} To overcome this challenge, we introduce a dynamic node-updating strategy for tree construction. While we track visit counts and value estimates like prior tree samplers, we do not treat node features as static once created. Instead, we dynamically update a node's feature representation upon each revisit, enabling deterministic flow maps to continuously generate novel terminal samples without duplication. Specifically, when revisiting $x_t$, we draw a new noise vector $\epsilon'$ and re-evaluate the stochastic interpolant between the fixed initial terminal node $x_0$ and $\epsilon'$ to obtain an updated state $x_t^{\text{new}}$. Because $x_t^{\text{new}}$ shares the same marginal distribution as $x_t$, all historical node statistics—including value estimations—remain valid. This simple dynamic refresh grants the deterministic flow-map additional freedom to sample novel targets from $p(x_0 \mid x_t)$, significantly enhancing local path exploration over static-node approaches.
\vspace{-15pt}
\paragraph{Sampling Mechanism and Convergence}
Similar to prior tree samplers, sampling in BSIT proceeds top-down from the root $x_T$ to a terminal node. At each step, a child node $x_{t-1} \in \mathcal{C}(x_t)$ is selected based on UCT score, involving $V(x_{t-1})$ and number of visits, terminating upon reaching $t = 0$ or an unexpanded node. Crucially, BSIT guarantees asymptotic convergence to the target distribution as the feedback budget increases as stated in the following proposition.
\vspace{-2pt}
\begin{Proposition}
    Let $r: \mathcal{X} \to \mathbb{R}$ be a bounded reward function, and $p_{prior}(x_0)$ be the source prior distribution. As the number of tree iterations $M \to \infty$, the empirical distribution of terminal states generated by BSIT, denoted by $\pi^{(M)}(x_0)$, converges in total variation to the target energy-based policy: $\pi^*(x_0) = \frac{p_{prior}(x_0) \exp\left( \tau \cdot r(x_0) \right)}{Z_0}, \quad \text{where } \quad \quad Z_0 = \int_{\mathcal{X}} p_{prior}(x_0) \exp\left( \tau \cdot r(x_0) \right) \mathrm{d}x_0$
\end{Proposition}
\vspace{-11pt}
We provide the pseudocode of BSIT in the Appendix.
\paragraph{Generalizability to Discrete Domain} 
The key motivation for extending inference-time search with BSIT samplers to the discrete domain is the continuous relaxation introduced by Categorical Flow Maps~\citep{roos2026categorical} (CFMs). Rather than navigating the complexities of discrete Markov chains, CFMs model the generative process as a smooth, continuous probability path constrained directly on the probability simplex $\Delta^K$. The core trick relies on using a simplex-constrained endpoint parameterization. Because the entire flow trajectory unfolds in this continuous space, we can plug in continuous-domain inference-time search algorithms—like BSIT samplers—to categorical settings. The only modification required is a simple argmax projection at the very end of the generated path to yield the discrete output. Our empirical findings on a large-scale molecule generation task highlight BSIT's effectiveness for efficient search in the discrete domain.

\vspace{-7pt}
\section{Experiments and Results}
\vspace{-7pt}
\paragraph{Analysis in 2D setting} We begin by analyzing BSIT in simple 2D settings to answer the following questions:
%\vspace{-12pt}
\begin{itemize}
    \item \textbf{Sparse \& Multimodal Rewards:} Can BSIT sample accurately from the target distribution using highly sparse, multimodal reward signals?
    \item \textbf{Budget-Constrained Search:} Does BSIT effectively uncover high-utility modes under limited feedback budgets?
    %\item \textbf{OOD Generalization:} Can BSIT generate target-aligned samples with an unseen/novel target property that are strictly out-of-distribution (OOD) relative to the prior?
    %\item \textbf{Efficiency under Standard Setting:} How does BSIT's feedback efficiency compare to prior methods in standard search settings?
\end{itemize}
\vspace{-8pt}
\paragraph{Experimental Settings}
Consider a Flow-Map model trained on alternate regions of support in a checkerboard pattern as the prior $p(x_0)$ (see Fig.~\ref{fig:vis-checkerboard} Prior). To evaluate BSIT under sparse, multimodal conditions, we design a reward function $r_{\text{sparse}}(x_0)$ over the support of $p$. %The reward is $1$ within three disjoint squares representing high-utility modes, and $0$ otherwise 
(see Fig.~\ref{fig:vis-checkerboard} Target). To evaluate feedback and compute efficiency, we compare BSIT against competitive baselines across a wide spectrum of feedback budgets ($\mathcal{B}$) and NFE constraints, respectively. To ensure a fair comparison, we hold the feedback budgets $\mathcal{T}$ and NFE constraints constant across all tree samplers during the construction phase. %Additionally, we evaluate BSIT's capability to generate target-aligned samples outside the support of the prior $p(x_0)$. For this analysis, we define a reward function $r_{\text{OOD}}(x_0)$ that yields $1$ in the region between two squares of the prior's support and $0$ otherwise (see Fig.~\ref{fig:vis-checkerboard} target (right)). 
Across all settings, the ultimate goal is to draw as many samples as possible with high $r_{\text{sparse}}(x_0)$ under a fixed evaluation budget.
\newcommand{\largeresultpanel}[2]{%
  \begin{minipage}[t]{0.18\textwidth}
  \centering
  {\small\bfseries #1\par}
  \vspace{4pt}
  \includegraphics[width=\linewidth,height=0.23\textheight,keepaspectratio]{#2}
  \end{minipage}%
}
\vspace{-11pt}
\begin{figure*}[t!]
\centering
{\setlength{\tabcolsep}{1.5pt}%
\setlength{\arrayrulewidth}{0.5pt}%
\begin{tabular}{@{}ccccc@{}}
\largeresultpanel{Prior}{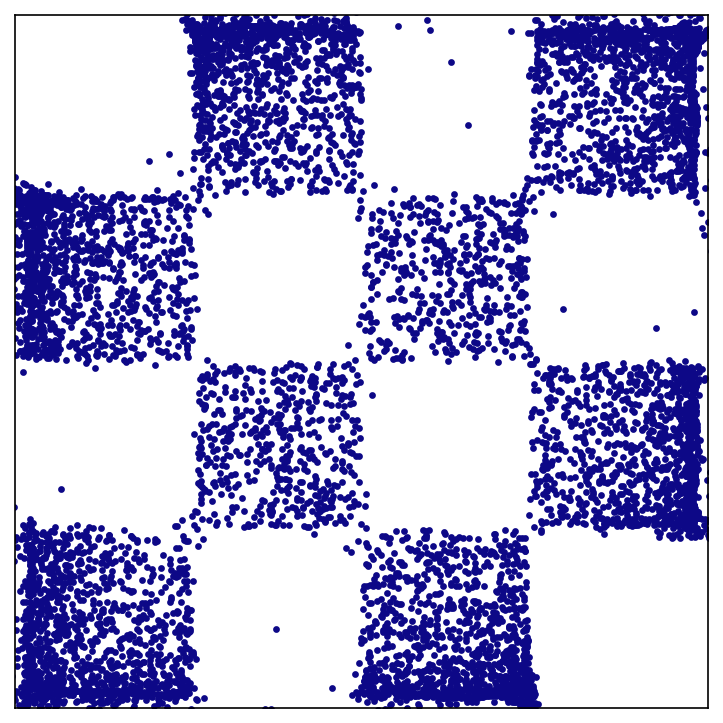} &
\largeresultpanel{Target}{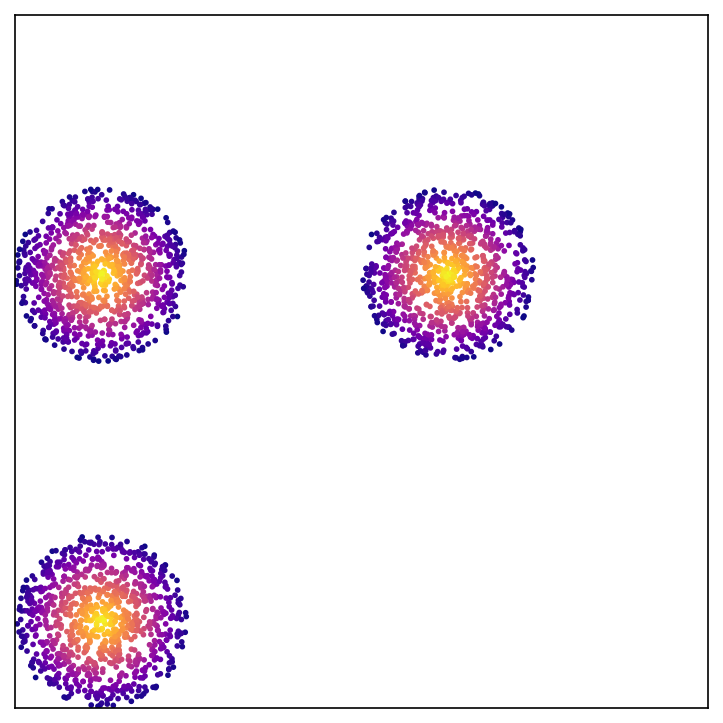} &
\largeresultpanel{DTS}{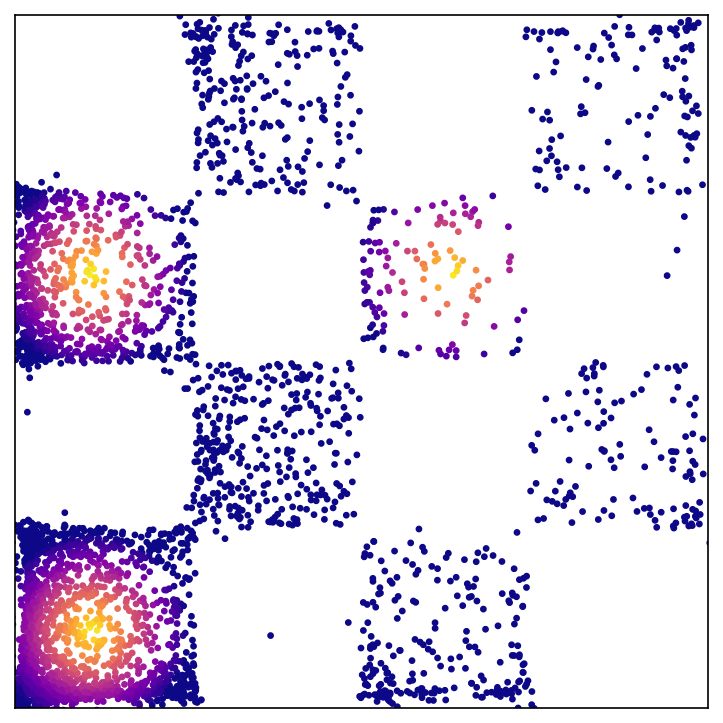} &
\largeresultpanel{BFMT}{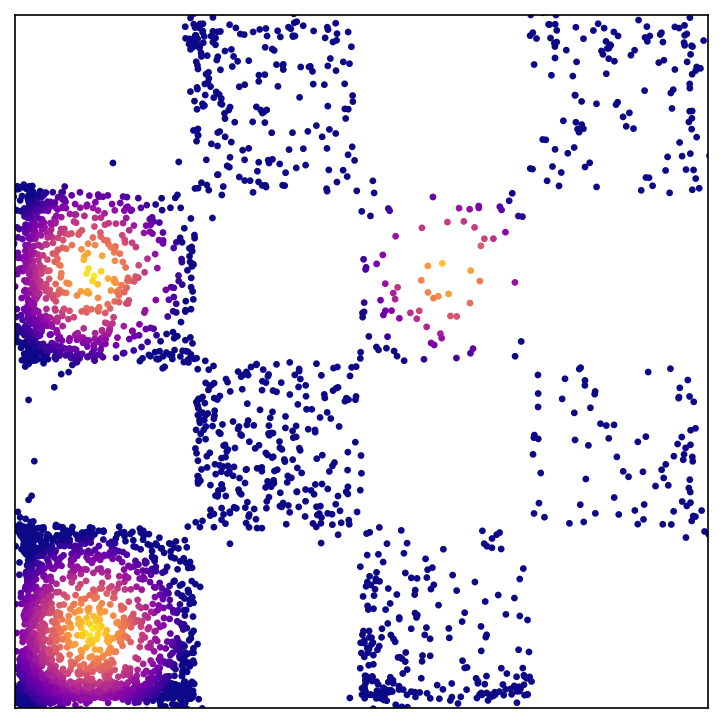} &
\largeresultpanel{BSIT}{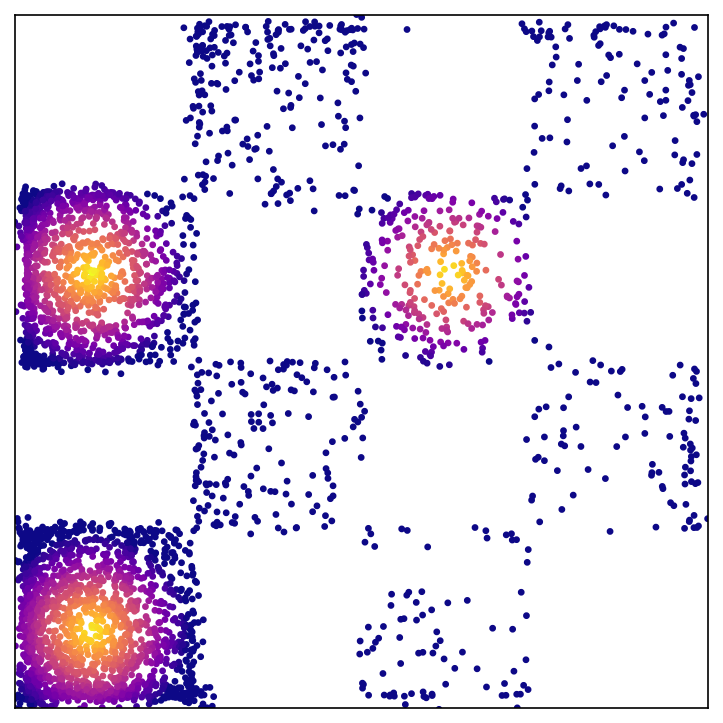}
\end{tabular}%
}
\caption{Qualitative comparison with $\mathbf{r}_{\mathrm{sparse}}$ as the target.}
\label{fig:vis-checkerboard}
\end{figure*}
%\end{comment}
\paragraph{Evaluation Metric}
%avg. reward, diverse reward
We compare the approaches using the mean cumulative reward ($r_{\text{mean}}$) across varying evaluation budgets ($\mathcal{B}$) and varying NFEs. All reported means are averaged across 5 independent trials with different random seeds.
\vspace{-9pt}
\paragraph{Baselines and Reward Function}
We benchmark BSIT against a diverse set of competitive samplers. These include state-of-the-art tree-based samplers, i.e., BFMT~\citep{ji2026bootstrap}, DTS~\citep{jain2026diffusion}, and the recently proposed Flow-Map (FM) approach~\citep{boffi2026build}. To simulate reward feedback, we design a sparse reward function $r_{\text{sparse}}$ defined over three disks, each specified by its center and radius. For a given point $(x, y)$, each disk $i$ defines a field $d_i - \lVert (x,y) - c_i \rVert$ that decreases radially from its center $c_i$ outward, where $d_i$ is its radius. Together they form three high-reward basins against a low-reward background. More details are in the Appendix.

\begin{table}[h!]
\centering
%\small
\caption{Evaluation with $r_{\text{mean}}$ across NFE and feedback budget. For this comparison, we consider $r_{\text{sparse}}$ as the target.}
\label{tab:div-checkerboard}
\resizebox{\linewidth}{!}{%
\begin{tabular}{lccc|ccc}
\hline
\hline
Method & \multicolumn{3}{c|}{NFE} & \multicolumn{3}{c}{Feedback Budget} \\
 & 1000 & 3000 & 5000 & 1000 & 3000 & 5000 \\
\hline
FM  & -0.5229 & -0.5181 & -0.5145 & -0.5229 & -0.5181 & -0.5145 \\
DTS  & -0.4873 & -0.4065 & -0.3621 & -0.2738 & -0.2294 & -0.2107 \\
BFMT & -0.2822 & -0.2122 & -0.2065 & -0.2822 & -0.2122 & -0.2065 \\
\rowcolor{blue!20} BSIT & -0.1489 & -0.1308 & -0.1170 & -0.1489 & -0.1308 & -0.1170 \\
\hline
\hline
%\vspace{-10pt}
\end{tabular}%
}
\end{table}
\begin{comment}
\vspace{-11pt}
\begin{figure}[h!]
\centering
\caption{Evaluation with $\textbf{r}_{\textbf{mean}}$ across NFE and feedback budget. For this comparison, we consider $r_{\text{sparse}}$ as the target.}
\label{fig:plot-cb}
\makebox[\linewidth][c]{%
\begin{subfigure}[c]{0.30\linewidth}
  \centering
  \includegraphics[width=\linewidth]{Figures/a.png}
  %\caption{First result.}
  \label{fig:two_panel_a}
\end{subfigure}\hspace{0.025\linewidth}%
\begin{subfigure}[c]{0.30\linewidth}
  \centering
  \includegraphics[width=\linewidth]{Figures/b.png}
  %\caption{Second result.}
  \label{fig:two_panel_b}
\end{subfigure}\hspace{0.04\linewidth}%
\raisebox{2.4ex}{\begin{minipage}[c]{0.23\linewidth}
  \raggedright\scriptsize
  \textbf{Method}\par\vspace{3pt}
  \textcolor{purple!70!black}{\rule{1.4ex}{1.4ex}}\hspace{3pt}Flow-map\par\vspace{2pt}
  \textcolor{blue!70!black}{\rule{1.4ex}{1.4ex}}\hspace{3pt}DTS\par\vspace{2pt}
  \textcolor{orange!85!black}{\rule{1.4ex}{1.4ex}}\hspace{3pt}BFMT\par\vspace{2pt}
  \textcolor{green!55!black}{\rule{1.4ex}{1.4ex}}\hspace{3pt}BSIT
\end{minipage}%
}}
\end{figure}
\end{comment}
\vspace{-11pt}
\paragraph{Results}
Empirically, BSIT successfully uncovers all distinct high-utility modes of $r_{\text{sparse}}(x_0)$ (see Fig.~\ref{fig:vis-checkerboard}). 
Compared to prior tree samplers, BSIT uncovers more high-utility samples and attains higher mean rewards (Table~\ref{tab:div-checkerboard}) across varying feedback budgets and NFE constraints, demonstrating concurrent gains in both feedback and compute efficiency under a sparse reward signal. 
%Compared to prior tree samplers, BSIT uncovers a more diverse set of high-utility samples under fixed feedback budgets (see Table~\ref{tab:div-checkerboard}). BSIT also attains higher mean reward across varying feedback budgets ( Fig.~\ref{fig:plot-cb}) compared to prior tree-based approaches, highlighting the effectiveness of BSIT in producing both higher diversity high utiity samples compared to prior approaches under a fixed feedback budget. 
Notably, this performance gain widens in higher budget regimes, highlighting BSIT's efficacy for budget-constrained search—a strength we attribute to its enhanced exploration capabilities. 
Across all experiments, BSIT yields identical performance for equal values of NFE and feedback budget, as each feedback evaluation consumes a single NFE.
%Furthermore, BSIT maintains superior diversity and higher mean rewards across varying NFEs, demonstrating computational efficiency by leveraging the Flow-Map and stochastic-interpolant to construct trajectories using a single NFE.
%Furthermore, BSIT maintains superior diversity and mean rewards compared to prior tree samplers across varying NFEs, demonstrating its compute efficiency by leveraging the Flow-Map to construct trajectories using a single NFE.

%Crucially, unlike prior approaches, BSIT uncovers high-utility samples under $r_{\text{OOD}}(x_0)$, outside the support of the prior model (see Fig.~\ref{fig:vis-checkerboard} (right)), highlighting that the enhanced path-space diversity and less correlated marginal distributions across denoising steps is key to significantly expand target space coverage.
\vspace{-11pt}
\paragraph{Analysis with High-Dimension Image Setting}
To analyse BSIT’s feedback-efficient search capabilities in high dimensions, we utilize ImageNet-1k~\citep{deng2009imagenet}. Specifically, we aim to answer the following questions:
\vspace{-12pt}
\begin{itemize}
    %\item \textbf{Budget-Constrained Search:} Does BSIT effectively uncover high-utility modes under limited feedback budgets?
    %\item \textbf{OOD Search:} Can BSIT generate target-aligned samples with an unseen/novel target property that are strictly out-of-distribution (OOD) relative to the prior?
    \item \textbf{Feedback Efficiency under Fine-grained  Search Setting:} How does BSIT's feedback efficiency compare to prior methods in fine-grained search settings?
    \item \textbf{Search under Sparse Reward:} Can BSIT generate target-aligned samples using highly sparse reward signals representing novel/rare target properties relative to the prior?
    %Can BSIT sample accurately from the target distribution using highly sparse, multimodal reward signals?
\end{itemize}
\begin{figure*}[t!]
\centering
{\setlength{\tabcolsep}{2pt}%
\begin{tabular}{@{}ccc|ccc@{}}
\multicolumn{3}{c}{Fine-grained: \itshape ``A red daisy.''} &
\multicolumn{3}{c}{Sparse: \itshape ``Three cats.''} \\[1pt]
\textbf{DTS} & \textbf{BFMT} & \textbf{BSIT} &
\textbf{FMET} & \textbf{BFMT} & \textbf{BSIT} \\
\hline
\includegraphics[width=0.15\textwidth]{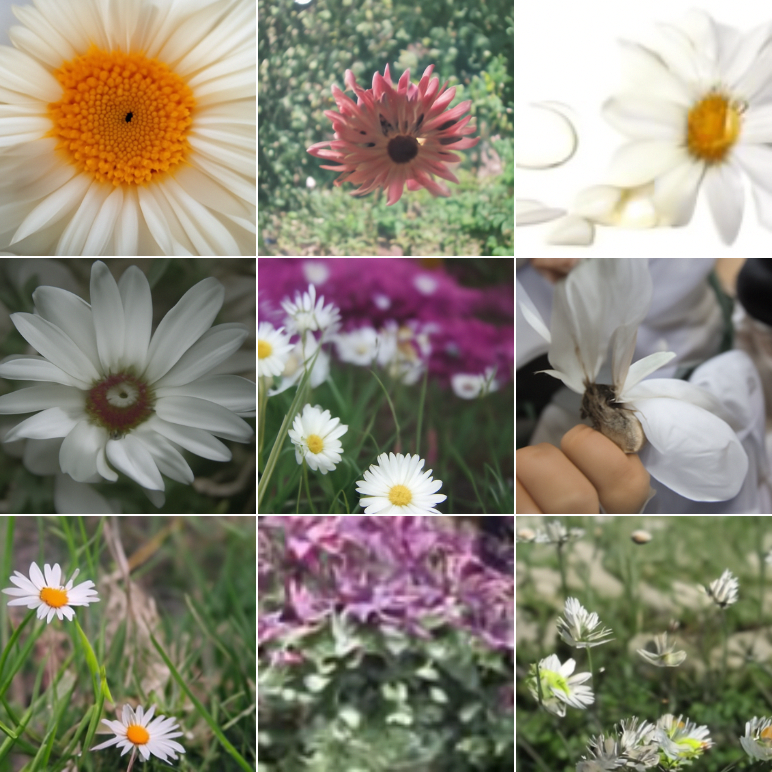} &
\includegraphics[width=0.15\textwidth]{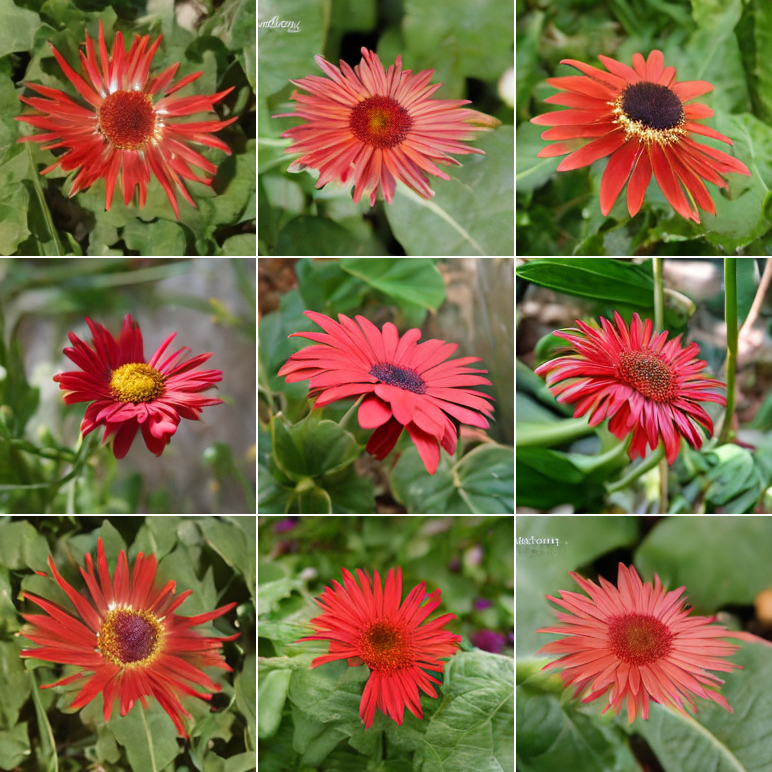} &
\includegraphics[width=0.15\textwidth]{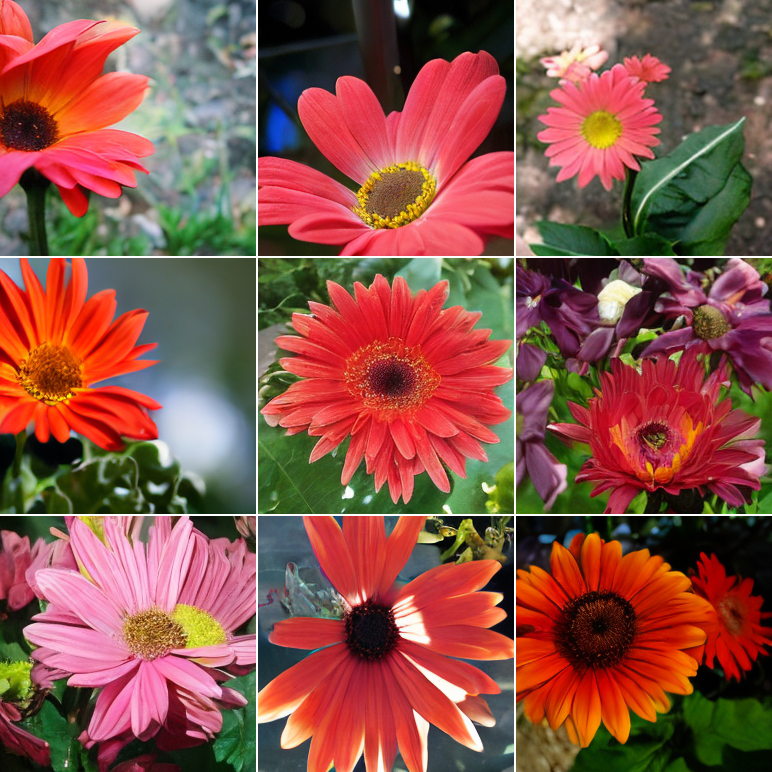} &
\includegraphics[width=0.15\textwidth]{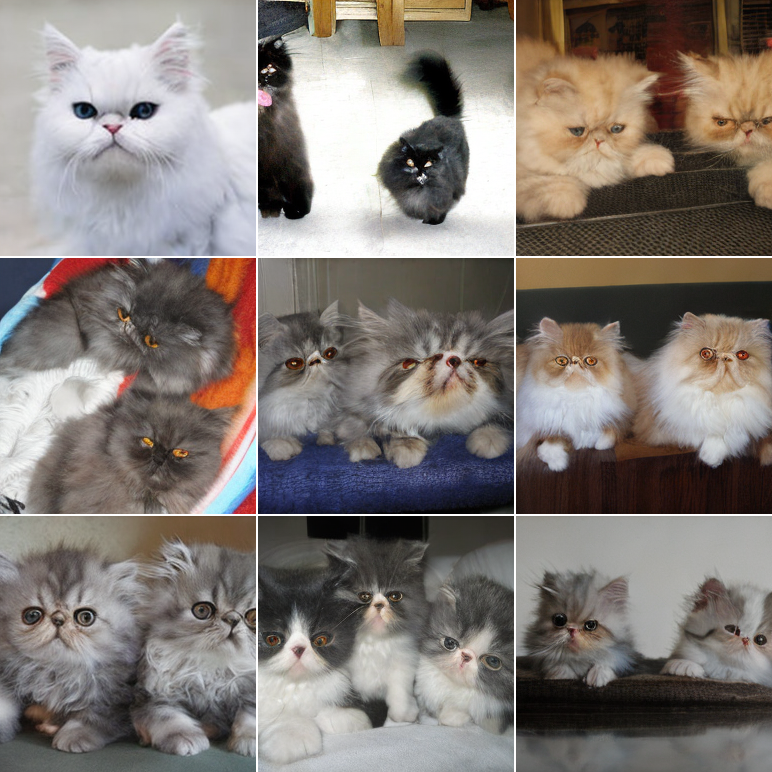} &
\includegraphics[width=0.15\textwidth]{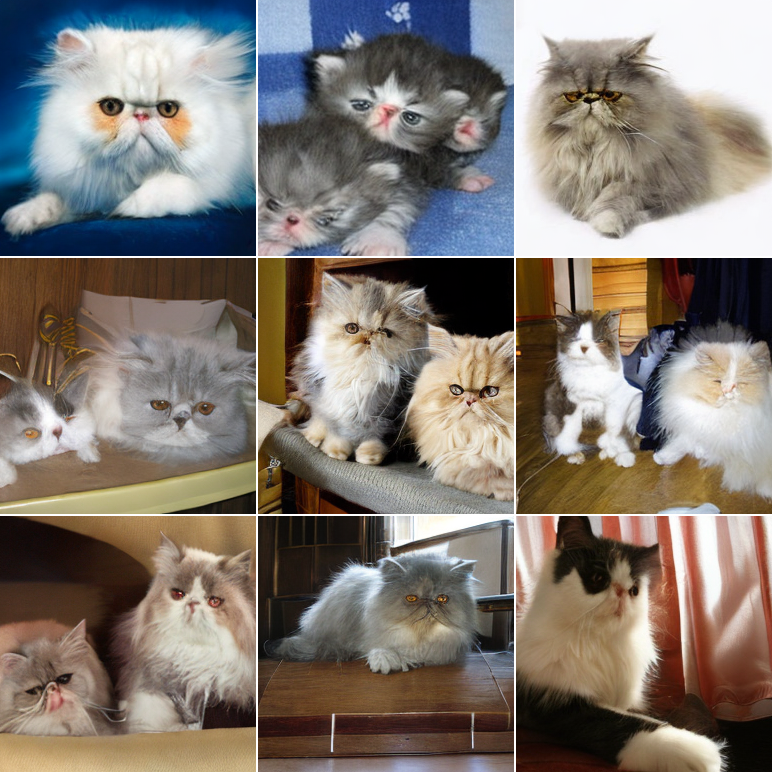} &
\includegraphics[width=0.15\textwidth]{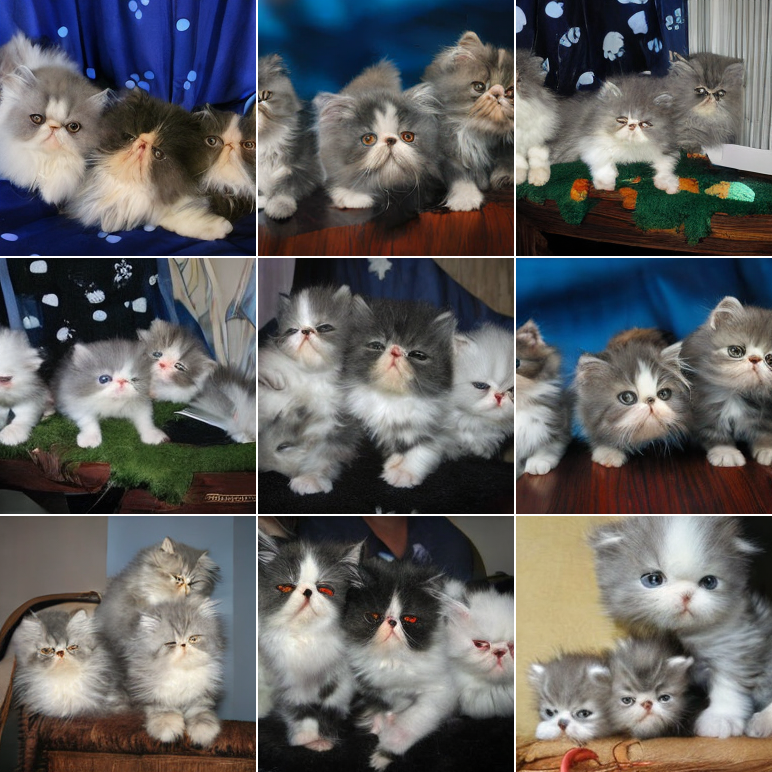}
\end{tabular}%
}
\caption{\small{Search Visualizations. (left) Reward evaluation with $r_{\text{standard}}$.; (right) Reward evaluation with $r_{\text{sparse}}$. }}
\label{fig:search-visu}
\vspace{-12pt}
\end{figure*}
\vspace{-11pt}
%\begin{comment}
\paragraph{Experimental Setting}
Here, the prior is a distilled flow map model optimized jointly via the $\mathcal{L}^{\text{distill}}_{\text{inst}}$ and $\mathcal{L}^{\text{distill}}_{\text{cons}}$ objectives (see Background). We benchmark all approaches across different target classes, revealing fine-grained specifications solely through iterative feedback (class and prompt details in Appendix). Performance is assessed via two complementary metrics: \textit{mean reward} ($r_{\text{mean}}$), capturing overall population-wide alignment efficacy, and \textit{max reward} ($r_{\text{max}}$), measuring peak alignment performance—a critical metric when downstream applications prioritize a single high-quality sample over average generation quality. Consequently, we include MFM ~\citep{potaptchik2026meta} as a baseline exclusively for $r_{\text{max}}$, since its design inherently optimizes for a single best-quality sample rather than population-wide efficiency. As MFM-Base is a Flow-Map Euler Tree sampler, we denote it as FMET.
In contrast to the standard setting, to evaluate BSIT’s search capability under a sparse reward setting, we construct a set of rare target prompts featuring a novel/rare property relative to the prior model's support (details in Appendix). %Specifically, we consider the rare target prompt as "XYZ".
Crucially, similar to the 2D setting, across all ImageNet experiments, we hold the feedback budgets $\mathcal{T}$ and NFE
constraints constant across all tree samplers during the construction phase. The derived sampler must rely entirely on sequential feedback to progressively steer generations toward the true fine-grained target prompt under a fixed evaluation and compute budget.
%Here, the prior is a distilled flow map model optimized jointly via the $\mathcal{L}^{\text{distill}}_{\text{inst}}$ and $\mathcal{L}^{\text{distill}}_{\text{cons}}$ objectives (see Background). We benchmark all approaches across five classes using **standard target specifications**, revealing fine-grained details solely through iterative feedback (prompt details in Appendix). Performance is assessed via two complementary metrics: \textit{mean reward} ($r_{\text{mean}}$), capturing overall population-wide alignment efficacy, and \textit{max reward} ($r_{\text{max}}$), measuring peak alignment performance. Consequently, MFM (BoN)~\citep{potaptchik2026meta} is evaluated exclusively on max reward, as its design inherently optimizes for a single best-quality sample. **In contrast to these standard setups, we evaluate out-of-support search capability under a sparse reward setting using a rare target specification ("XYZ") that introduces a novel property relative to the prior distribution.** Crucially, across all ImageNet experiments, the target specifications remain unknown at the outset; the model must rely entirely on sequential terminal feedback to progressively discover them.
%\vspace{-8pt}
\paragraph{Baselines and Evaluations}
We benchmark BSIT against a diverse set of competitive baselines. These include state-of-the-art tree-based samplers, i.e., BFMT~\citep{ji2026bootstrap}, DTS~\citep{jain2026diffusion}, Feynman-Kac Steering (FKS)~\citep{singhal2025general}, the SMC-based approach DAS~\citep{kim2025test}, and MFM-Base
(FMET)~\citep{potaptchik2026meta}. To simulate feedback, we employ ImageReward~\citep{xu2023imagereward} as the reward model. %To guard against reward hacking, we evaluate the samples with independent metrics, i.e., the VQA score measured with CLIP-FlanT5~\citep{lin2024evaluating} 
To mitigate reward hacking, generated samples are evaluated using independent metrics, namely VQA scores via CLIP-FlanT5~\citep{lin2024evaluating}
and the reward score from ImageReward. To differentiate between the standard and sparse reward settings, we specify the corresponding reward functions as $r_{\text{standard}}$ and $r_{\text{sparse}}$.
\newcommand{\legenditem}[4]{%
  \begin{tikzpicture}[baseline=-0.6ex]
    \draw[#1, line width=1.5pt, #2] (0,0) -- (0.48,0);
    \node[text=#1, fill=white, inner sep=0.4pt] at (0.24,0) {\tiny $#3$};
  \end{tikzpicture}\,#4%
}
\newcommand{\methodlegend}{%
  {\scriptsize
  \setlength{\tabcolsep}{0.5pt}%
  \begin{tabular}{@{}cccccc@{}}
    \legenditem{cyan!70!black}{solid}{\circ}{DAS} &
    \legenditem{red!75!black}{dashed}{\square}{FKS} &
    \legenditem{purple!80!black}{dotted}{\diamond}{DTS} &
    \legenditem{orange!85!black}{dash dot}{\star}{BFMT} &
    \legenditem{green!60!black}{solid}{\blacktriangledown}{\textbf{BSIT}}
    
  \end{tabular}%
  }%
}
\newcommand{\barpanel}[2]{%
  \begin{minipage}[t]{0.48\linewidth}
    \centering
    \includegraphics[width=\linewidth,height=0.16\textheight,keepaspectratio]{#1}\par
    \vspace{-0.3ex}
    {\tiny #2}
  \end{minipage}%
}
\begin{figure*}[h!]
\centering
\begin{subfigure}[t]{0.49\linewidth}
  \centering
  \caption{\bfseries Max ImageReward with $r_{\text{standard}}$}
  \vspace{0.4ex}
  \methodlegend\par
  \vspace{0.4ex}
  \barpanel{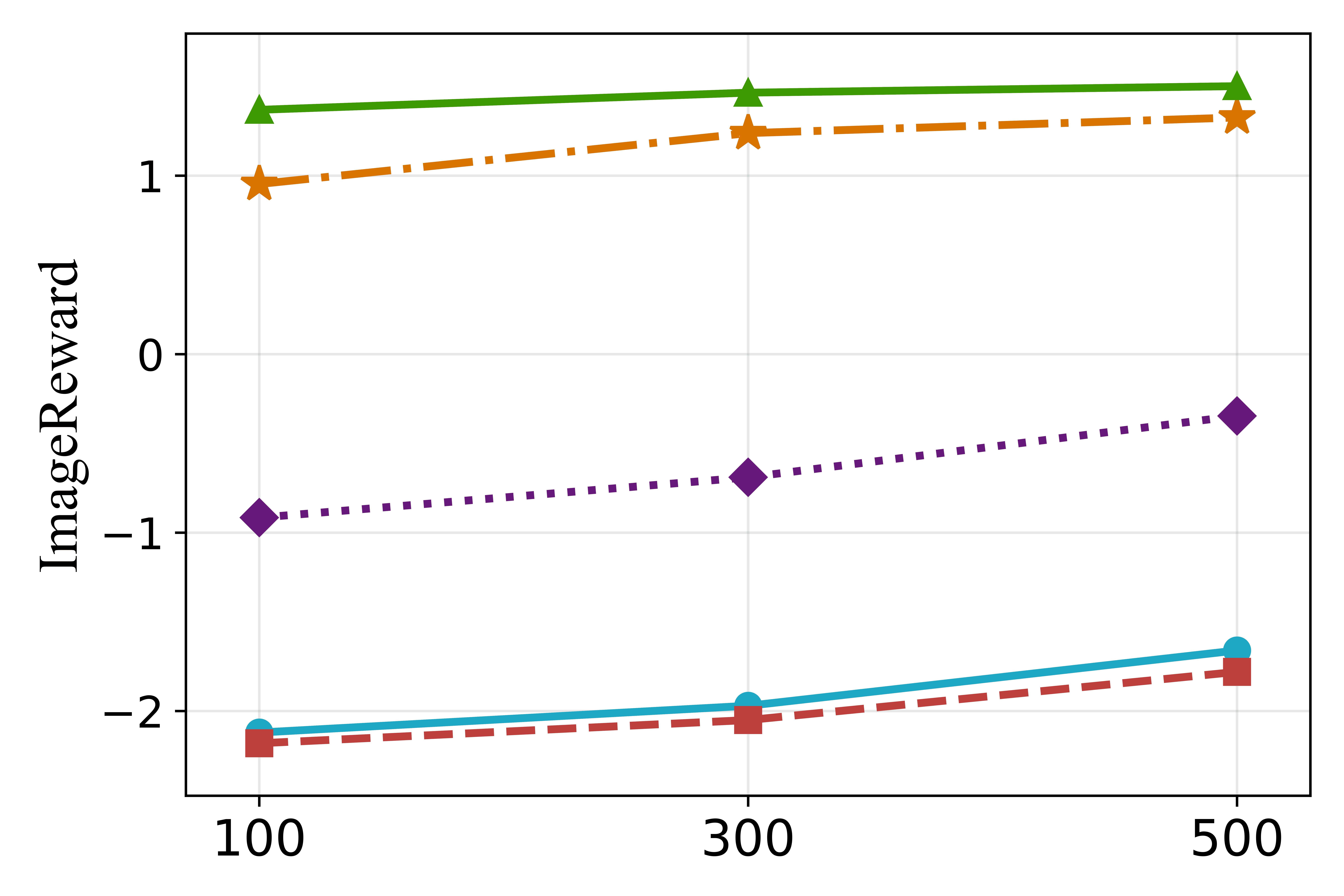}{}\hfill
  \barpanel{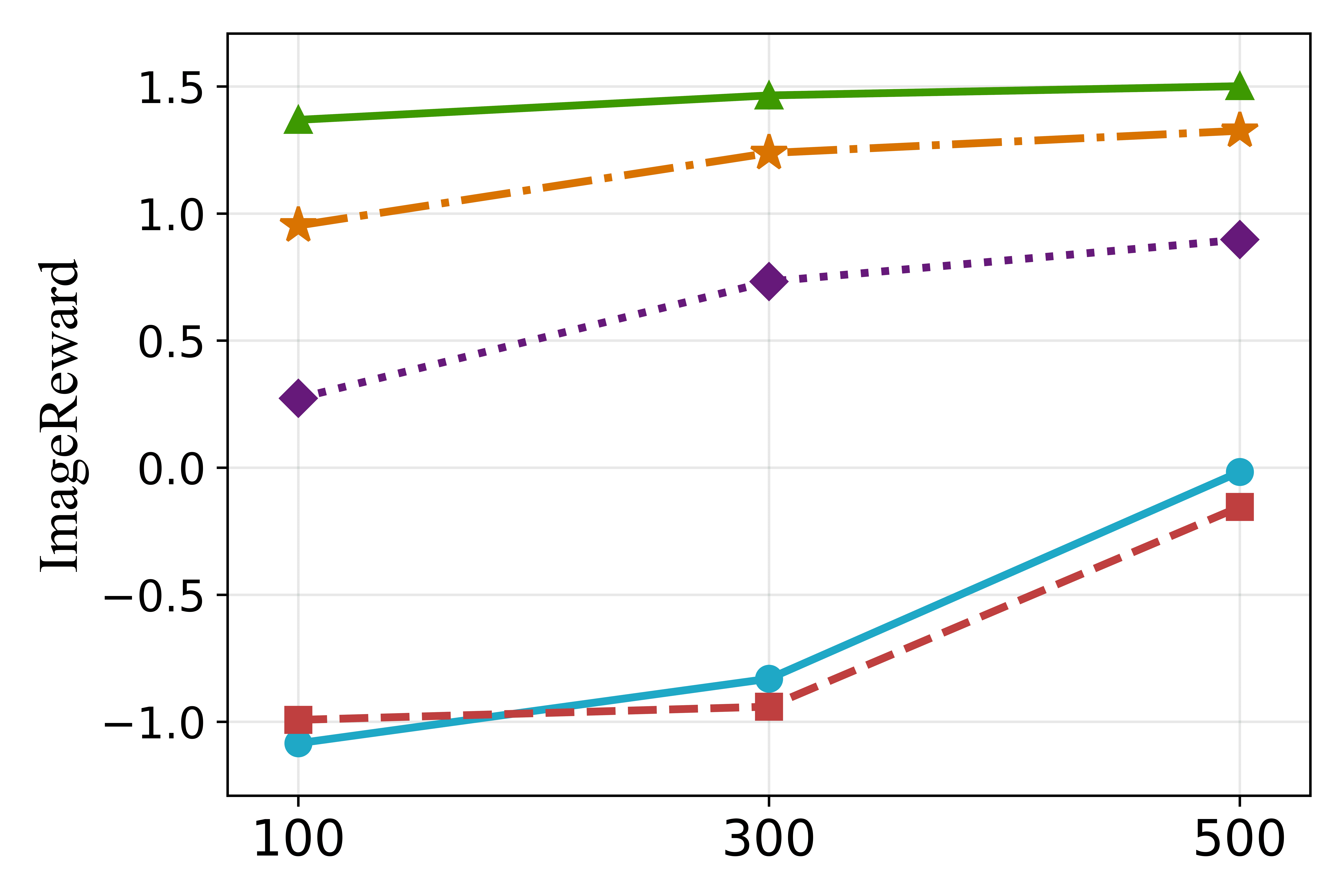}{}
  \label{sub:standard}
\end{subfigure}\hfill%
\begin{subfigure}[t]{0.49\linewidth}
  \centering
  \caption{\bfseries Mean VQA Score with $r_{\text{standard}}$}
  \vspace{0.4ex}
  \methodlegend\par
  \vspace{0.4ex}
  \barpanel{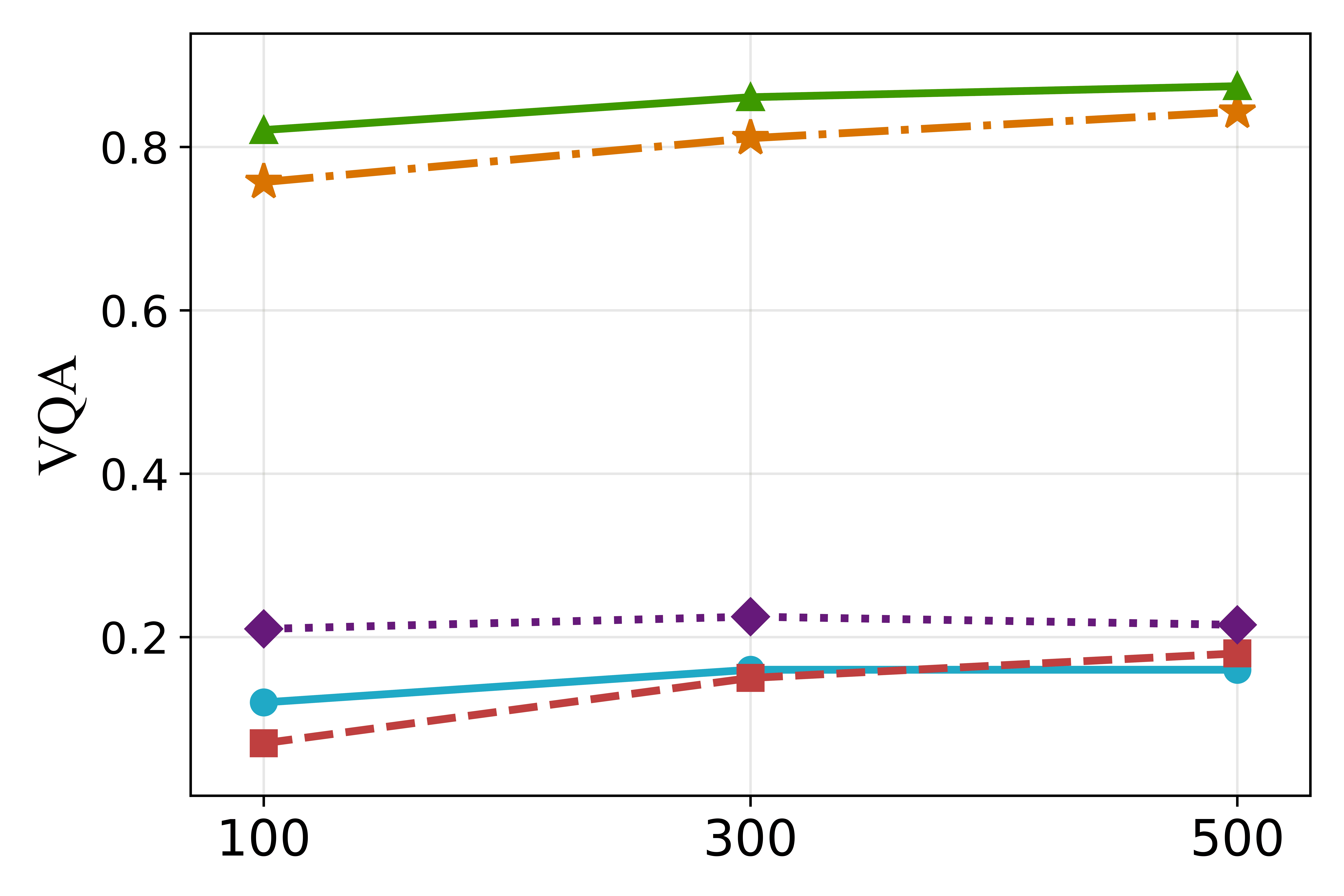}{}\hfill
  \barpanel{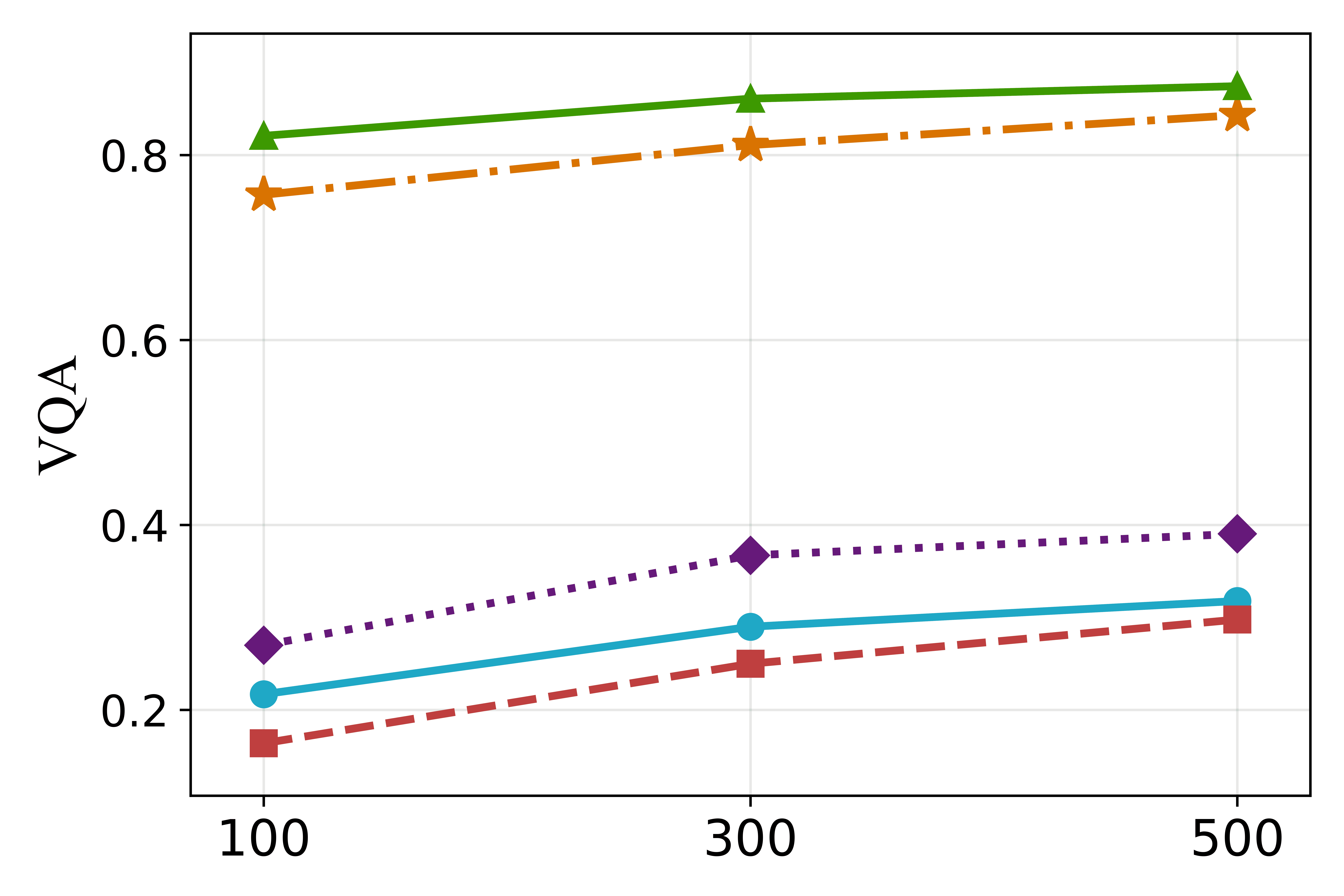}{}
  \label{sub:sparse}
\end{subfigure}
\caption{Additional search performance under fine-grained search setting.}
\label{fig:search-result}
\end{figure*}

\begin{table}[h!]
  \centering
  \small
  \caption{Mean reward comparison across NFE and feedback budget for $\mathbf{r}_{\mathrm{standard}}$, with ImageReward.}
  \label{tab:standard-mean}
  \resizebox{\columnwidth}{!}{%
    \begin{tabular}{lccc|ccc}
      \hline\hline
      Method & \multicolumn{3}{c|}{NFE} &
      \multicolumn{3}{c}{Feedback Budget} \\
       & 100 & 300 & 500 & 100 & 300 & 500 \\
      \hline
      DAS  & -2.2669 & -1.9541 & -2.1804 & -2.1336 & -1.8729 & -1.3416 \\
      FKS  & -2.2773 & -2.1279 & -2.0265 & -2.1727 & -1.6387 & -1.5446 \\
      DTS  & -1.5369 & -1.5308 & -1.3193 & -1.3442 & -0.3381 & -0.1409 \\
      BFMT & 0.2738 & 0.4557 & 0.6165 & 0.2738 & 0.4557 & 0.6165 \\
      \rowcolor{blue!20} BSIT & 0.3269 & 0.5367 & 0.6891 & 0.3269 & 0.5367 & 0.6891 \\
      \hline\hline
    \end{tabular}%
  }
\end{table}

\begin{table}[h!]
  \centering
  \small
  \caption{Max reward comparison across feedback budget for $r_{\mathrm{sparse}}$, with ImageReward.}
  \label{tab:sparse-max}
  \resizebox{\columnwidth}{!}{%
    \begin{tabular}{lccc|ccc}
      \hline\hline
      Method & \multicolumn{3}{c|}{Max} & \multicolumn{3}{c}{Mean} \\
       & 100 & 300 & 500 & 100 & 300 & 500 \\
      \hline
      FMET  & 0.2889 & 0.3042 & 0.5677 & N/A & N/A & N/A \\
      BFMT & 0.0421 & 0.4309 & 0.5782 & -1.0297 & -1.1421 & -1.1162 \\
      \rowcolor{blue!20} BSIT & 1.1229 & 1.2679 & 1.3356 & -0.3926 & -0.4434 & -0.5385 \\
      \hline\hline
    \end{tabular}%
  }
\end{table}
\vspace{-7pt}
\paragraph{Results on ImageNet}
In Fig.~\ref{fig:search-result}, we present quantitative comparisons of VQA and reward scores, using $r_{\text{standard}}$ as the reward, across varying numbers of NFEs and feedback budgets ($\mathcal{B}$). Table~\ref{tab:standard-mean} reports the mean reward comparisons across varying NFEs and feedback budgets ($\mathcal{B}$) using $r_{\text{standard}}$ as the reward function, with ImageReward serving as the proxy reward model.
BSIT outperforms all baselines across varying feedback budgets and NFE constraints, demonstrating both its feedback and compute efficiency in accelerating feedback-driven search. The feedback efficiency stems primarily from its enhanced exploration capabilities relative to prior tree samplers. The computational efficiency stems from using a flow-map to sample the terminal node from any denoising step and by leveraging an iterative stochastic interpolant approach to construct the intermediate nodes of the trajectory, requiring only a single NFE for the entire tree path construction. 
Similarly, under the sparse reward function $r_{\text{sparse}}$ (See Table \ref{tab:sparse-max}), BSIT significantly outperforms all baselines (in terms of both mean and max reward) across all feedback budgets. This highlights its efficacy in generating samples with novel target properties under a strict feedback budget.
%We observe a similar trend when comparing the performance with the maximum reward (i.e., $r_{\text{max}}$) under , as reported in Table 6. 
Overall, these results confirm BSIT’s compute and feedback efficiency in generating target-aligned samples across both standard and sparse reward settings.
Additionally, we provide qualitative comparisons against the most competitive baselines in Fig.~\ref{fig:search-visu}, under both standard and sparse reward settings, with additional experimental details and visualizations in the Appendix.
\begin{figure}[!h]
\centering
\small
\begin{minipage}[t]{0.48\textwidth}
\captionsetup{type=table}
\centering
\caption{\small{Mean Reward Comparison across feedback budget with ``Aphelocoma woodhouseii'' as the target species.}}
\label{tab:sparse-species}
\resizebox{\linewidth}{!}{%
\begin{tabular}{lccc|ccc}
\hline
\hline
Method & \multicolumn{3}{c|}{NFE} & \multicolumn{3}{c}{Feedback Budget} \\
 & 1000 & 3000 & 5000 & 1000 & 3000 & 5000 \\
\hline
%DAS  & TODO & TODO & TODO & TODO & TODO & TODO \\
%FKS  & TODO & TODO & TODO & TODO & TODO & TODO \\
%DTS  & TODO & TODO & TODO & TODO & TODO & TODO \\
DTS  & 0.0600 & 0.09334 & 0.112 & 0.1484 & 0.1709 & 0.1731 \\
BFMT & 0.1516 & 0.1690 & 0.1756 & 0.1516 & 0.1690 & 0.1756 \\
\rowcolor{blue!20} BSIT & 0.1780 & 0.1915 & 0.1987 & 0.1780 & 0.1915 & 0.1987 \\
\hline
\hline
\end{tabular}%
}
\end{minipage}
\end{figure}
%\vspace{-12pt}
\begin{figure}[h!]
\centering
{\setlength{\tabcolsep}{1pt}%
\begin{tabular}{@{}ccc@{}}
\multicolumn{3}{c}{\itshape ``Aphelocoma woodhouseii as target species.''} \\[1pt]
\textbf{Target} & \textbf{BFMT} & \textbf{BSIT} \\
\hline
\includegraphics[width=0.32\columnwidth]{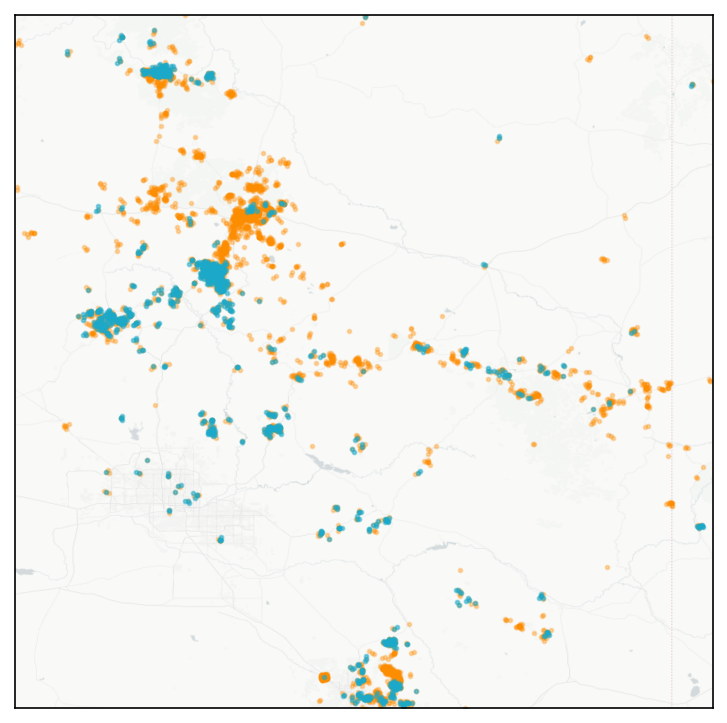} &
\includegraphics[width=0.32\columnwidth]{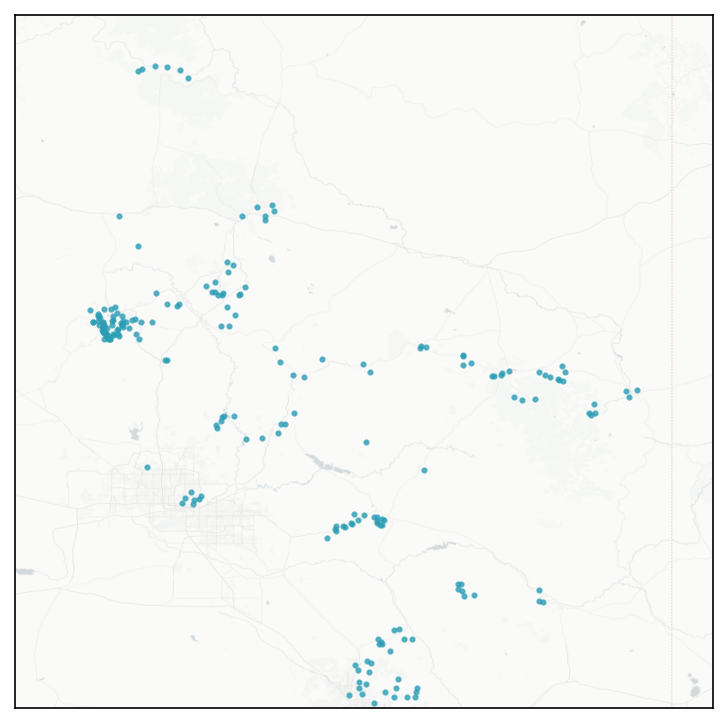} &
\includegraphics[width=0.32\columnwidth]{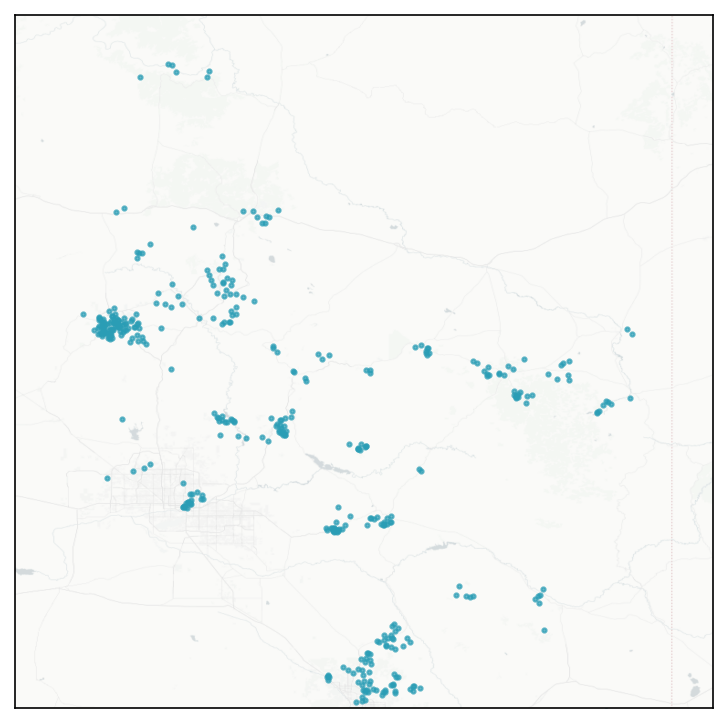}
\end{tabular}%
}
\caption{\small Comparative visualizations of uncovered target regions.}
\label{fig:search-visu-sp}
%\vspace{-12pt}
\end{figure}
%The empirical findings suggest that BSIT consistently achieves higher rewards scores with a fixed feedback budget constraint and also with a fixed NFE than all competing approaches, underscoring its efficacy in accelerating online feedback-driven search. This superior performance is primarily due to its enhanced exploration capability compared to the prior tree samplers.
\vspace{-6pt}
\begin{figure*}[!t]
\centering
\newlength{\molh}\setlength{\molh}{2.0cm}
\newlength{\rowgap}\setlength{\rowgap}{4pt}

% ================= LEFT: 40% width, 2x2 grid (a,b,c,d) =================
\begin{minipage}[c]{0.30\linewidth}
\centering
% ---- Row 1: (a), (b) ----
\begin{minipage}[t]{0.48\linewidth}
\centering
\begin{overpic}[width=0.9\linewidth, height=\molh]{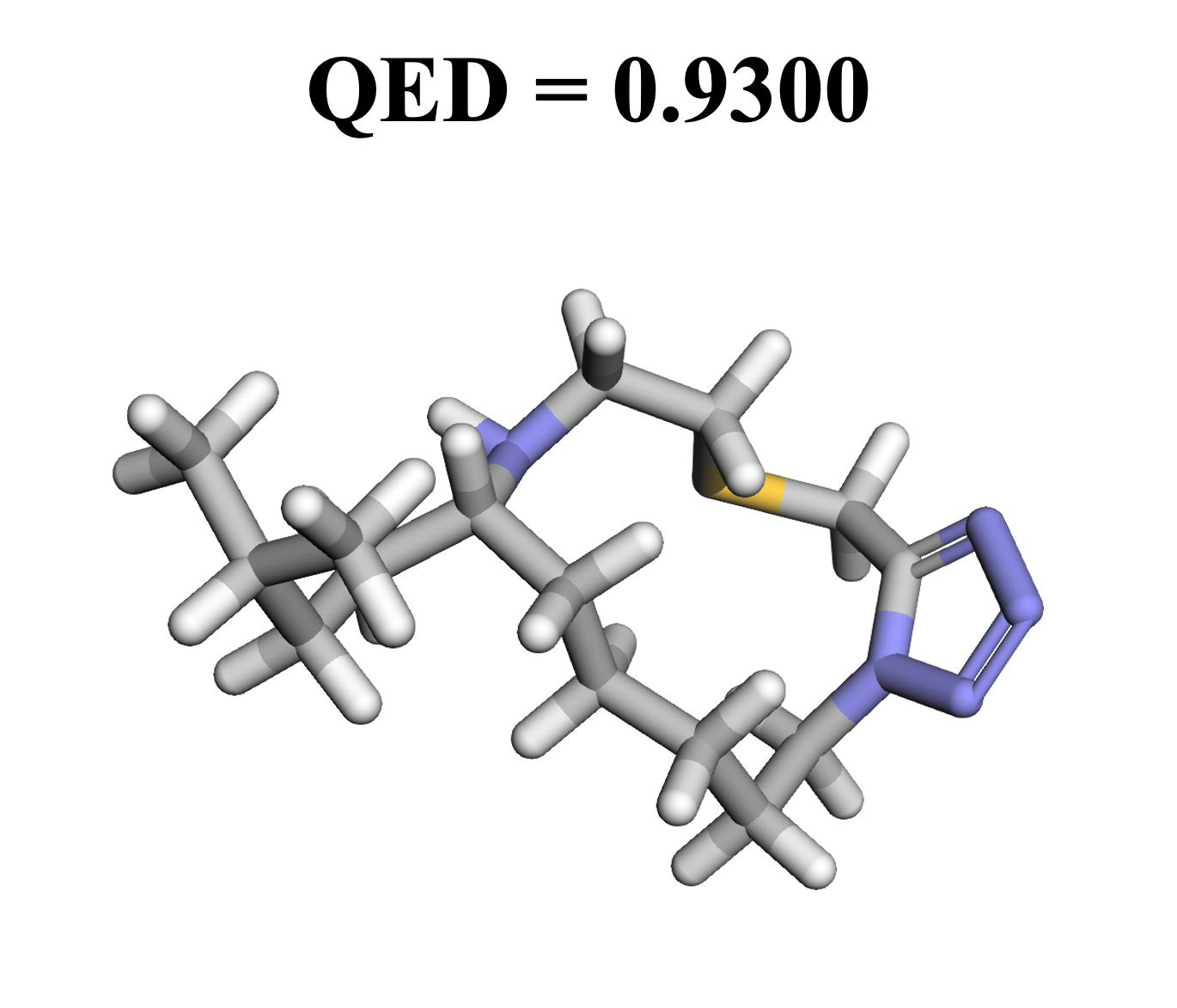}
%\put(2,88){\footnotesize\textbf{(a)}}
\end{overpic}
\end{minipage}\hfill
\begin{minipage}[t]{0.48\linewidth}
\centering
\begin{overpic}[width=0.9\linewidth, height=\molh]{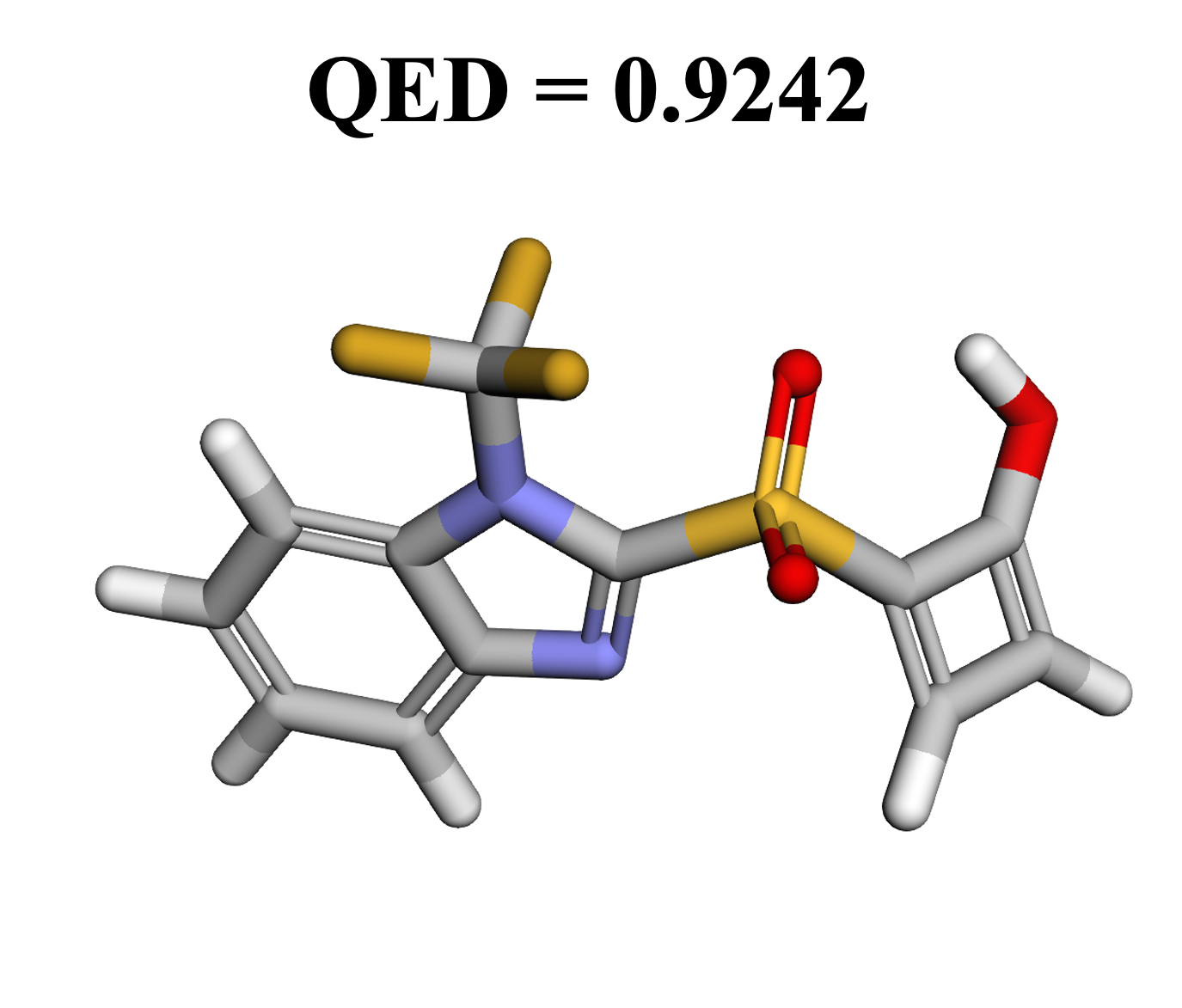}
%\put(2,88){\footnotesize\textbf{(b)}}
\end{overpic}
\end{minipage}

\vspace{\rowgap}

% ---- Row 2: (c), (d) ----
\begin{minipage}[t]{0.48\linewidth}
\centering
\begin{overpic}[width=0.9\linewidth, height=\molh]{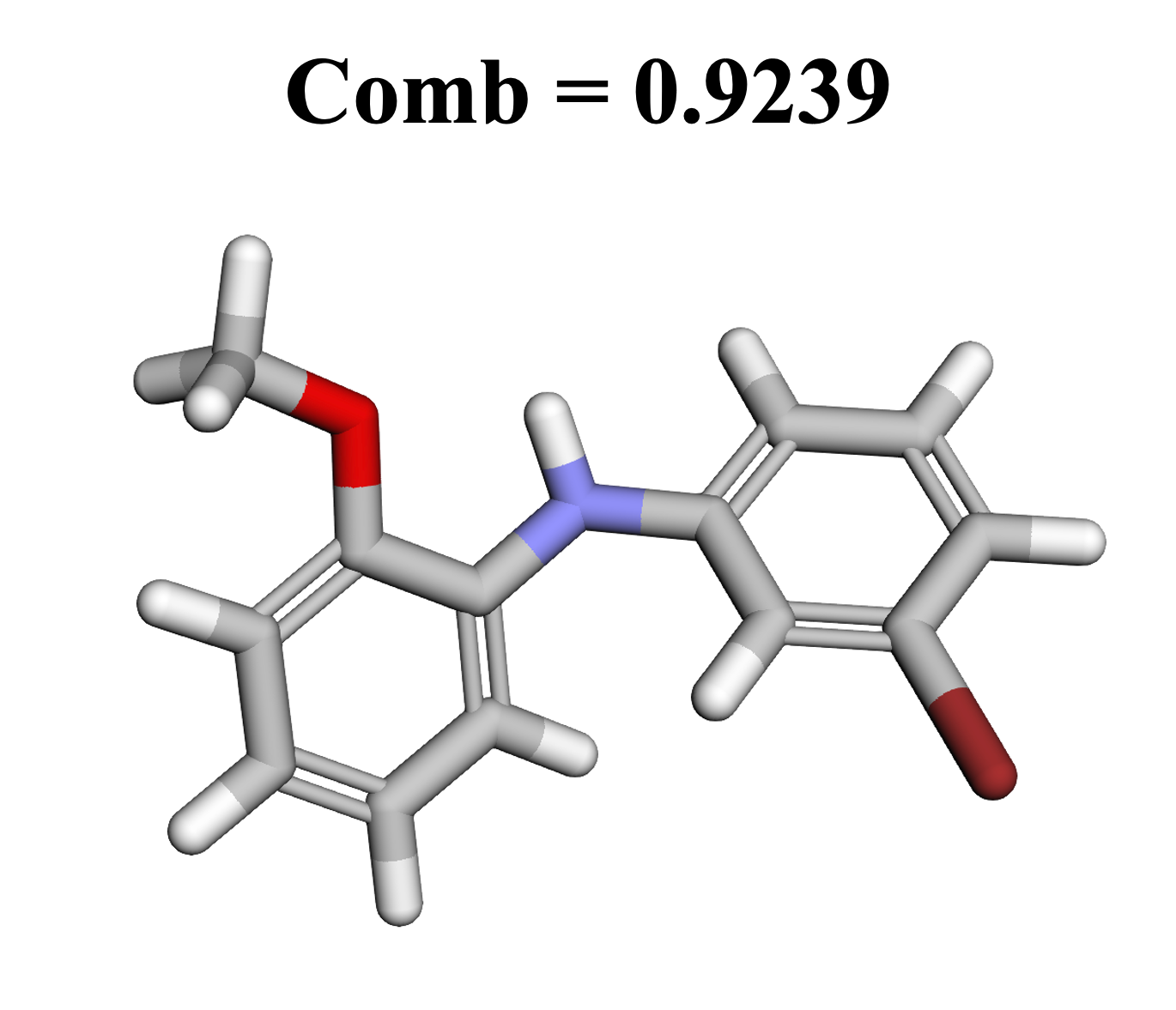}
%\put(2,88){\footnotesize\textbf{(c)}}
\end{overpic}
\end{minipage}\hfill
\begin{minipage}[t]{0.48\linewidth}
\centering
\begin{overpic}[width=0.9\linewidth, height=\molh]{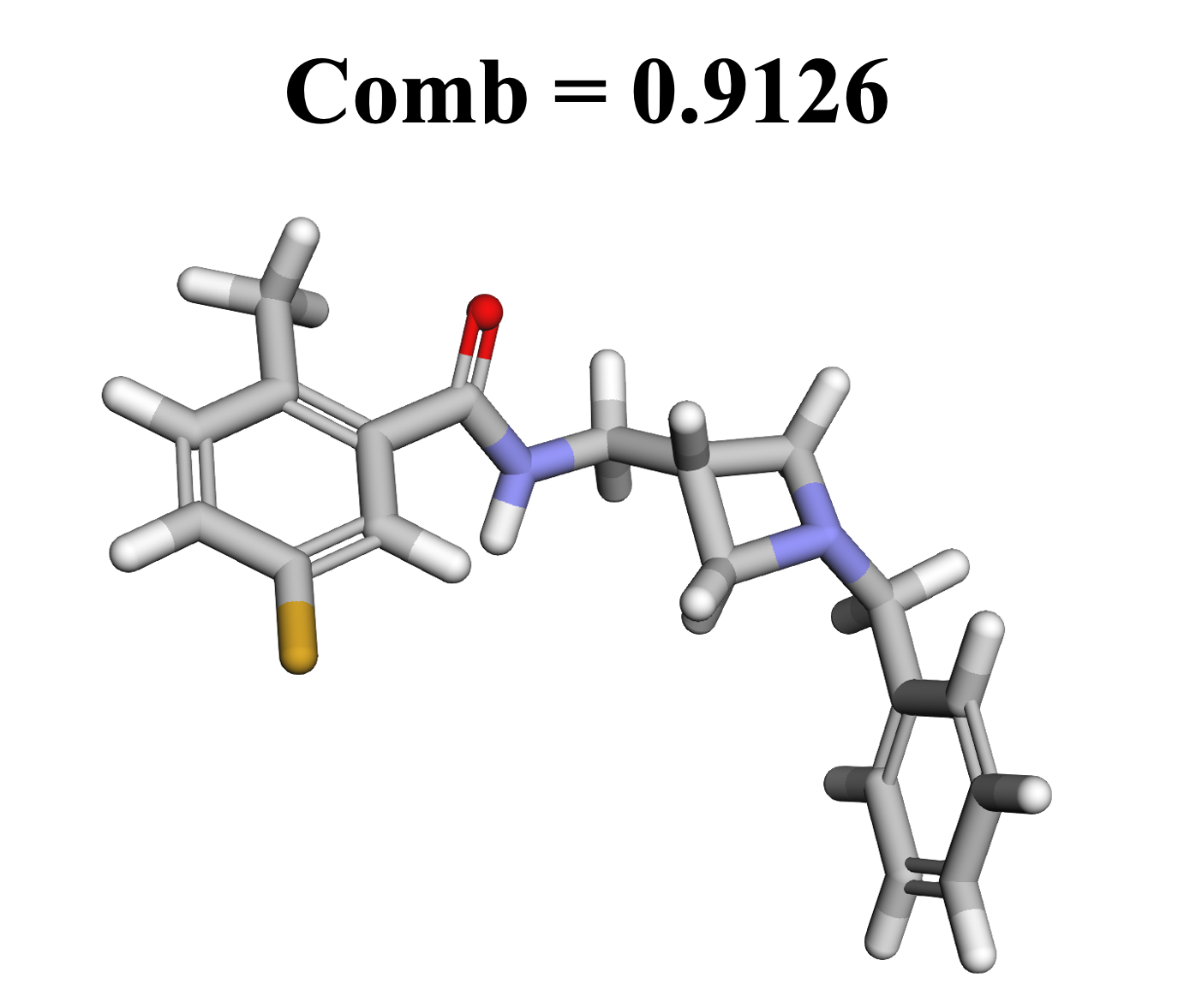}
%\put(2,88){\footnotesize\textbf{(d)}}
\end{overpic}
\end{minipage}
\end{minipage}%
%\hfill
\hspace{2mm}
% ================= MIDDLE: 30% width, (e) =================
\begin{minipage}[c]{0.25\linewidth}
\centering
\begin{overpic}[width=\linewidth, height=\dimexpr 2\molh+\rowgap\relax]{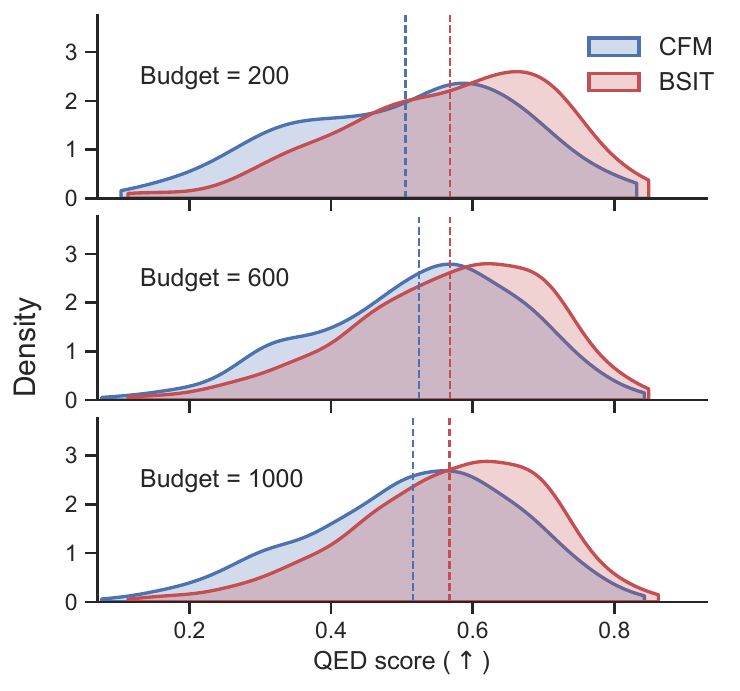}
%\put(2,93){\footnotesize\textbf{(e)}}
\end{overpic}
\end{minipage}%
%\hfill
\hspace{2mm}
% ================= RIGHT: 30% width, (f) =================
\begin{minipage}[c]{0.25\linewidth}
\centering
\begin{overpic}[width=\linewidth, height=\dimexpr 2\molh+\rowgap\relax]{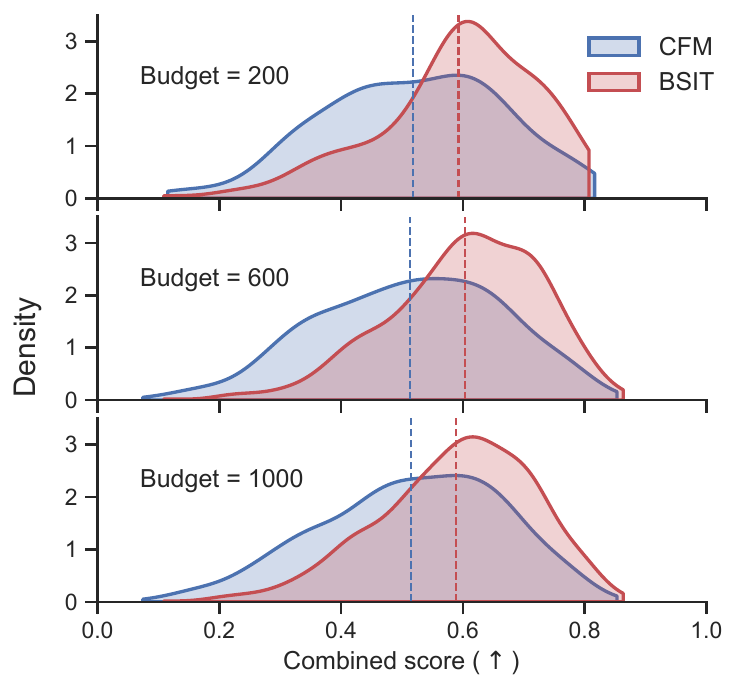}
%\put(2,93){\footnotesize\textbf{(f)}}
\end{overpic}
\end{minipage}
\caption{Effectiveness of BFMT on target-aligned feedback-constrained Molecule Generation Task.}
\label{fig:mol-res}
\end{figure*}
\paragraph{Rare Species Search Task}
%We analyze the effectiveness of BSIT in a practical setting, such as species discovery task. To this end, we create a search task involving a rare bird category named "xyz" as the search target. We train the prior across several bird categories (detailed in Appendix).  Each datapoint in the training set containing the location information (i.e. latitude, longitude) of birds spanning across diverse categories, so the prior model aims to learn the bird's spatial distribution across a broad geospatial area. During inference we aim to sample the regions containing one or instances of a rare target bird categories. We consider the search target to be discovered in a way that resonate with a rare multimodal reward. Every query of a specific lat, lon location reuturns a binary outcome, it is 1 is the surrounding of that particular spot with a radius of xyz contains the target bird caterory else the reward is 0, and the objective of the task is to find out as many local regions as posssible within a given feedback budget. 
To evaluate BSIT in a practical setting, we design a geospatial search task targeting a rare bird species. The prior flow-map model is trained on species location coordinates~\citep{van2018inaturalist} (latitude and longitude) across diverse avian categories to learn their broad spatial distributions (see Appendix). At inference, the goal is to discover local regions hosting instances of the target species under a sparse multi-modal reward structure. 
%\vspace{-8pt}
Querying a location returns a binary reward: $1$ if the target species is present within a fixed radius, and $0$ otherwise. The objective is to maximize the discovery of distinct target-containing regions within a strict feedback budget. We present the results in Table~\ref{tab:sparse-species}. Across various rare-species discovery tasks (see Appendix) and feedback budgets, BSIT achieves significant performance gains, highlighting its effectiveness for sample-efficient search in real-world applications. We present qualitative comparisons in Figure~\ref{fig:search-visu-sp}, where the orange and blue points together form the prior distribution, and the blue points mark $r_{\mathrm{sparse}}$ (Fig.~\ref{fig:search-visu-sp} Target). These visualizations demonstrate that BSIT uncovers many disjoint modes of the target distribution under a fixed feedback budget, highlighting its superior exploration capabilities.
%\vspace{-8pt}

\paragraph{Results with Large-Scale Molecular Structure Generation Task}
%We consider $10$M molecules extracted from PubChem~\citep{kim2023pubchem}. The molecules are represented as SELFIES~\citep{krenn2020self}, a robust string representation for molecular structures. These molecules are used to train Categorical Flow Maps (CFM), a generative flow-matching model developed by~\citep{roos2026categorical}. We trained the CFM model for $16$ epochs using the default hyperparameters and then used the trained model for inference-time search. We subsequently used BFMT to steer generation toward regions of chemical space with high (i) Quantitative Estimate of Drug-likeness (QED) and (ii) a  combination of QED and synthetic accessibility (SA) scores. QED measures a molecule's drug-likeness, while SA estimates how easily it can be synthesized. Our goal is to generate molecules that satisfy these objectives and assess whether BFMT produces a higher proportion of desirable molecules than vanilla CFM sampling, while maintaining comparable validity and uniqueness.
For our experiments, we utilize 10 million molecules from PubChem~\citep{kim2023pubchem} represented as SELFIES~\citep{krenn2020self}. These are used to train Categorical Flow Maps (CFM)~\citep{roos2026categorical} for 16 epochs using default hyperparameters. We then apply BSIT for inference-time search, steering generation toward regions with high Quantitative Estimate of Drug-likeness (QED) and a combined QED and synthetic accessibility (SA) score. This evaluates whether BSIT yields a higher proportion of desirable molecules than vanilla CFM sampling while preserving validity and uniqueness. Fig.~\ref{fig:mol-res} depicts the target property distributions and representative molecule visualizations.
\begin{comment}
\begin{wraptable}{r}{0.56\textwidth}
\centering
\vspace{-10pt}
\caption{Validity Check for the Generated Molecules.}
\vspace{-4pt}
\label{tab:mole}
\footnotesize
\setlength{\tabcolsep}{3pt}
\renewcommand{\arraystretch}{0.85}
\begin{tabular}{lcccc}
\toprule
Method & Uniqueness & Validity & Mean QED & Mean QED+SA \\
\midrule
CFM & todo & todo & todo & todo\\
\rowcolor{blue!20}BFMT & todo & todo & todo & todo\\
\bottomrule
\end{tabular}
\vspace{-8pt}
\end{wraptable}
\end{comment}
%\begin{wraptable}{r}{0.72\textwidth}
\begin{table}
\centering
\vspace{-1pt}
\caption{Validity Check for the Generated Molecules. For both QED and combined (i.e., QED $\&$ SA), we report the mean value.}
\vspace{-4pt}
\label{tab:mole-res}
\footnotesize
\setlength{\tabcolsep}{2.6pt}
\renewcommand{\arraystretch}{0.80}
\begin{tabular}{l ccc|ccc}
\toprule
Method & Validity & Unique. & QED & Validity & Unique. & QED$\&$SA \\
\midrule
CFM & 0.9396 & 0.9968 & 0.4892 & 0.9396 & 0.9968 & 0.5173 \\
%BFMT & 0.9558 & 0.9644 & 0.5610 & 0.9540 & 0.9746 & 0.5746 \\
\rowcolor{blue!20}BSIT & 0.9590 & 0.9520 & \textbf{0.5618} & 0.9483 & 0.9582 & \textbf{0.5815} \\
\bottomrule
\end{tabular}
\vspace{-8pt}
\end{table}
%\end{wraptable}
  Empirically, BSIT consistently surpasses CFM by producing samples that better satisfy the target criteria across all evaluated metrics and feedback budgets. Table~\ref{tab:mole-res} summarizes the mean validity, uniqueness, QED, and combined QED-SA scores for a feedback budget of $1000$. These findings demonstrate that BSIT effectively optimizes target properties without sacrificing chemical validity, establishing it as a highly feedback-efficient sampler for navigating complex discrete spaces like molecular design. Additional results across varying budgets are presented in the Appendix.
\vspace{-7pt}
\paragraph{Conclusion}
We present BSIT, a novel tree-sampling framework designed for fine-grained search constrained by extreme reward sparsity and limited evaluation budgets. By unlocking greater path-space diversity through decoupled trajectory rollouts, BSIT efficiently discovers high-utility regions without requiring prohibitive feedback budgets. We benchmark BSIT across diverse rare-attribute search tasks; BSIT consistently achieves superior feedback and compute efficiency over competitive tree-based baselines. 
\clearpage
%\end{comment}

\bibliography{aaai2027}

\clearpage

%\input{ReproducibilityChecklist.tex}

\appendix
\onecolumn
\section*{Supplementary Material: Bootstrap Stochastic Interpolant Tree for Inference-time Search Under Sparse Reward}

\section{Proof of Proposition 1}
\begin{proof}
\textbf{Step 1:} The soft value estimator for a single trajectory rollout reduces to evaluating the reward at the terminal boundary:
    $$V(x_t) = r(x_0)$$

Let $\mu = \mathbb{E}[x_0 \mid x_t]$. We perform a first-order Taylor expansion of $r(x_0)$ about its mean $\mu$:

$$r(x_0) = r(\mu) + \nabla r(\mu)^T (x_0 - \mu) + R_1(x_0)$$

where $R_1(x_0)$ is the Lagrange remainder satisfying $R_1(x_0) = \mathcal{O}(\Vert{}x_0 - \mu\Vert{}^2)$.

Taking the variance of both sides, the deterministic constant $r(\mu)$ vanishes, leaving:

$$\text{Var}(r(x_0) \mid x_t) = \text{Var}\left( \nabla r(\mu)^T (x_0 - \mu) \right) + \text{Var}(R_1(x_0)) + 2\text{Cov}\left(\nabla r(\mu)^T (x_0 - \mu), R_1(x_0)\right)$$

By properties of the variance operator for linear combinations ($\text{Var}(A^T Y) = A^T \text{Var}(Y) A$), we isolate the primary covariance matrix $\mathbf{\Sigma} = \mathbb{E}[(x_0 - \mu)(x_0 - \mu)^T \mid x_t]$:

$$\text{Var}(r(x_0) \mid x_t) = \nabla r(\mu)^T \mathbf{\Sigma} \nabla r(\mu) + \mathcal{O}(\Vert{}\mathbf{\Sigma}\Vert{}^2)$$

\textbf{Step 2:} Evaluating Covariance Structure Difference:

In the DDPM-style trajectory construction approach (such as Diffusion Tree Sampling~\citep{jain2026diffusion}), the terminal node $x_0$ is reached by simulating a recursive, multi-step reverse-time SDE or discretized ODE chain from $x_t$ to $x_0$. Let $N$ be the number of steps (NFE). The terminal point $x_0$ is computed as:

$$x_0 = \Phi_\theta\Big(\dots \Phi_\theta\big(\Phi_\theta(x_t) + \sigma_t z_t\big) + \dots + \sigma_1 z_1\Big), \quad z_k \sim \mathcal{N}(0, I)$$

where $\Phi_\theta$ represents the non-linear denoiser network drift/score evaluation step. By applying the Law of Total Variance sequentially across the $N$ steps of the trajectory, the variance of the terminal node placement accumulates along the path:

$$\text{Var}_{\text{DDPM}}(x_0 \mid x_t) = \sum_{k=1}^N \mathbb{E}\left[ \left\Vert{} \prod_{j=1}^{k-1} \mathcal{J}_{\Phi}(x_j) \right\Vert{}^2 \sigma_k^2 \right]$$

where $\mathcal{J}_{\Phi}$ is the Jacobian matrix of the denoising neural network step.

Because the reward function $r(x_0)$ is evaluated on a terminal node distorted by this compounding chain of intermediate estimation errors and local noise injections, the variance of the value target $\hat{V}_{\text{DDPM}}(x_t) = \frac{1}{\tau}\log \exp(\tau r(x_0))$ scales with this cumulative trajectory variance.

By following similar structural reasoning,
i.e., the node $x_0^{\text{DDPM}}$ is the output of an $N$-step composition of neural network drift functions $\Phi$. Let $z_k \sim \mathcal{N}(0, I)$ be the independent noise injected at step $k$. By the law of total variance applied recursively:

$$\mathbf{\Sigma}_{\text{DDPM}} = \sum_{k=1}^N \mathbf{J}_k \left( \sigma_k^2 \mathbf{I} \right) \mathbf{J}_k^T$$

where $\mathbf{J}_k = \prod_{j=1}^{k-1} \mathcal{J}_{\Phi}(x_j)$ is the product of the Jacobians from the intermediate denoising steps.

In BSIT, the path construction is completely decoupled. The terminal node $x_0$ is generated in a single non-recursive step directly via the distilled flow-map terminal sampler:

$$x_0 = \hat{v}_{\theta}(z_0; x_t), \quad z_0 \sim \mathcal{N}(0, I)$$

Crucially, although intermediate states $x_s$ are populated using a stochastic interpolant bridge ($x_s = \alpha_s x_0 + \beta_s z$), the intermediate path choices do not retroactively alter or add noise to $x_0$. The terminal reward $r(x_0)$ depends exclusively on a single function evaluation.

The variance of the value estimator under BSIT is bound entirely by the boundary mapping:

$$\text{Var}_{\text{BSIT}}(x_0 \mid x_t) = \text{Var}_{z_0}\big(\hat{v}_{\theta}(z_0; x_t) \mid x_t\big)$$

The node $x_0^{\text{BSIT}}$ is generated in a single step via the distilled flow map $\hat{v}_{\theta}$ using a single base noise source $z_0 \sim \mathcal{N}(0, I)$:

$$\mathbf{\Sigma}_{\text{BSIT}} = \mathcal{J}_{\hat{v}_{\theta}} \left( \sigma_0^2 \mathbf{I} \right) \mathcal{J}_{\hat{v}_{\theta}}^T$$

Because $\mathbf{\Sigma}_{\text{DDPM}}$ sums compounding non-linear step variances scaled by products of intermediate network Jacobians $\mathbf{J}_k$, it strictly dominates the single-step boundary mapping variance under finite evaluation budgets. Specifically, in the next step, we derive the following relation.

\textbf{Step 3: Covariance Relation Between DDPM-style Rollout and Bootstrap Stochastic Interpolant Approach}

$$\mathbf{\Sigma}_{\text{BSIT}} \prec \mathbf{\Sigma}_{\text{DDPM}} \iff v^T \mathbf{\Sigma}_{\text{BSIT}} v < v^T \mathbf{\Sigma}_{\text{DDPM}} v \quad \forall v \in \mathbb{R}^d \setminus \{0\}$$

The statement $\mathbf{\Sigma}_{\text{BSIT}} \prec \mathbf{\Sigma}_{\text{DDPM}}$ means that the matrix difference $\mathbf{\Sigma}_{\text{DDPM}} - \mathbf{\Sigma}_{\text{BSIT}}$ is strictly positive definite. By definition, a symmetric matrix $\mathbf{A}$ is positive definite if and only if its quadratic form is strictly positive for all non-zero vectors $v$:

$$v^T (\mathbf{\Sigma}_{\text{DDPM}} - \mathbf{\Sigma}_{\text{BSIT}}) v > 0 \iff v^T \mathbf{\Sigma}_{\text{BSIT}} v < v^T \mathbf{\Sigma}_{\text{DDPM}} v \quad \forall v \in \mathbb{R}^d \setminus \{0\}$$

Physically, $v^T \mathbf{\Sigma} v$ represents the directional variance of the terminal node along the direction of the vector $v$. Therefore, the matrix inequality simply states that the directional variance of the terminal node is strictly larger in  DDPM-style rollout than in BSIT along every single spatial direction.

Let's analyse the rationale behind the Summation in DDPM-style rollout dominating the Single Term in BSIT.  Let us analyze the structural decomposition of both covariance matrices. In DDPM-style rollout, the final terminal node is a recursive composition of $N$ intermediate steps. Let $\epsilon_k \sim \mathcal{N}(0, \sigma_k^2 \mathbf{I})$ be the independent random noise injected at each reverse denoising step $k$.

By applying the Law of Total Covariance recursively backward from step $N$ to $1$, the total covariance at the terminal step is the sum of the propagated variances from every single intermediate step:

$$\mathbf{\Sigma}_{\text{DDPM}} = \mathbf{J}_1 (\sigma_1^2 \mathbf{I}) \mathbf{J}_1^T + \mathbf{J}_2 (\sigma_2^2 \mathbf{I}) \mathbf{J}_2^T + \dots + \mathbf{J}_N (\sigma_N^2 \mathbf{I}) \mathbf{J}_N^T = \sum_{k=1}^N \mathbf{J}_k \left( \sigma_k^2 \mathbf{I} \right) \mathbf{J}_k^T$$

Where $\mathbf{J}_k = \prod_{j=1}^{k-1} \mathcal{J}_{\Phi_j}$ represents the compounding product of the neural network Jacobians from the boundary down to step $k$.

Each individual matrix component in this sum, $\mathbf{A}_k = \mathbf{J}_k \left( \sigma_k^2 \mathbf{I} \right) \mathbf{J}_k^T$, is a positive semi-definite matrix because for any vector $v$:

$$v^T \mathbf{A}_k v = \sigma_k^2 v^T \mathbf{J}_k \mathbf{J}_k^T v = \sigma_k^2 \Vert{}\mathbf{J}_k^T v\Vert{}^2 \ge 0$$

In BSIT, we bypass the entire intermediate chain. The terminal node is generated using a single base noise source $z_0 \sim \mathcal{N}(0, \sigma_0^2 \mathbf{I})$ passed through the distilled flow-map $\hat{v}_{\theta}$. Its covariance is simply a single localized projection:

$$\mathbf{\Sigma}_{\text{BSIT}} = \mathcal{J}_{\hat{v}_{\theta}} \left( \sigma_0^2 \mathbf{I} \right) \mathcal{J}_{\hat{v}_{\theta}}^T$$

To prove $\mathbf{\Sigma}_{\text{BSIT}} \prec \mathbf{\Sigma}_{\text{DDPM}}$, we evaluate the directional variance difference:

$$v^T \mathbf{\Sigma}_{\text{DDPM}} v - v^T \mathbf{\Sigma}_{\text{BSIT}} v = \sum_{k=1}^N \sigma_k^2 \Vert{}\mathbf{J}_k^T v\Vert{}^2 - \sigma_0^2 \Vert{}\mathcal{J}_{\hat{v}_{\theta}}^T v\Vert{}^2$$

Because distillation objectives train the single-step velocity map $\hat{v}_{\theta}$ to replicate the expectation of the end-to-end integration path, the mapping Jacobian satisfies an amortization property close to the average total trajectory drift:

\begin{equation}
\Vert{}\mathcal{J}_{\hat{v}_{\theta}}^T v\Vert{}^2 \approx \Vert{}\mathbf{J}_1^T v\Vert{}^2
\label{eq:asump-distill}
\end{equation}

Substituting this back into the inequality, we isolate the remaining terms of the sequential summation:

$$v^T \mathbf{\Sigma}_{\text{DDPM}} v - v^T \mathbf{\Sigma}_{\text{BSIT}} v \approx \sum_{k=2}^N \sigma_k^2 \Vert{}\mathbf{J}_k^T v\Vert{}^2$$

Since the denoising steps require non-vanishing noise scales ($\sigma_k^2 > 0$) to maintain stochastic transitions and search capability, and assuming the neural Jacobians $\mathbf{J}_k$ are full rank (non-degenerate network layers), every single remaining term in the summation is strictly positive:

$$\sum_{k=2}^N \sigma_k^2 \Vert{}\mathbf{J}_k^T v\Vert{}^2 > 0 \quad \forall v \neq 0$$

DDPM-style rollout forces the terminal state to inherit the cumulative, additive randomness of $N$ separate, independent stochastic noise injections scaled by deep compounding network layers. BSIT completely deletes the terms $k=2$ through $N$ from the terminal node's ancestry.

As a result, the summation $\mathbf{\Sigma}_{\text{DDPM}}$ strictly dominates the single term $\mathbf{\Sigma}_{\text{BSIT}}$ in the Löwner ordering, proving the relation mathematically.

\textbf{Remark on Assumption to obtain the relation defined in Eqn.~\ref{eq:asump-distill}:}

In the distillation process, the relation $\Vert{}\mathcal{J}_{\hat{v}_{\theta}}^T v\Vert{}^2 \approx \Vert{}\mathbf{J}_1^T v\Vert{}^2$ reflects the core objective of velocity-based teacher distillation.

Here we present a discussion on why the single-step mapping Jacobian $\mathcal{J}_{\hat{v}_{\theta}}$ aligns with the first-step trajectory Jacobian $\mathbf{J}_1$ of the integrated ODE flow:

The goal of the student model is to bypass the sequential multi-step integration of the teacher drift field $b_s(x_s; x_t)$ and directly output the terminal data point in a single leap. By definition, the exact terminal map $\hat{X}_{1,0}(z; x_t)$ satisfies the end-to-end integration flow starting from a reference noise sample $z$ at time $s=1$:

$$\hat{X}_{1,0}(z; x_t) = z + \int_{1}^{0} b_s(x_s; x_t) \, \mathrm{d}s$$

When we differentiate this terminal map with respect to the starting noise coordinate $z$ to see how variations in the initial noise field propagate to the final target, the Leibniz integral rule yields:

$$\mathcal{J}_{\hat{X}_{1,0}} = \frac{\partial \hat{X}_{1,0}(z; x_t)}{\partial z} = \mathbf{I} + \int_{1}^{0} \frac{\partial b_s(x_s; x_t)}{\partial x_s} \frac{\partial x_s}{\partial z} \, \mathrm{d}s$$

Now consider the DDPM-style sequential trajectory formulation (e.g., DTS and BFMT). The first discrete integration step moving backwards away from the noise boundary $s=1$ down toward an intermediate state $s_1$ is given by:

$$x_{s_1} = z - h \cdot b_1(z; x_t)$$

The Jacobian of this very first structural step with respect to the initial noise position $z$ is:

$$\mathbf{J}_1 = \frac{\partial x_{s_1}}{\partial z} = \mathbf{I} - h \frac{\partial b_1(z; x_t)}{\partial z}$$

The Path Amortization Equivalence: Because the consistency loss $\mathcal{L}^{\text{distill}}_{\text{cons}}$ explicitly forces the student model to minimize the mean squared error against the instantaneous velocity field across the entire continuous interval $[1 \to 0]$, the student's single-step mapping gradient $\mathcal{J}_{\hat{v}_{\theta}}$ acts as an amortized average of the continuous path derivatives.

Under an optimized distillation objective, the single global map pushes the initial noise variations forward exactly like the first physical step $\mathbf{J}_1$ of the true velocity trajectory leaving the noise boundary. Therefore, their directional derivatives align tightly:

$$\Vert{}\mathcal{J}_{\hat{v}_{\theta}}^T v\Vert{}^2 \approx \Vert{}\mathbf{J}_1^T v\Vert{}^2$$

This equivalence highlights why the remaining terms ($\mathbf{J}_2$ through $\mathbf{J}_N$) in the DDPM-style rollout summation are completely eliminated by the BSIT framework. While DDPM-style rollout must accumulate additional noise projections at every subsequent step down the chain, BSIT captures the entire mapping variance right at the initial boundary.

\textbf{Step 4:} Deriving the Final Result:

Let $v = \nabla r(\mu)$. Assuming the reward landscape has a non-vanishing gradient at the expected terminal location ($\nabla r(\mu) \neq 0$), we substitute $v$ directly into our quadratic forms:

$$\nabla r(\mu)^T \mathbf{\Sigma}_{\text{BSIT}} \nabla r(\mu) < \nabla r(\mu)^T \mathbf{\Sigma}_{\text{DDPM}} \nabla r(\mu) \quad \text{Utilizing the results derived in Step 3}$$

Applying our Taylor variance expansion from Step 1:

$$\text{Var}(\hat{V}_{\text{BSIT}}(x_t)) - \mathcal{O}(\Vert{}\mathbf{\Sigma}_{\text{BSIT}}\Vert{}^2) < \text{Var}(\hat{V}_{\text{DDPM}}(x_t)) - \mathcal{O}(\Vert{}\mathbf{\Sigma}_{\text{DDPM}}\Vert{}^2)$$

For localized, highly tracking models where the error remainders are bounded, this structurally guarantees:

$$\text{Var}(\hat{V}_{\text{BSIT}}(x_t)) < \text{Var}(\hat{V}_{\text{DDPM}}(x_t)) \quad \blacksquare$$

\end{proof}

\section{Proof of Proposition 2}
Let $x_t \in \mathbb{R}^d$ be a fixed state at time $t$, and let $s \in (0, t)$ be an intermediate time step along the reverse generative process. We analyze the differential entropy $H(x_s \mid x_t) = -\mathbb{E}_{x_s \sim p(x_s \mid x_t)} [\log p(x_s \mid x_t)]$ under both sampling paradigms, i.e., BSIT and DDPM-style rollout (DTS, BFMT).

\textbf{Step 1:} Conditional Entropy under BSIT framework: 

In BSIT, given $x_t$, the intermediate node state $x_s$ is constructed via the linear stochastic interpolant:

$$x_s = \alpha_s x_0 + \beta_s z_s \quad \text{where, } x_0 \sim p_{\hat{v}_{\theta}(\epsilon; x_t)}(x_0 \mid x_t)$$

where $z_s \sim \mathcal{N}(0, \mathbf{I}_d)$ is independent of $x_t$ and $x_0$, and $x_0$ is the single-step meta-flow-map~\citep{potaptchik2026meta} model prediction mapping $x_t$ (and initial boundary noise $z_0$) to data space $\mathbb{R}^d$.

By the law of total variance conditioned on $x_t$:

$$\operatorname{Cov}(x_s \mid x_t) = \operatorname{Cov}_{x_0}\left( \mathbb{E}_{z_s}[x_s \mid x_0, x_t] \;\middle\vert{}\; x_t \right) + \mathbb{E}_{x_0} \left[ \operatorname{Cov}_{z_s}(x_s \mid x_0, x_t) \;\middle\vert{}\; x_t \right]$$

 By utilizing the definition of stochastic interpolant, we can write: $\mathbb{E}_{z_s}[x_s \mid x_0, x_t] = \alpha_s x_0 + \beta_s \mathbb{E}[z_s] = \alpha_s  x_0$.
 
Therefore, the first term evaluates to:

$$\operatorname{Cov}_{ x_0}\left( \mathbb{E}_{z_s}[x_s \mid  x_0, x_t] \;\middle\vert{}\; x_t \right) = \alpha_s^2 \operatorname{Cov}( x_0 \mid x_t)$$

Similarly, for the second term, the inner covariance evaluates to: $\operatorname{Cov}_{z_s}(x_s \mid \hat{x}_0, x_t) = \operatorname{Cov}_{z_s}(\alpha_s \hat{x}_0 + \beta_s z_s \mid \hat{x}_0) = \beta_s^2 \mathbf{I}_d$.

Therefore, the second term evaluates to:

$$\mathbb{E}_{\hat{x}_0}\left[ \beta_s^2 \mathbf{I}_d \;\middle\vert{}\; x_t \right] = \beta_s^2 \mathbf{I}_d$$

Hence, the total Covariance Matrix $\mathbf{\Sigma}_{\text{BSIT}}$ under the BSIT trajectory construction scheme can be expressed as:

$$\mathbf{\Sigma}_{\text{BSIT}} \triangleq \operatorname{Cov}(x_s \mid x_t) = \alpha_s^2 \operatorname{Cov}(\hat{x}_0 \mid x_t) + \beta_s^2 \mathbf{I}_d$$

\textbf{Step 2:} Lower Bounding Differential Entropy via Entropy Power Inequality (EPI):
Since $x_s = \alpha_s \hat{x}_0 + \beta_s z_s$ is the sum of two independent random vectors (conditioned on $x_t$), we apply the Entropy Power Inequality:

$$N(x_s \mid x_t) \ge N(\alpha_s \hat{x}_0 \mid x_t) + N(\beta_s z_s \mid x_t)$$

where $N(X) = \frac{1}{2\pi e} \exp\left(\frac{2}{d} H(X)\right)$ is the entropy power.
Because $z_s \sim \mathcal{N}(0, \mathbf{I}_d)$ is purely Gaussian:

$$H(\beta_s z_s) = \frac{d}{2} \log (2\pi e \beta_s^2) \implies N(\beta_s z_s) = \beta_s^2$$

Since $N(\alpha_s x_0 \mid x_t) \ge 0$, we have:

$$N(x_s \mid x_t) \ge \beta_s^2 \implies H_{\text{BSIT}}(x_s \mid x_t) \ge \frac{d}{2} \log (2\pi e \beta_s^2)$$

\textbf{Step 3:} Upper Bounding Path Entropy of DDPM-style trajectory rollout sampler (e.g., DTS, BFMT)

In standard continuous-time reverse SDEs / discretized DDPMs, the step-wise conditional transition from $x_t$ down to $x_s$ over a small step $\Delta t = t - s$ follows:

$$x_s = x_t - f(x_t, t)\Delta t + g(t)\sqrt{\Delta t}\,\eta, \quad \eta \sim \mathcal{N}(0, \mathbf{I}_d)$$

\textbf{Evaluating Local Transition Covariance:} To prevent trajectory discretization error from exploding, local drift evaluations require a small diffusion coefficient $g(t)^2 \Delta t = \sigma^2_{\text{step}} \ll \beta_s^2$.
The conditional covariance of the transition $p_{\text{Diff}}(x_s \mid x_t)$ is bounded by:

\begin{equation}
\operatorname{Cov}_{\text{DDPM}}(x_s \mid x_t) \preceq \sigma^2_{\text{bounded}} \mathbf{I}_d
\label{eq:ddpm-lo}
\end{equation}

where $\sigma^2_{\text{bounded}} \le \mathcal{O}(\Delta t) \cdot \lambda_{\max}(\mathbf{J}_{\text{score}})$. Next, we present the derivation of Equation~\ref{eq:ddpm-lo}.

To this end, we evaluate the local variability of the drift $f(x, t)$ around $x_t$ due to local perturbations or trajectory uncertainty $\delta x$; we expand $f(x_t + \delta x, t)$ using its first-order Taylor expansion:

$$f(x_t + \delta x, t) \approx f(x_t, t) + \nabla_x f(x_t, t) \, \delta x$$

Since the score function $s_\theta(x_t, t)$ dominates the nonlinear spatial gradients of the drift vector field, its Jacobian matrix is defined as:

$$\mathbf{J}_{\text{score}} \triangleq \nabla_x s_\theta(x_t, t) = \nabla_x \left( \nabla_x \log p_t(x_t) \right)$$

Thus, any local variation in the spatial state $\delta x$ causes a variation in the drift term given by:

$$\delta f \approx \mathbf{J}_{\text{score}} \, \delta x$$

In the update equation:

$$x_s = x_t - \Delta t \cdot f(x_t, t) + g(t)\sqrt{\Delta t}\,\eta$$

The drift term is scaled by $\Delta t$. Therefore, the linearized variation in $x_s$ induced by the drift's spatial gradient is:

$$\delta x_s^{(\text{drift})} = -\Delta t \cdot \left( \mathbf{J}_{\text{score}} \, \delta x \right)$$

Applying the covariance operator $\operatorname{Cov}(\mathbf{A} Y) = \mathbf{A} \operatorname{Cov}(Y) \mathbf{A}^T$ to this drift variation term:

$$\operatorname{Cov}\left( \delta x_s^{(\text{drift})} \right) = (-\Delta t \mathbf{J}_{\text{score}}) \cdot \operatorname{Cov}(\delta x) \cdot (-\Delta t \mathbf{J}_{\text{score}})^T$$

Pulling out $(-\Delta t)$ from both sides produces a factor of $(-\Delta t)^2 = \Delta t^2$:

$$\operatorname{Cov}\left( \delta x_s^{(\text{drift})} \right) = \Delta t^2 \, \mathbf{J}_{\text{score}} \operatorname{Cov}(\delta x) \mathbf{J}_{\text{score}}^T$$

Since $\operatorname{Cov}(\delta x)$ is on the order of $\mathcal{O}(1)$ or bounded locally, this contribution simplifies under Big-O notation to:

$$\mathcal{O}(\Delta t^2) \, \mathbf{J}_{\text{score}} \mathbf{J}_{\text{score}}^T$$

Because the injected Brownian noise vector $\eta \sim \mathcal{N}(0, \mathbf{I}_d)$ is independent of the local state perturbations $\delta x$, their covariances add directly:

$$\operatorname{Cov}_{\text{DDPM}}(x_s \mid x_t) = \underbrace{\operatorname{Cov}\left( g(t)\sqrt{\Delta t} \, \eta \right)}_{\text{Brownian Noise Term}} + \underbrace{\operatorname{Cov}\left( \delta x_s^{(\text{drift})} \right)}_{\text{Linearized Drift Term}}$$

$$\operatorname{Cov}_{\text{DDPM}}(x_s \mid x_t) = \sigma_{\text{step}}^2 \mathbf{I}_d + \mathcal{O}(\Delta t^2) \mathbf{J}_{\text{score}} \mathbf{J}_{\text{score}}^T$$

To show that a symmetric covariance matrix $\mathbf{A}$ is bounded by an isotropic scaling of the identity matrix $\mathbf{A} \preceq c \mathbf{I}_d$, it suffices to show that the largest eigenvalue of $\mathbf{A}$ is less than or equal to $c$:$$\mathbf{A} \preceq \lambda_{\max}(\mathbf{A}) \mathbf{I}_d$$

From our local step expansion, the transition covariance is:$$\operatorname{Cov}_{\text{DDPM}}(x_s \mid x_t) = g(t)^2 \Delta t \mathbf{I}_d + \Delta t^2 \mathbf{J}_{\text{score}} \mathbf{J}_{\text{score}}^T$$

Applying the eigenvalue bound to the matrix term $\mathbf{J}_{\text{score}} \mathbf{J}_{\text{score}}^T$:

$$\mathbf{J}_{\text{score}} \mathbf{J}_{\text{score}}^T \preceq \lambda_{\max}\left(\mathbf{J}_{\text{score}} \mathbf{J}_{\text{score}}^T\right) \mathbf{I}_d = \lambda_{\max}^2(\mathbf{J}_{\text{score}}) \mathbf{I}_d$$

Substituting this back gives:$$\operatorname{Cov}_{\text{DDPM}}(x_s \mid x_t) \preceq \Big( g(t)^2 \Delta t + \Delta t^2 \lambda_{\max}^2(\mathbf{J}_{\text{score}}) \Big) \mathbf{I}_d$$

Defining $\sigma_{\text{bounded}}^2:$
We define the scalar variance bound $\sigma_{\text{bounded}}^2$ as the coefficient in front of $\mathbf{I}_d$:

$$\sigma_{\text{bounded}}^2 \triangleq g(t)^2 \Delta t + C \cdot \Delta t^2 \cdot \lambda_{\max}^2(\mathbf{J}_{\text{score}})$$

This gives the matrix inequality directly:
$$\operatorname{Cov}_{\text{DDPM}}(x_s \mid x_t) \preceq \sigma_{\text{bounded}}^2 \mathbf{I}_d$$

Big-O Order Reduction in $\Delta t$:

Both terms in $\sigma_{\text{bounded}}^2$ depend on the discretization step size $\Delta t$: Diffusion coefficient term: $g(t)^2 \Delta t = \mathcal{O}(\Delta t)$

Drift Jacobian term: $C \cdot \Delta t^2 \lambda_{\max}^2(\mathbf{J}_{\text{score}}) = \mathcal{O}(\Delta t^2) \lambda_{\max}^2(\mathbf{J}_{\text{score}})$

When $\Delta t \ll 1$ is small, the overall scalar expression is dominated by terms of order $\Delta t$. Factoring out $\Delta t$:
$$\sigma_{\text{bounded}}^2 = \Delta t \cdot \left[ g(t)^2 + C \Delta t \, \lambda_{\max}^2(\mathbf{J}_{\text{score}}) \right]$$

In Big-O asymptotic notation (absorbing constant diffusion bounds $g(t)^2$ and network smoothness coefficients), the dominant growth rate of this single-step upper bound simplifies to:$$\sigma_{\text{bounded}}^2 \le \mathcal{O}(\Delta t) \cdot \lambda_{\max}(\mathbf{J}_{\text{score}})$$

\textbf{Step 5:} Applying Maximum Entropy Principle: By Shannon's Theorem, among all multivariate continuous distributions with a given covariance matrix $\mathbf{\Sigma}$, the Gaussian distribution $\mathcal{N}(\mu, \mathbf{\Sigma})$ uniquely maximizes differential entropy:

$$H(X) \le \frac{1}{2} \log \det (2\pi e \mathbf{\Sigma})$$

Applying this upper bound to standard sequential diffusion (in DTS and BFMT) yields:

$$H_{\text{DDPM}}(x_s \mid x_t) \le \frac{1}{2} \log \det \left( 2\pi e \cdot \sigma^2_{\text{bounded}} \mathbf{I}_d \right) = \frac{d}{2} \log (2\pi e \sigma^2_{\text{bounded}})$$

\textbf{Step 6:} Derivation of final result (Strict Path Entropy Domination): 

Subtracting the upper bound of standard diffusion from the lower bound of BSIT:

$$H_{\text{BSIT}}(x_s \mid x_t) - H_{\text{DDPM}}(x_s \mid x_t) \ge \frac{d}{2} \log (2\pi e \beta_s^2) - \frac{d}{2} \log (2\pi e \sigma^2_{\text{bounded}}) = \frac{d}{2} \log \left( \frac{\beta_s^2}{\sigma^2_{\text{bounded}}} \right)$$

For any non-trivial discretization step $N \gg 1$, the cumulative single-step variance of continuous SDE integration satisfies $\sigma^2_{\text{bounded}} \ll \beta_s^2$. Therefore:

$$\frac{\beta_s^2}{\sigma^2_{\text{bounded}}} > 1 \implies \log \left( \frac{\beta_s^2}{\sigma^2_{\text{bounded}}} \right) > 0$$

$$\mathbf{H_{\text{BSIT}}(x_s \mid x_t) > H_{\text{DDPM}}(x_s \mid x_t)}$$

\textbf{Remark:}

\textbf{Standard Diffusion Mode Collapse:} Sequential step-wise sampling forces $x_s$ into an extremely narrow tubular neighborhood around the ODE/SDE trajectory, driving local conditional entropy $H_{\text{DDPM}}(x_s \mid x_t)$ to shrink as step sizes refine.

\textbf{BSIT Mode Preservation:} Because BSIT injects an independent global boundary noise term $\beta_s z_s$ directly at level $s$, the conditional entropy $H_{\text{BSIT}}(x_s \mid x_t)$ remains explicitly lower-bounded by $\frac{d}{2} \log(2\pi e \beta_s^2)$, guaranteeing that the path space retains sufficient volume and prevents search degeneration during online tree exploration.

\section{Proof of Proposition 3}

\begin{proof}
Let's denote $p(\cdot \mid \mathbf{x}_t)$ to represent a general proposal distribution. Under the BSIT framework, this would correspond to transitions under the pretrained flow-map model $p_{\text{BSIT}}(\cdot \mid \mathbf{x}_t)$. The mechanism of $p_{\text{BSIT}}(\cdot \mid \mathbf{x}_t)$ is detailed in Equation 8 of the main paper. Additionally, we use $\hat{q}(\cdot \mid \mathbf{x}_t)$ to denote the transition density of BSIT.

\paragraph{Step 1: Transition probability under BSIT.} 
Under BSIT, given a node $\mathbf{x}_t$, we create each child by first generating the terminal node sample $x_0$ via the flow-map model $\hat{v}_{\theta}(\epsilon; x_t)$, and then applying stochastic interpolation between $x_0$ and noise $\epsilon \sim \mathcal{N}(0, I)$ (as specified in Eqn. 8 in the main paper). During tree traversal, we select the next state $\mathbf{x}_{t-1}$ proportional to the exponentiated soft value function. Thus, the transition probability of BSIT from $\mathbf{x}_t$ to $\mathbf{x}_{t-1}$ is given by:
\begin{equation}
\hat{q}(\mathbf{x}_{t-1} \mid \mathbf{x}_t) = \frac{p_{\text{BSIT}}(\mathbf{x}_{t-1} \mid \mathbf{x}_t) \exp \big(\tau V(\mathbf{x}_{t-1})\big)}{\int p_{\text{BSIT}}(\mathbf{x}_{t-1} \mid \mathbf{x}_t) \exp \big(\tau V(\mathbf{x}_{t-1})\big) \, d\mathbf{x}_{t-1}} = \frac{p_{\text{BSIT}}(\mathbf{x}_{t-1} \mid \mathbf{x}_t) \exp \big(\tau V(\mathbf{x}_{t-1})\big)}{\exp \big(\tau V(\mathbf{x}_t)\big)},
\end{equation}
where the second equality follows from the definition of the soft Bellman equation satisfied by our proposed BSIT sampler (see the discussion in the Method section):
\begin{equation}
V_t(x_t) = \frac{1}{\tau} \log \int p_{\text{BSIT}}(x_{t-1} \mid x_t) \exp\left( \tau \cdot V(x_{t-1}) \right) \mathrm{d}x_{t-1}
\label{eq:bellman-app}
\end{equation}
where $V(x_{t-1} ) = \frac{1}{\tau} \log \mathbb{E}_{x_0 \mid x_{t-1}} [\exp(\tau r(x_0))]$ is the soft value estimate of the child node of $x_t$.

\paragraph{Step 2: Joint density of trajectory.} 
In practice, the root node of BSIT contains a dummy state $\mathbf{x}_{T+1}$ that transitions to the diffusion process prior $\hat{q}(\mathbf{x}_T \mid \mathbf{x}_{T+1}) = \mathcal{N}(0, I)$. Then, the joint density of a full trajectory $\{\mathbf{x}_T, \mathbf{x}_{T-1}, \dots, \mathbf{x}_0\}$ under BSIT is given by:
\begin{align*}
\hat{q}(\mathbf{x}_T, \mathbf{x}_{T-1}, \dots, \mathbf{x}_0) &= \prod_{t=1}^{T+1} \hat{q}(\mathbf{x}_{t-1} \mid \mathbf{x}_t) \\
&= \prod_{t=1}^{T+1} \frac{p_{\text{BSIT}}(\mathbf{x}_{t-1} \mid \mathbf{x}_t) \exp \big(\tau V(\mathbf{x}_{t-1})\big)}{\exp \big(\tau V(\mathbf{x}_t)\big)} \\
&= \frac{\exp \big(\tau V(\mathbf{x}_0)\big)}{\exp \big(\tau V(\mathbf{x}_{T+1})\big)} \prod_{t=1}^{T+1} p_{\text{BSIT}}(\mathbf{x}_{t-1} \mid \mathbf{x}_t) \\
&= \frac{\exp \big(\lambda \hat{v}(\mathbf{x}_0)\big)}{\exp \big(\lambda \hat{v}(\mathbf{x}_{T+1})\big)} p_{\text{BSIT}}(\mathbf{x}_T, \mathbf{x}_{T-1}, \dots, \mathbf{x}_0).
\end{align*}

\paragraph{Step 3: Marginalizing.} 
Marginalizing over intermediate states $\mathbf{x}_1, \dots, \mathbf{x}_T$, we get the distribution of terminal state $\mathbf{x}_0$:
\begin{align*}
\hat{q}(\mathbf{x}_0) &= \int \hat{q}(\mathbf{x}_T, \mathbf{x}_{T-1}, \dots, \mathbf{x}_0) \, d\mathbf{x}_T d\mathbf{x}_{T-1} \dots d\mathbf{x}_1 \\
&= \frac{\exp \big(\tau V(\mathbf{x}_0)\big)}{\exp \big(\tau V(\mathbf{x}_{T+1})\big)} \int p_{\text{BSIT}}(\mathbf{x}_T, \mathbf{x}_{T-1}, \dots, \mathbf{x}_0) \, d\mathbf{x}_T d\mathbf{x}_{T-1} \dots d\mathbf{x}_1 \\
&= \frac{\exp \big(\tau V(\mathbf{x}_0)\big)}{\exp \big(\tau V(\mathbf{x}_{T+1})\big)} p_{\text{BSIT}}(\mathbf{x}_0).
\end{align*}

By definition, the soft value function at the terminal node is $V(\mathbf{x}_0) = r(\mathbf{x}_0)$. Plugging this and using the definition of value function from Equation~\ref{eq:bellman-app}, we have:
\begin{align*}
\hat{q}(\mathbf{x}_0) &= \frac{\exp \big(\tau r(\mathbf{x}_0)\big) p_{\text{BSIT}}(\mathbf{x}_0)}{\int p_{\text{BSIT}}(\mathbf{x}_T, \mathbf{x}_{T-1}, \dots, \mathbf{x}_0 \mid \mathbf{x}_{T+1}) \exp \big(\lambda r(\mathbf{x}_0)\big) \, d\mathbf{x}_T d\mathbf{x}_{T-1} \dots d\mathbf{x}_1 d\mathbf{x}_0} \\
&= \frac{\exp \big(\tau r(\mathbf{x}_0)\big) p_{\text{BSIT}}(\mathbf{x}_0)}{\int p_{\text{BSIT}}(\mathbf{x}_0) \exp \big(\tau r(\mathbf{x}_0)\big) \, d\mathbf{x}_0}.
\end{align*}

This has the form of the target distribution as defined in Proposition 3, except that it uses the value estimates $V$ that are calculated based on rollouts starting from each state $\mathbf{x}_t$. In the limit of infinite rollouts, these value estimates approach the true soft values, confirming that the sampling distribution $\hat{q}$ from BSIT exactly matches the target energy-based policy $\pi^*$ defined in the statement of Proposition 3.
Therefore, BSIT is consistent, as it correctly generates samples from the desired target distribution asymptotically.
\end{proof}

\clearpage

\section{Details of Experimental Setting}

\paragraph{Additional Details on 2D Experimental Setting}

For a fair comparison, we use an identical model architecture across all generative models used in baselines. Specifically, in DTS, we adopt a 3-layer MLP with a hidden dimension of 128 for the diffusion model. The model is trained in the DDPM regime. In BSIT, BFMT, and the Flow-Map baseline, we use the same 3-layer, width-128 MLP for the flow-map model, trained with the LSD loss~\citep{boffi2025build}. In the tree setting, for DTS, we use 100 denoising steps, but restrict child-node expansion to 5 timesteps distributed evenly along the trajectory. For BSIT and BFMT, we correspondingly set the maximum tree depth to 5, where the terminal node lies at depth 0.

Beyond the description in the main paper, we detail the reward function here. To keep it 1-Lipschitz continuous while inducing a sparse, multi-modal structure, we define it as follows:

\[
r_{\text{sparse}} = \operatorname{clip}\!\left(4 \cdot \max_{i}\; w_i\left(r_i - \sqrt{(x - c_{x,i})^2 + (y - c_{y,i})^2}\right),\; -1,\; 1\right)
\]

\paragraph{Additional Details on Imagenet Experimental Setting}
For the ImageNet experiments, all flow-map-based baselines (BSIT, BFMT, and MFM-Base) share a single Meta Flow-Map model ~\citep{potaptchik2026meta} trained on ImageNet. For diffusion model based baselines (DAS, FKS and DTS), we use the state of the art pretrained diffusion model on ImageNet ~\cite{yu2025pixeldit}.
For tree structure construction, we adopt a 5-level tree (with terminal level $0$) in BSIT, BFMT, and FMET. For all SMC sampling methods (DAS, FKS) and the Euler Tree sampling method, we fix 5 resampling steps and increase the number of samples per step as the sampling budget grows. 
For BSIT and BFMT in the ImageNet experiments, owing to MFM's meta nature, we compute the instantaneous velocity each time a node is revisited; therefore, we do not need to sample a new $\epsilon^{\prime}$ upon every revisit.
We evaluate all competitive approaches across two settings. In the standard reward setting, the fine-grained target prompt contains standard features within the support of the prior model. In the sparse reward setting, the target prompt contains novel or rare attributes relative to the prior model's support. In Table~\ref{tb:st-a} below, we outline example class and prompt details for both settings.
\begin{table}[h]
  \centering
  \begin{minipage}[t]{0.49\columnwidth}
    \centering
    {\scriptsize\bfseries Fine-grained Prompt Under Standard Reward Setting\par}
    \vspace{2pt}
    {\scriptsize
    \begin{tabular}{|p{0.25\linewidth}|p{0.61\linewidth}|}
      \hline
      \textbf{Target} & \textbf{Fine-Grained Prompt} \\
      \hline
      Sports Car & A Red Sports car. \\
      \hline
      Cat & A Black Cat. \\
      \hline
      Dog & A Dog in the Snow. \\
      \hline
      Pepper & A Green Bell Pepper. \\
      \hline
      Daisy & A Red Daisy. \\
      \hline
    \end{tabular}%
    }
  \end{minipage}%
  \hfill
  \begin{minipage}[t]{0.49\columnwidth}
    \centering
    {\scriptsize\bfseries Fine-grained Prompt Under Sparse Reward Setting\par}
    \vspace{2pt}
    {\scriptsize
    \begin{tabular}{|p{0.25\linewidth}|p{0.61\linewidth}|}
      \hline
      \textbf{Target} & \textbf{Fine-Grained Prompt} \\
      \hline
      Goldfish & A White Fish. \\
      \hline
      Cat & Three Cats. \\
      \hline
      Dog & Two Dogs in the Snow. \\
      \hline
      Parrot & Three Parrots. \\
      \hline
      Zibra & Three Zibras. \\
      \hline
    \end{tabular}%
    }
  \end{minipage}
  \caption{Target classes and their corresponding fine-grained prompts for both standard (left) and sparse reward (right) settings.}
  \label{tb:st-a}
\end{table}

\paragraph{Additional Details on Rare Species Discovery Setting}
We used the same model architecture as in the 2D checkerboard setting. The flow-map model is trained on a joint distribution of four species: \textit{Aphelocoma wollweberi}, \textit{Aphelocoma woodhouseii}, \textit{Cyanocitta stelleri}, and \textit{Gymnorhinus cyanocephalus}, over the region spanning $31^\circ$ to $37^\circ$ latitude and $-113^\circ$ to $-107^\circ$ longitude. The training phase uses the Lagrangian Self-Distillation (LSD) loss~\citep{boffi2025build}. For tree structure construction, we adopt a 5-level tree. Querying a location yields a binary reward: 1 if the target species is present within a specified radius of 0.015 of the side length, and 0 otherwise.

\paragraph{Compute Resources}
All sampling experiments are conducted on a single NVIDIA A100 (80GB) GPU.

\section{Additional Quantitative and Qualitative Results in Rare Species Discovery Task}
In the main paper (Table 4 and Figure 5), we evaluated the search performance of BSIT on the rare species discovery task using species "Aphelocoma woodhouseii" as the target, comparing its performance both quantitatively and qualitatively against competitive baselines. To demonstrate the generalizability of BSIT in this practical setting, we extend these experiments to a different target rare species (specifically "Gymnorhinus cyanocephalus") under identical feedback budget constraints. We report the quantitative results in Table~\ref{fig:species-comparison-app} (left) and qualitative comparisons in Figure~\ref{fig:species-comparison-app} (right). Consistent with our previous findings, BSIT achieves superior performance across target species, confirming its effectiveness and adaptability across diverse rare-search tasks. 

\begin{figure}[h!]
\centering
\begin{minipage}[h]{0.53\columnwidth}
  \centering
  \vspace*{11pt}
  {\scriptsize\bfseries Mean Reward Comparison\par}
  \vspace{1pt}
  \resizebox{\linewidth}{!}{%
    \begin{tabular}{lccc|ccc}
      \hline\hline
      Method & \multicolumn{3}{c|}{NFE} &
      \multicolumn{3}{c}{Feedback Budget} \\
       & 1000 & 3000 & 5000 & 1000 & 3000 & 5000 \\
      \hline
      DTS  & 0.0783 & 0.1269 & 0.1250 & 0.1603 & 0.1673 & 0.1708 \\
      BFMT & 0.1595 & 0.1658 & 0.1731 & 0.1595 & 0.1658 & 0.1731 \\
      \rowcolor{blue!20} BSIT & 0.1870 & 0.1953 & 0.1968 & 0.1870 & 0.1953 & 0.1968 \\
      \hline\hline
    \end{tabular}%
  }
\end{minipage}%
\hspace{0.005\columnwidth}%
\begin{minipage}[h]{0.46\columnwidth}
  \centering
  {\scriptsize\itshape \textcolor{purple}{Gymnorhinus cyanocephalus} as target species\par}
  \vspace{1pt}
  {\setlength{\tabcolsep}{0.3pt}%
  \begin{tabular}{@{}ccc@{}}
    \textbf{Target} & \textbf{BFMT} & \textbf{BSIT} \\
    \hline
    \includegraphics[width=0.32\linewidth]{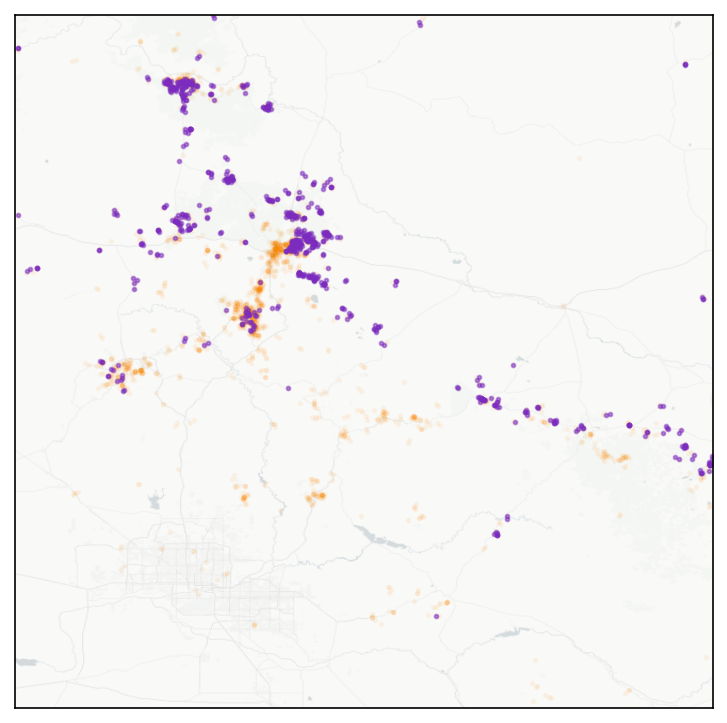} &
    \includegraphics[width=0.32\linewidth]{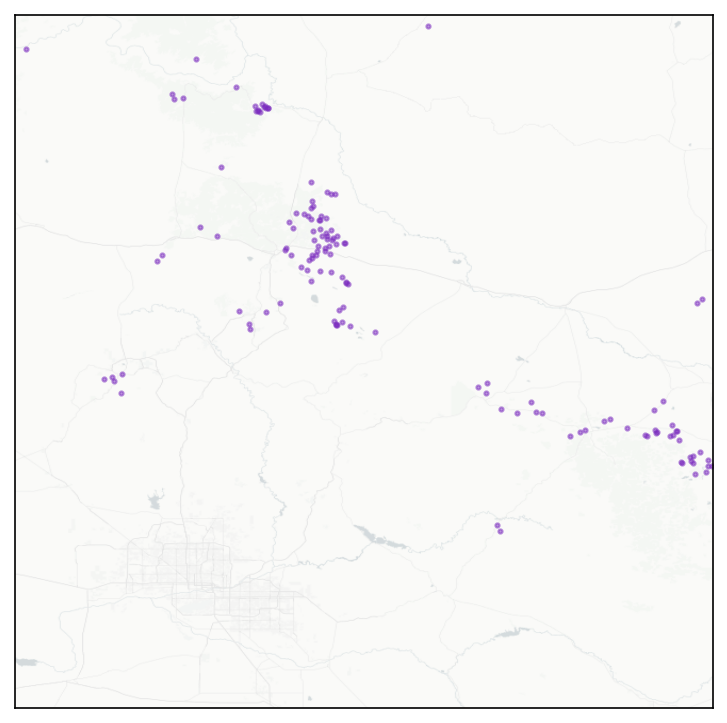} &
    \includegraphics[width=0.32\linewidth]{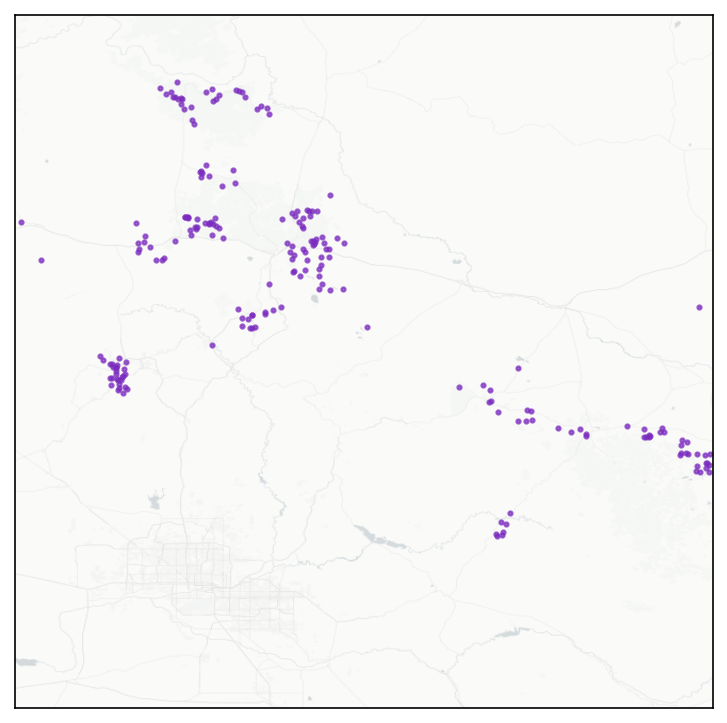}
  \end{tabular}%
  }
\end{minipage}
\caption{Mean-reward results and comparative visualizations for
``Gymnorhinus cyanocephalus'' as the target species.}
\label{fig:species-comparison-app}
\end{figure}

\section{Additional Quantitative Results with Sparse Reward on Imagenet}
Figure~\ref{fig:search-result-fb} presents additional results on ImageNet under sparse rewards for larger feedback budget constraints. While Table~3 in the main paper reports performance across a narrower budget range, the empirical outcomes in Figure~\ref{fig:search-result-fb}—alongside Table~3—demonstrate that BSIT consistently outperforms competitive baselines across the full spectrum of evaluation budgets, from highly restricted to abundant. This highlights the effectiveness of BSIT for feedback-efficient sampling in sparse reward regimes.

\begin{comment}
\newcommand{\legenditem}[4]{%
  \begin{tikzpicture}[baseline=-0.6ex]
    \draw[#1, line width=1.5pt, #2] (0,0) -- (0.48,0);
    \node[text=#1, fill=white, inner sep=0.4pt] at (0.24,0) {\tiny $#3$};
  \end{tikzpicture}\,#4%
}
\newcommand{\methodlegend}{%
  {\scriptsize
  \setlength{\tabcolsep}{0.5pt}%
  \begin{tabular}{@{}cccc@{}}
    %\legenditem{cyan!70!black}{solid}{\circ}{DAS} &
    %\legenditem{red!75!black}{dashed}{\square}{FKS} &
    %\legenditem{purple!80!black}{dotted}{\diamond}{FEMT} &
    \legenditem{orange!85!black}{dash dot}{\star}{BFMT} &
    \legenditem{green!60!black}{solid}{\blacktriangledown}{\textbf{BSIT}}
    
  \end{tabular}%
  }%
}
\newcommand{\barpanel}[2]{%
  \begin{minipage}[h]{0.99\linewidth}
    \centering
    \includegraphics[width=\linewidth,height=0.16\textheight,keepaspectratio]{#1}\par
    \vspace{-0.3ex}
    {\tiny #2}
  \end{minipage}%
}
\end{comment}
\begin{figure}[h!]
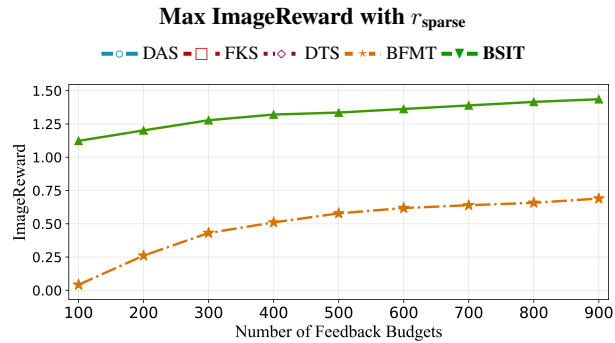

  \centering
  {\footnotesize\bfseries Max ImageReward with $r_{\text{sparse}}$\par}
  \vspace{0.4ex}
  \methodlegend\par
  \vspace{0.4ex}
  \barpanel{Figures/max_sparse.png}{}
  \caption{Search performance comparison across varying feedback budgets using $r_{\text{sparse}}$ as the reward function on ImageNet.}
  \label{fig:search-result-fb}
\end{figure}

\section{A Qualitative Analysis of BSIT's Enhanced Exploration Capability under Sparse Reward Setting}

Figure~\ref{fig:spiral-search-trajectory} illustrates the temporal evolution of BSIT’s search strategy under extreme reward sparsity. In the initial phases, BSIT engages in broad, unconstrained exploration, spanning semantically diverse regions within the target category prior to grounding fine-grained rare target attributes. As feedback accumulates, the search smoothly transitions into a targeted exploitation regime—progressively refining samples until they precisely align with the rare-attribute prompt. This distinct explore-then-exploit dynamic underscores the framework's superior path-space exploration, driving high feedback efficiency in constrained sparse-reward settings. Additional qualitative evolutions across different target categories are provided below.

\clearpage

\begin{figure*}[h]
\centering
\newcommand{\gridimage}[1]{%
  \IfFileExists{#1}{%
    \includegraphics[width=0.16\textwidth,height=0.16\textwidth,keepaspectratio]{#1}%
  }{%
    \fbox{\begin{minipage}[c][0.16\textwidth][c]{0.16\textwidth}\centering\scriptsize Missing image\end{minipage}}%
  }%
}
\newcommand{\feedbackarrow}[1]{%
  \begin{tikzpicture}[x=\linewidth,y=1cm]
    \draw[blue!35!black,line width=1.2pt,-{Stealth[length=2.4mm]}] (0.13,0) -- (0.87,0);
    \node[fill=white,inner sep=1.2pt,font=\scriptsize\bfseries] at (0.50,0) {#1};
  \end{tikzpicture}%
}
% Edit the prompt and image paths for the LEFT panel here.
\newcommand{\leftsearchpanel}{%
  {\small\itshape Sparse Prompt: ``\textcolor{purple}{Three cats.}''\par}
  \vspace{0.8ex}
  {\setlength{\tabcolsep}{7pt}%
  \renewcommand{\arraystretch}{1.0}%
  \begin{tabular}{@{}cccc@{}}
    \gridimage{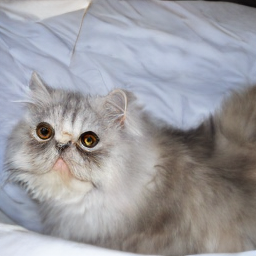} &
    \gridimage{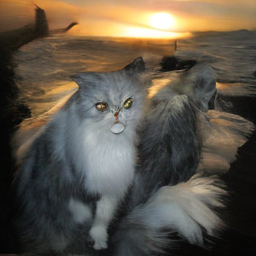} &
    \gridimage{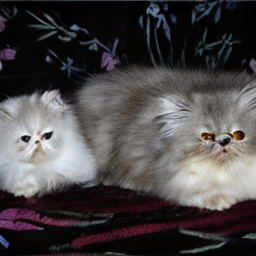} &
    \gridimage{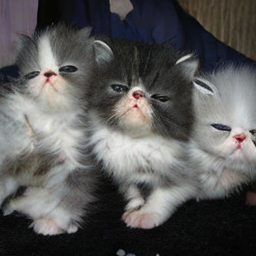} \\
    \multicolumn{4}{@{}c@{}}{\feedbackarrow{Initial Feedback Steps}} \\[0.8ex]
    \gridimage{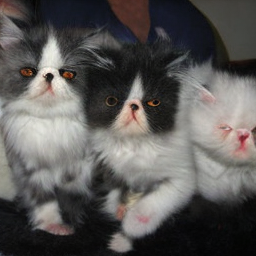} &
    \gridimage{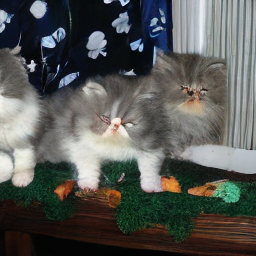} &
    \gridimage{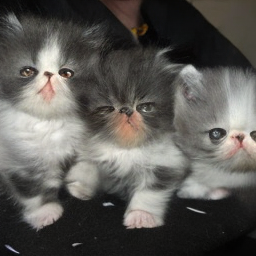} &
    \gridimage{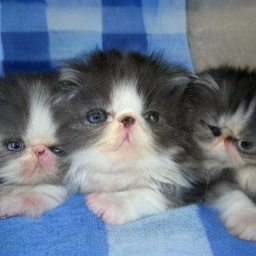} \\
    \multicolumn{4}{@{}c@{}}{\feedbackarrow{Early-Intermediate Feedback Steps}} \\[0.8ex]
    \gridimage{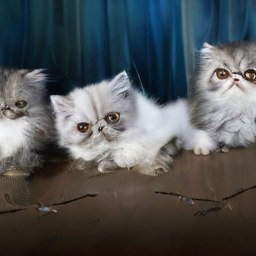} &
    \gridimage{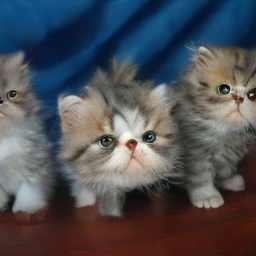} &
    \gridimage{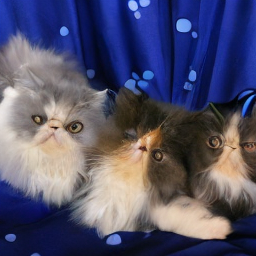} &
    \gridimage{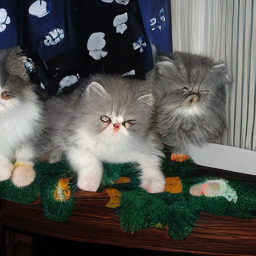} \\
    \multicolumn{4}{@{}c@{}}{\feedbackarrow{Late-Intermediate Feedback Steps}} \\[0.8ex]
    \gridimage{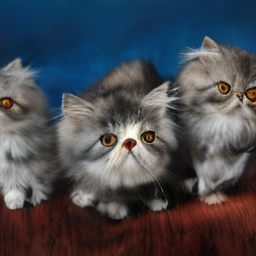} &
    \gridimage{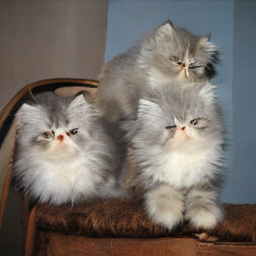} &
    \gridimage{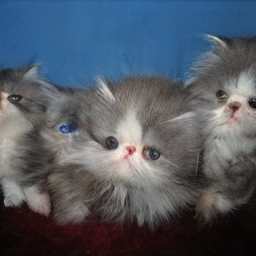} &
    \gridimage{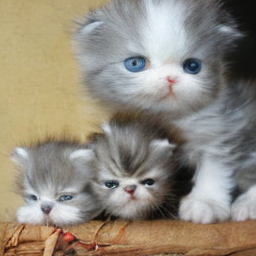} \\
    \multicolumn{4}{@{}c@{}}{\feedbackarrow{Final Feedback Steps}}
  \end{tabular}%
  }%
}

% Edit the prompt and image paths for the RIGHT panel here.
\newcommand{\rightsearchpanel}{%
  {\small\itshape Sparse Prompt: ``\textcolor{purple}{Three parrots.}''\par}
  \vspace{0.8ex}
  {\setlength{\tabcolsep}{7pt}%
  \renewcommand{\arraystretch}{1.0}%
  \begin{tabular}{@{}cccc@{}}
    \gridimage{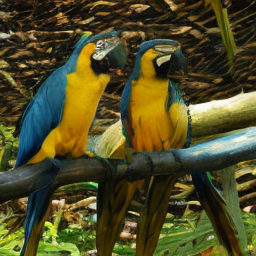} &
    \gridimage{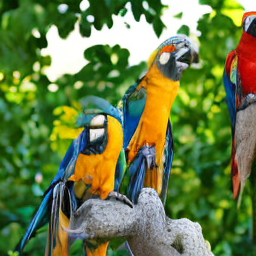} &
    \gridimage{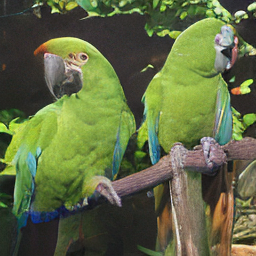} &
    \gridimage{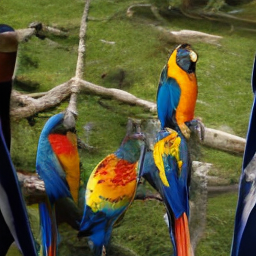} \\
    \multicolumn{4}{@{}c@{}}{\feedbackarrow{Initial Feedback Steps}} \\[0.8ex]
    \gridimage{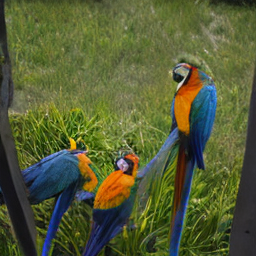} &
    \gridimage{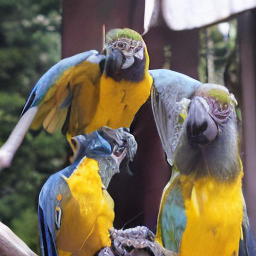} &
    \gridimage{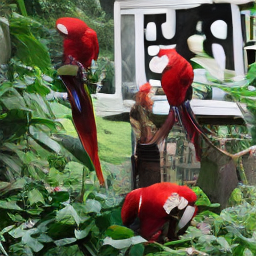} &
    \gridimage{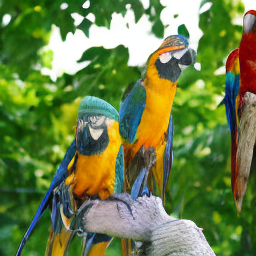} \\
    \multicolumn{4}{@{}c@{}}{\feedbackarrow{Early-Intermediate Feedback Steps}} \\[0.8ex]
    \gridimage{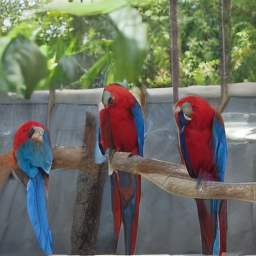} &
    \gridimage{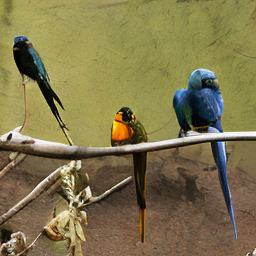} &
    \gridimage{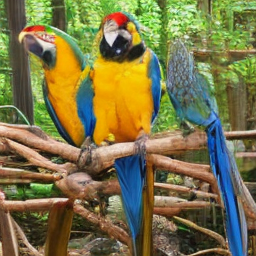} &
    \gridimage{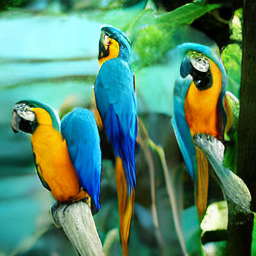} \\
    \multicolumn{4}{@{}c@{}}{\feedbackarrow{Late-Intermediate Feedback Steps}} \\[0.8ex]
    \gridimage{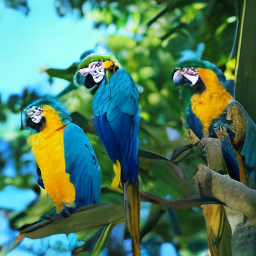} &
    \gridimage{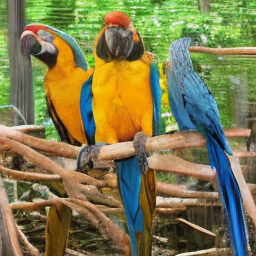} &
    \gridimage{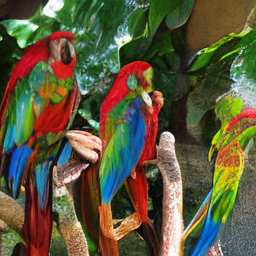} &
    \gridimage{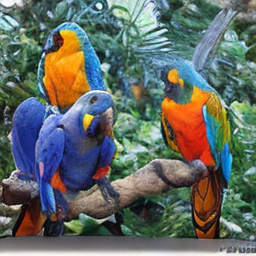} \\
    \multicolumn{4}{@{}c@{}}{\feedbackarrow{Final Feedback Steps}}
  \end{tabular}%
  }%
}

{\setlength{\tabcolsep}{0pt}%
\setlength{\arrayrulewidth}{0.5pt}%
\begin{tabular}{@{}c|c@{}}
  \resizebox{0.498\textwidth}{!}{%
    \begin{minipage}[t]{0.75\textwidth}
      \centering
      \leftsearchpanel
    \end{minipage}%
  }
  &
  \resizebox{0.498\textwidth}{!}{%
    \begin{minipage}[t]{0.75\textwidth}
      \centering
      \rightsearchpanel
    \end{minipage}%
  }
\end{tabular}%
}
\vspace{-2pt}
\caption{\small{Exploration Strategy of BSIT under Sparse Reward Setting.}}
\label{fig:spiral-search-trajectory}
\end{figure*}

\section{Comparative Visualizations of Rare Species Discovery Pattern over a Larger Geo-spatial Map}
Figure~\ref{fig:search-result-large-geo} visualizes the set of target samples (species "Aphelocoma woodhouseii") discovered by various competitive approaches across a broad geospatial region. These additional spatial map plots demonstrate that BSIT achieves superior coverage of the ground-truth target distribution for the rare species "Aphelocoma woodhouseii".
\begin{figure}[h!]
  \centering
  %{\footnotesize\bfseries Mean ImageReward with $r_{\text{sparse}}$\par}
  \vspace{0.4ex}
  %\methodlegend\par
  \vspace{0.4ex}
  \centering
  {\itshape \textcolor{purple}{Aphelocoma woodhouseii} as target species\par}
  \vspace{1pt}
  {\setlength{\tabcolsep}{0.3pt}%
  \begin{tabular}{@{}ccc@{}}
    \textbf{Target} & \textbf{BFMT} & \textbf{BSIT} \\
    \hline
    \includegraphics[width=0.32\linewidth]{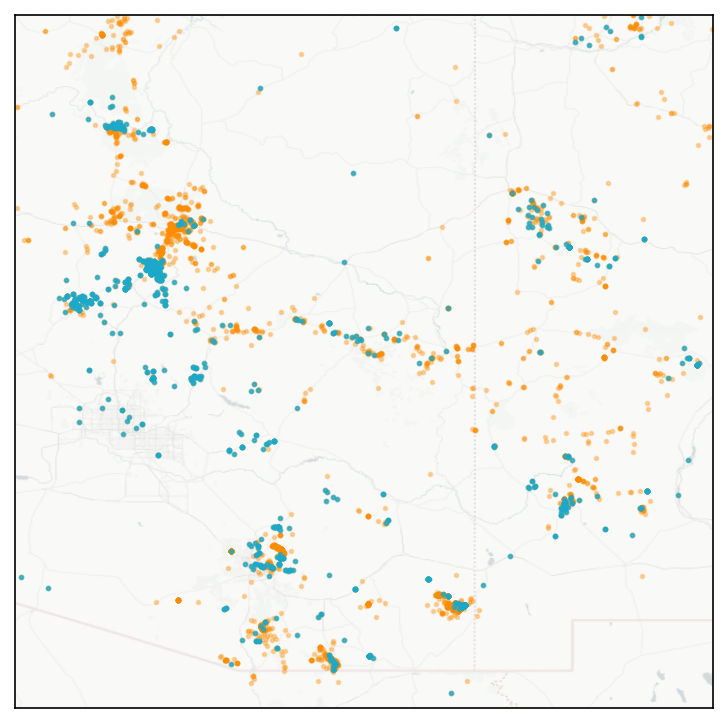} &
    \includegraphics[width=0.32\linewidth]{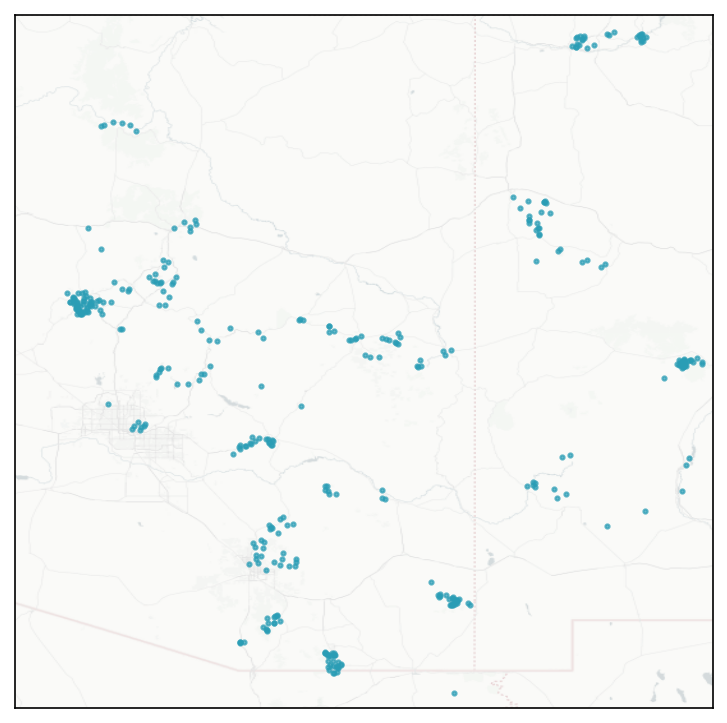} &
    \includegraphics[width=0.32\linewidth]{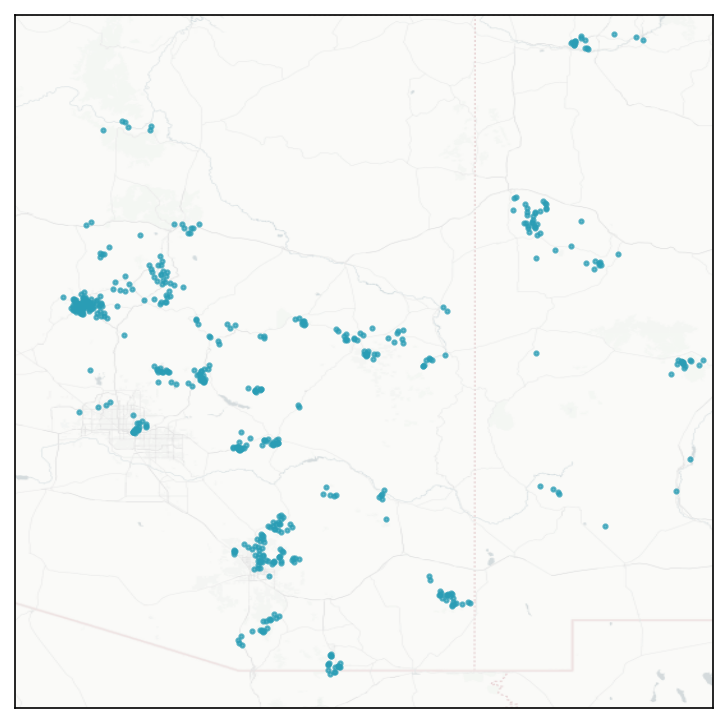}
  \end{tabular}%
  }
  \caption{Aphelocoma woodhouseii as the target species over a large geo-spatial map.}
  \label{fig:search-result-large-geo}
\end{figure}

%Figure 9 visualizes the evolution of BSIT’s search strategy across different feedback stages under a sparse reward setting. In the early stages, BSIT explores broadly, producing samples with semantically diverse features within the target category, albeit without precise alignment to the fine-grained target prompt with rare features. As feedback accumulates and the search matures, BSIT progressively converges toward samples that closely match the target prompt — reflecting BSIT's effective exploration capability, resulting in an efficient search process. This exploration-then-exploitation dynamic underscores BFMT’s strong exploratory capability, which is fundamental to its success in feedback-constrained sparse reward settings. We provide additional visualizations with other target categories with rare features below.

\section{ Additional Visualizations of Competitive Approaches on Sparse Reward Settings}
To highlight performance under sparse rewards, Figure~\ref{fig:sparse-visu-add} compares BSIT against baseline samplers on ImageNet when targeting objectives with rare or novel features. Aligning with our quantitative results, BSIT excels at generating semantically diverse, target-aligned samples with rare/novel features within a limited feedback budget, validating its strong exploratory capabilities under sparse reward conditions. 

\begin{figure*}[h]
\centering
\newcommand{\triplepanel}[3]{%
  \setlength{\tabcolsep}{2pt}%
  \begin{tabular}{@{}c@{\hspace{4pt}}c@{\hspace{4pt}}c@{}}
    \includegraphics[width=0.315\linewidth]{#1} &
    \includegraphics[width=0.315\linewidth]{#2} &
    \includegraphics[width=0.315\linewidth]{#3}
  \end{tabular}%
}
\newcommand{\panelblock}[4]{%
  \begin{minipage}[t]{\linewidth}
    \centering
    {\itshape ``#1''\par}
    \vspace{0pt}
    \triplepanel{#2}{#3}{#4}
  \end{minipage}%
}
\newcommand{\topheaders}{%
  {\fontsize{9.5}{10.5}\selectfont
  \begin{tabular}{@{}c@{\hspace{4pt}}c@{\hspace{4pt}}c@{}}
    \makebox[0.315\linewidth][c]{\textbf{FMET}} &
    \makebox[0.315\linewidth][c]{\textbf{BFMT}} &
    \makebox[0.315\linewidth][c]{\textbf{BSIT}}
  \end{tabular}%
  }
}
\noindent
\topheaders
\vspace{-0.4ex}
\noindent\rule{\textwidth}{0.6pt}
\vspace{-0.8ex}
\begin{minipage}[t]{\textwidth}
\centering
\panelblock{Sparse Fine-grained Prompt: \textcolor{purple}{A white fish.}}{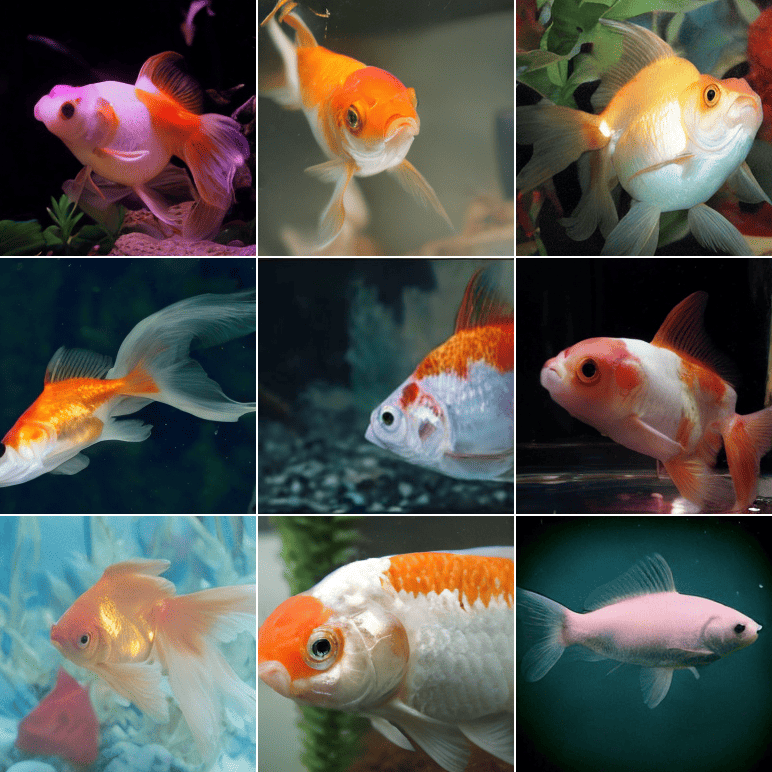}{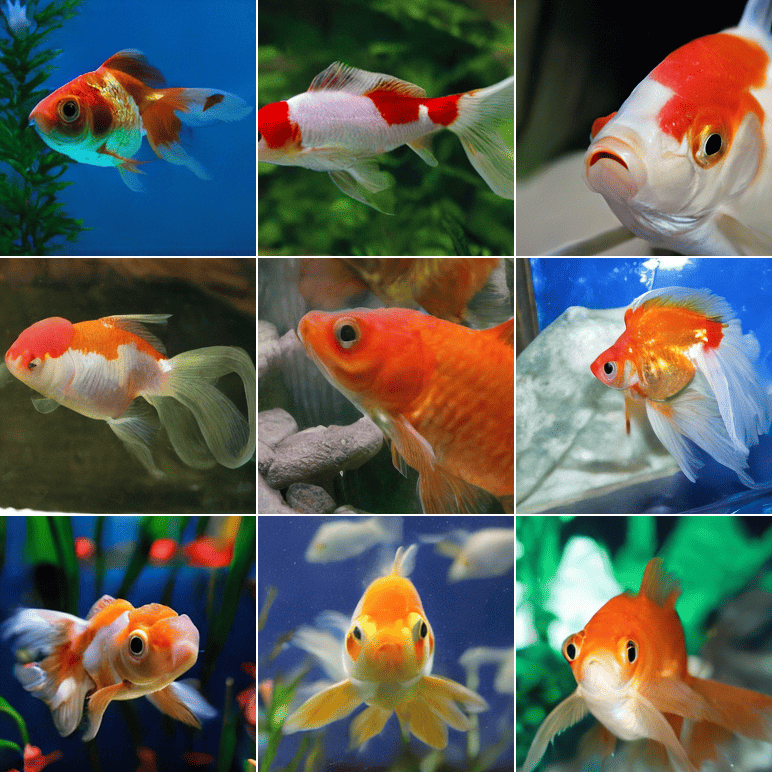}{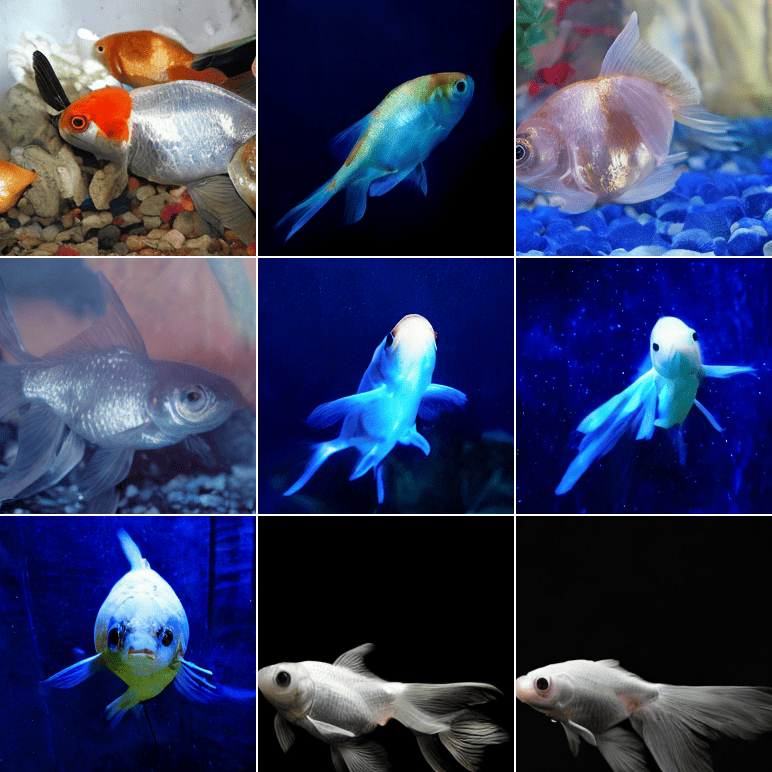}
\end{minipage}
\vspace{0.8ex}
\noindent\rule{\textwidth}{0.6pt}
\vspace{-0.4ex}
\topheaders
\vspace{-0.4ex}
\noindent\rule{\textwidth}{0.6pt}
\vspace{-0.8ex}
\begin{minipage}[t]{\textwidth}
\centering
\panelblock{Sparse Fine-grained Prompt: \textcolor{purple}{Two dogs in the snow.}}{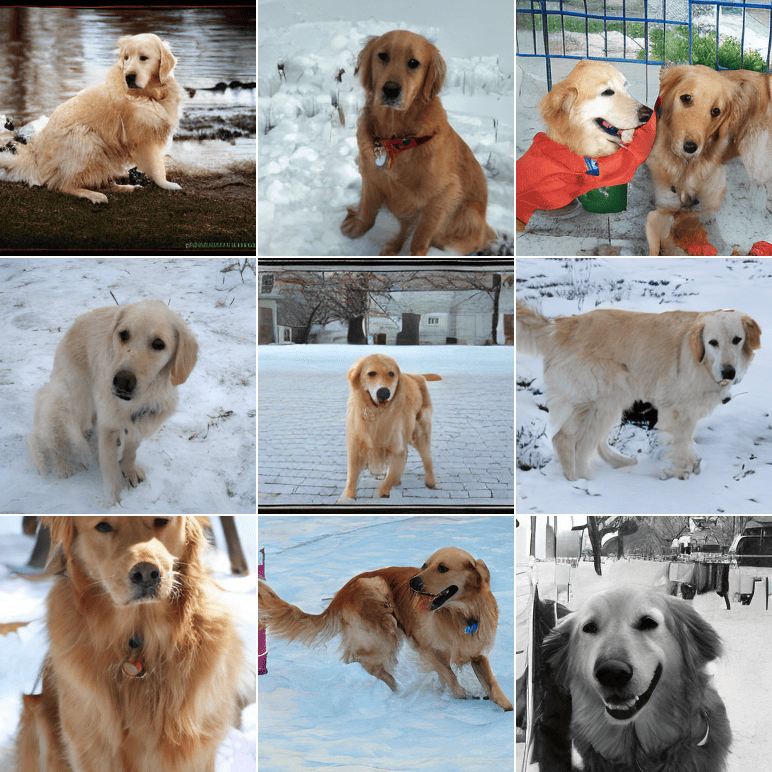}{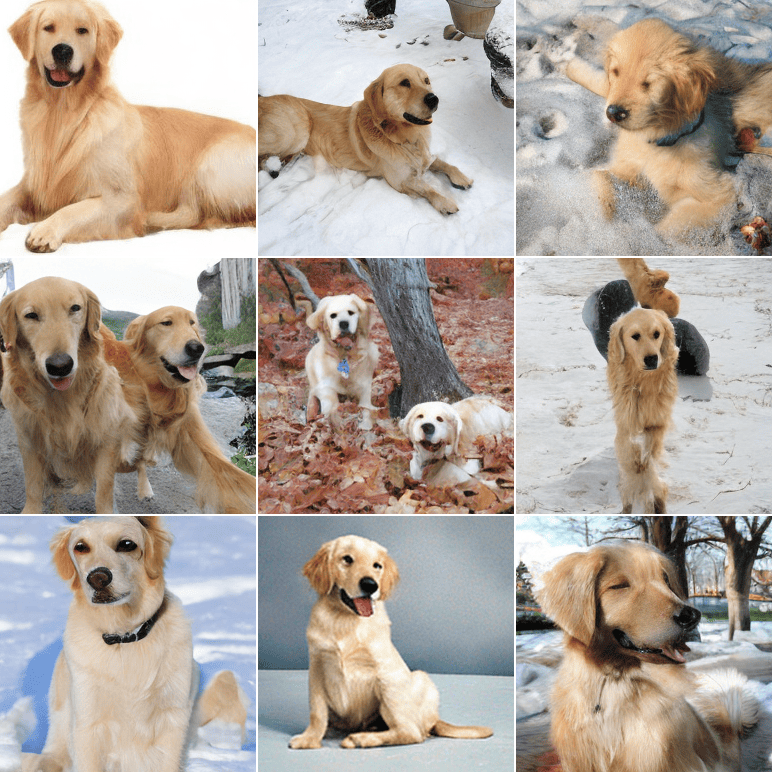}{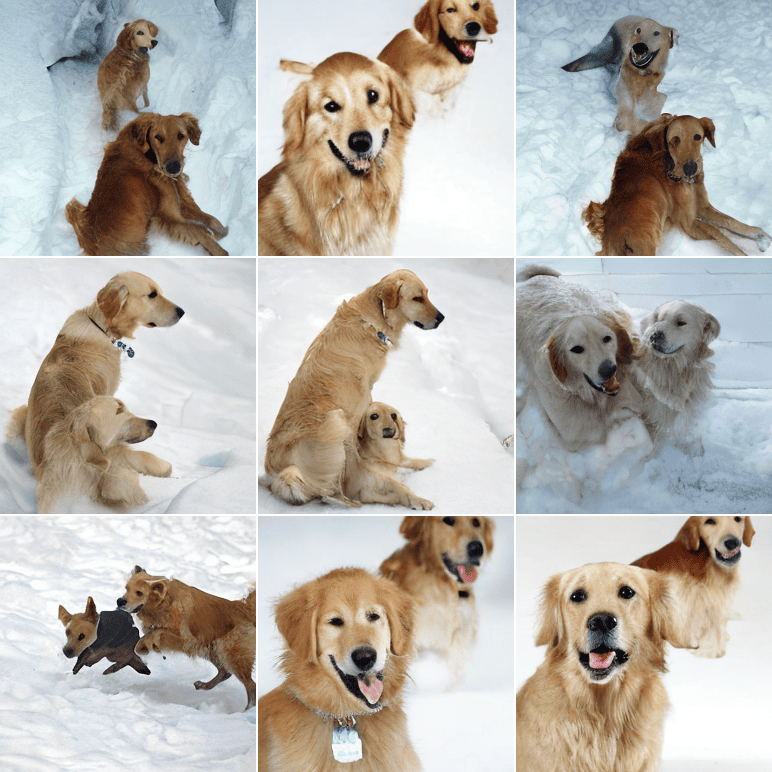}
\end{minipage}%
\vspace{0.6ex}
\noindent\rule{\textwidth}{0.6pt}
\caption{\small{Additional Search Visualizations with sparse fine-grained prompt.}}
\label{fig:sparse-visu-add}
%\vspace{-18pt}
\end{figure*}
\vspace{-10pt}
\section{Additional Results with Molecule Generation Task}\label{sec:mole-addq}
\vspace{-6pt}
Tables~\ref{tab:mole_K} and~\ref{tab:mole_S} report the mean validity, uniqueness, QED, and composite QED and SA scores for samples generated under feedback budgets 200 and 600, respectively. These results further highlight BSIT's ability to efficiently generate valid, target-aligned molecules even under tightly constrained feedback budgets.
\begin{table}[!h]
\centering
\caption{Validity Check for the Generated Molecules.}
\label{tab:mole_K}
\setlength{\tabcolsep}{5pt}
\renewcommand{\arraystretch}{0.90}
\begin{tabular}{lccc|ccc}
\toprule
Method & Uniqueness & Validity & Mean QED & Uniqueness & Validity & Mean QED+SA \\
\midrule
CFM & 0.9993 & 0.9410 & 0.4876 & 0.9993 & 0.9410 & 0.5148\\
%BFMT & 0.9823 & 0.9620 & 0.5489 & 0.9778 & 0.9500 & 0.5627\\
\rowcolor{blue!20}BSIT & 0.9674 & 0.9540 & \textbf{0.5482} & 0.9695 & 	0.9520 & \textbf{0.5642}\\
\bottomrule
\end{tabular}
\end{table}
\vspace{-16pt}
\begin{table}[!h]
\centering
\caption{Validity Check for the Generated Molecules.}
\label{tab:mole_S}
\setlength{\tabcolsep}{5pt}
\renewcommand{\arraystretch}{0.90}
\begin{tabular}{lccc|ccc}
\toprule
Method & Uniqueness & Validity & Mean QED & Uniqueness & Validity & Mean QED+SA \\
\midrule
CFM & 0.9989 & 0.9410 & 0.4925 & 0.9989 & 0.9410 & 0.5202\\
%BFMT & 0.9774 & 0.9579 & 0.5575 & 0.9766 & 0.9588 & 0.5702\\
\rowcolor{blue!20}BSIT & 0.9603 & 0.9573 & \textbf{0.5528} & 0.9543 & 0.9567 & \textbf{0.5762}\\
\bottomrule
\end{tabular}
\end{table}

\section{ Additional Visualizations of Competitive Approaches on Standard Reward Setting}

In Figure~\ref{fig:standard-visu-add}, we present additional qualitative results on ImageNet operating under a fixed feedback budget constraint with standard fine-grained prompts as target objectives. By evaluating BSIT side-by-side with the most competitive baseline samplers, these visual comparisons further corroborate our quantitative findings: BSIT consistently demonstrates a superior capability to generate semantically diverse, fine-grained target-aligned samples even under a strictly limited feedback budget.

\begin{figure*}[h]
\centering
\newcommand{\triplepanel}[3]{%
  \setlength{\tabcolsep}{2pt}%
  \begin{tabular}{@{}c@{\hspace{4pt}}c@{\hspace{4pt}}c@{}}
    \includegraphics[width=0.315\linewidth]{#1} &
    \includegraphics[width=0.315\linewidth]{#2} &
    \includegraphics[width=0.315\linewidth]{#3}
  \end{tabular}%
}
\newcommand{\panelblock}[4]{%
  \begin{minipage}[t]{\linewidth}
    \centering
    {\itshape ``#1''\par}
    \vspace{0pt}
    \triplepanel{#2}{#3}{#4}
  \end{minipage}%
}
\newcommand{\topheaders}{%
  {\fontsize{9.5}{10.5}\selectfont
  \begin{tabular}{@{}c@{\hspace{4pt}}c@{\hspace{4pt}}c@{}}
    \makebox[0.315\linewidth][c]{\textbf{DTS}} &
    \makebox[0.315\linewidth][c]{\textbf{BFMT}} &
    \makebox[0.315\linewidth][c]{\textbf{BSIT}}
  \end{tabular}%
  }
}
\noindent
\topheaders
\vspace{-0.4ex}
\noindent\rule{\textwidth}{0.6pt}
\vspace{-0.8ex}
\begin{minipage}[t]{\textwidth}
\centering
\panelblock{Standard Fine-grained Prompt: \textcolor{purple}{A red car.}}{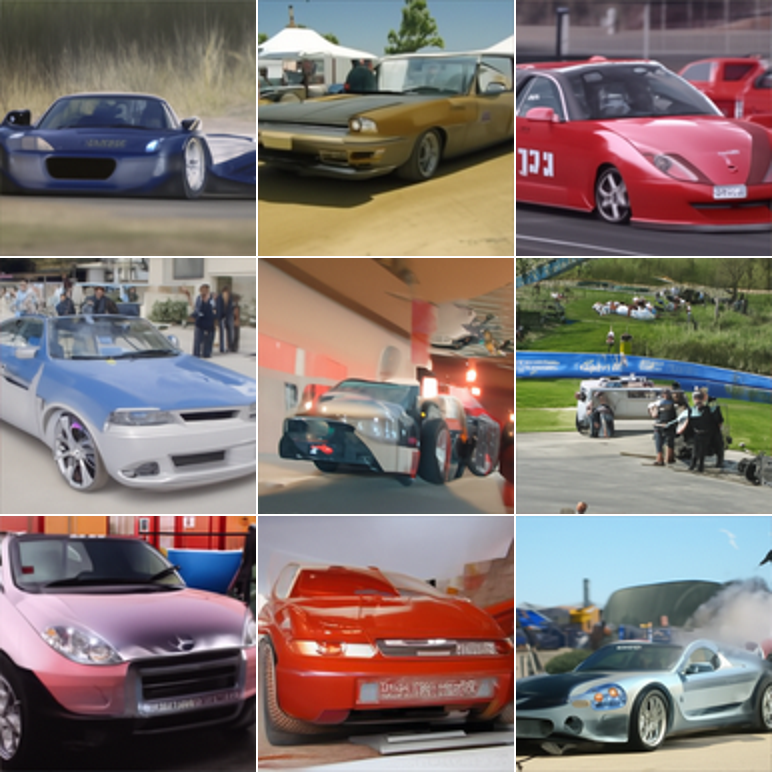}{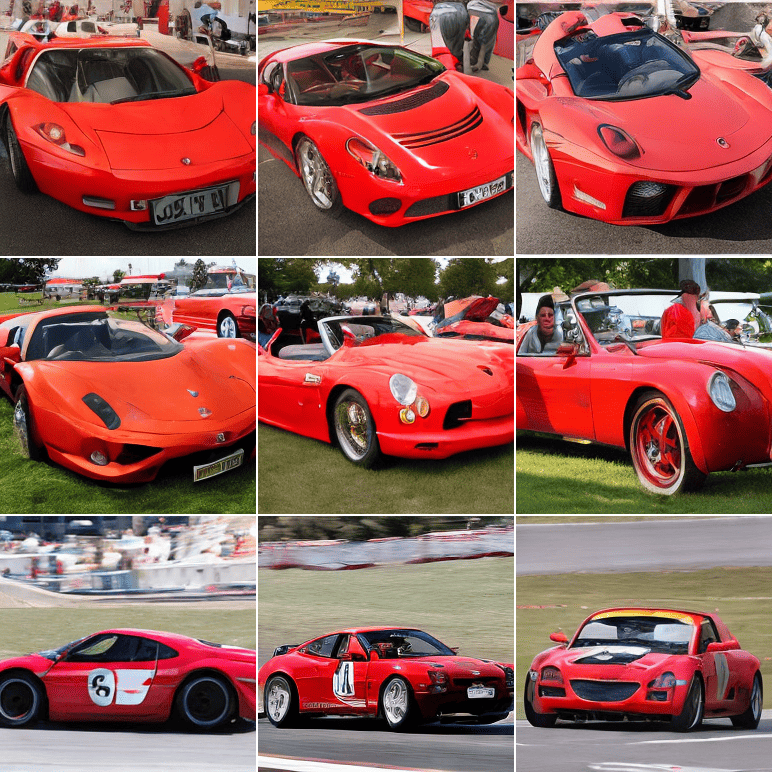}{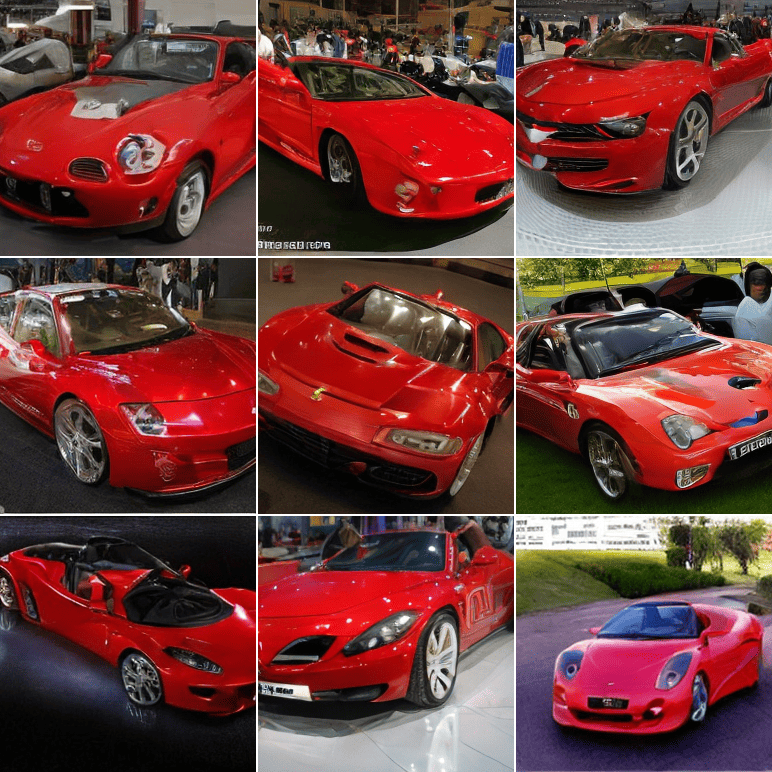}
\end{minipage}
\vspace{0.8ex}
\noindent\rule{\textwidth}{0.6pt}
\vspace{-0.4ex}
\topheaders
\vspace{-0.4ex}
\noindent\rule{\textwidth}{0.6pt}
\vspace{-0.8ex}
\begin{minipage}[t]{\textwidth}
\centering
\panelblock{Standard Fine-grained Prompt: \textcolor{purple}{A black cat.}}{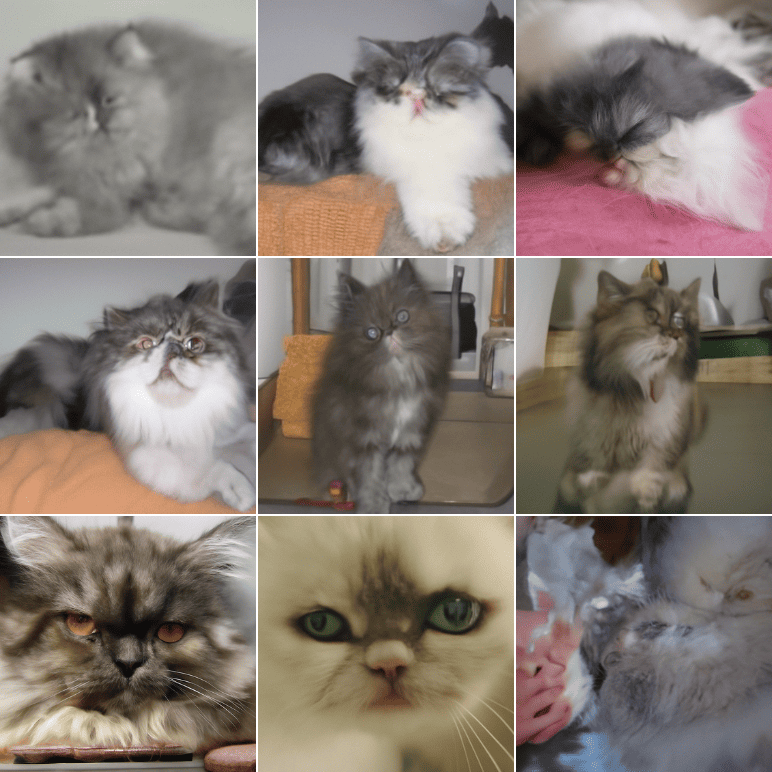}{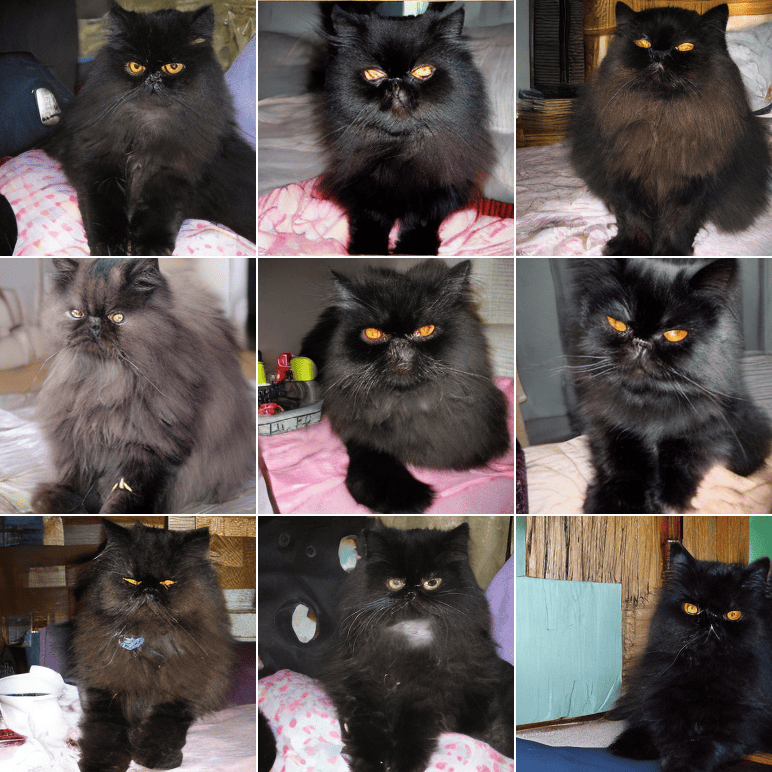}{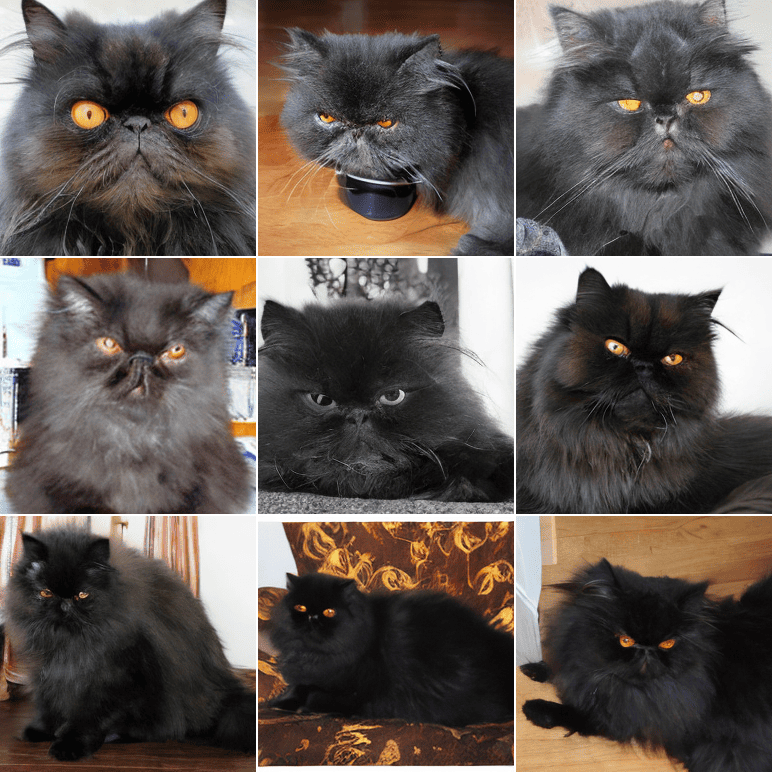}
\end{minipage}%
\vspace{0.6ex}
\noindent\rule{\textwidth}{0.6pt}
\caption{\small{Additional Search Visualizations with standard fine-grained prompt.}}
\label{fig:standard-visu-add}
%\vspace{-18pt}
\end{figure*}

\clearpage
\section{Pseudocode of BSIT sampler}
During tree traversal, if an intermediate node $x_t$ ($t > 0$) has not yet reached its branching capacity (i.e., $|\mathcal{C}(x_t)| < B(x_t)$), we actively expand it. Following \cite{couetoux2011continuous}, this dynamic maximum branching factor $B(x_t)$ scales with the parent's visit count: $B(x_t) = \lceil C \cdot N(x_t)^\alpha \rceil, \quad C > 0, \alpha \in (0, 1).$
Below, we present pseudocode for two variants of BSIT: one using a trained Flow-Map as the prior (see Algorithm~\ref{alg:bsit-fm}), and the other using a trained Meta-Flow-Map (see Algorithm~\ref{alg:bsit-mfm}).
\begin{algorithm}[H]
\caption{Bootstrap Stochastic Interpolant Tree Sampling (BSIT)}
\label{alg:bsit-fm}
\begin{algorithmic}[1]
\REQUIRE \textcolor{purple}{Pretrained Flow-Map velocity field $\hat{v}_\theta(x_t)$}, reward function $r$, number of tree iterations $M$, inverse temperature $\tau$, progressive widening parameters $C, \alpha$.
\STATE Initialize root node $x_{T+1}$ with dummy value, $V(x_{T+1}) \gets 0$, $N(x_{T+1}) \gets 1$
\STATE Initialize tree $\mathcal{T}$ with root node $x_{T+1}$

\FOR{$m = 1, \dots, M$}
    \STATE $P \gets \{x_{T+1}\}$
    \STATE $t \gets T+1$
    
    \STATE \textbf{// Selection (via UCT)}
    \WHILE{$|\mathcal{C}(x_t)| \ge C \cdot N(x_t)^\alpha$ \textbf{and} $t > 0$}
        \STATE Select child maximizing UCT score:
        \STATE $x_{t-1} \gets \arg\max_{x \in \mathcal{C}(x_t)} \left( V(x) + c \sqrt{\frac{\ln N(x_t)}{N(x)}} \right)$
        \STATE $P \gets P \cup \{x_{t-1}\}$
        \STATE $t \gets t - 1$
    \ENDWHILE
    
    \STATE \textbf{// Expansion \& Rollout}
    \IF{$t > 0$ \textbf{and} $|\mathcal{C}(x_t)| < C \cdot N(x_t)^\alpha$}
        \WHILE{$t > 0$}
            \IF{$N(x_t) = 0$}
                 \STATE \textbf{// \textcolor{purple}{Bootstrap Stochastic Interpolant Trajectory Construction}}
                 \STATE \textcolor{blue}{$  \textcolor{red}{x_0} \sim p_{\hat{v}_{\theta}(x_t)}(x_0 | x_t)$}
                 \WHILE{\textcolor{blue}{$t > 0$}}    
                      \STATE \textcolor{blue}{Stochastic Interpolant: $x_{t-1} = \alpha_{t-1} x_0 + \beta_{t-1} \epsilon_{t-1}$; \text{where,} $\epsilon_{t-1} \sim \mathcal{N}(0, I)$}
                      \STATE \textcolor{blue}{Initialize $V(x_{t-1}) \gets 0$, $N(x_{t-1}) \gets 1$}
                      \STATE \textcolor{blue}{$P \gets P \cup \{x_{t-1}\}$}
                      \STATE \textcolor{blue}{$t \gets t - 1$}
                 \ENDWHILE
            \ELSE
                \STATE \textbf{// \textcolor{purple}{Dynamic Node State Construction}}
                \STATE \textcolor{blue}{$\epsilon_t \sim \mathcal{N}(0, I) \quad \xrightarrow{\text{Interpolant}} \quad x^{\text{new}}_t =\alpha_{t} \textcolor{red}{x_0} + \beta_{t} \epsilon_{t}$}
                \STATE \textcolor{blue}{$x_0 \sim p_{\hat{v}_{\theta}(x^{\text{new}}_t)}(x_0 | x^{\text{new}}_t)$}
                \WHILE{\textcolor{blue}{$t > 0$}} 
                     \STATE \textcolor{blue}{Stochastic Interpolant: $x_{t-1} = \alpha_{t-1} x_0 + \beta_{t-1} \epsilon_{t-1}$; \text{where,} $\epsilon_{t-1} \sim \mathcal{N}(0, I)$}
                     \STATE \textcolor{blue}{Initialize $V(x_{t-1}) \gets 0$, $N(x_{t-1}) \gets 1$}
                     \STATE \textcolor{blue}{$P \gets P \cup \{x_{t-1}\}$}
                     \STATE \textcolor{blue}{$t \gets t - 1$}
                \ENDWHILE
            \ENDIF
        \ENDWHILE
    \ENDIF
    
    \STATE \textbf{// Evaluation \& Backup}
    \STATE Evaluate terminal reward: $V(x_0) \gets r(x_0)$
    \FOR{$t = 0, \dots, T$}
        \STATE Soft Bellman backup: 
        \STATE $V(x_{t+1}) \gets \frac{1}{\tau} \log \sum_{x_t \in \mathcal{C}(x_{t+1})} \exp(\tau V(x_t))$
        \STATE Update visit count: $N(x_{t+1}) \gets N(x_{t+1}) + 1$
    \ENDFOR
\ENDFOR

\RETURN Tree $\mathcal{T}$
\end{algorithmic}
\end{algorithm}

\begin{algorithm}[H]
\caption{Bootstrap Stochastic Interpolant Tree Sampling (BSIT)}
\label{alg:bsit-mfm}
\begin{algorithmic}[1]
\REQUIRE \textcolor{purple}{Pretrained Meta Flow-Map velocity field $\hat{v}_\theta(\epsilon;x_t)$}, reward function $r$, number of tree iterations $M$, inverse temperature $\tau$, progressive widening parameters $C, \alpha$.
\STATE Initialize root node $x_{T+1}$ with dummy value, $V(x_{T+1}) \gets 0$, $N(x_{T+1}) \gets 1$
\STATE Initialize tree $\mathcal{T}$ with root node $x_{T+1}$

\FOR{$m = 1, \dots, M$}
    \STATE $P \gets \{x_{T+1}\}$
    \STATE $t \gets T+1$
    
    \STATE \textbf{// Selection (via UCT)}
    \WHILE{$|\mathcal{C}(x_t)| \ge C \cdot N(x_t)^\alpha$ \textbf{and} $t > 0$}
        \STATE Select child maximizing UCT score:
        \STATE $x_{t-1} \gets \arg\max_{x \in \mathcal{C}(x_t)} \left( V(x) + c \sqrt{\frac{\ln N(x_t)}{N(x)}} \right)$
        \STATE $P \gets P \cup \{x_{t-1}\}$
        \STATE $t \gets t - 1$
    \ENDWHILE
    
    \STATE \textbf{// Expansion \& Rollout}
    \IF{$t > 0$ \textbf{and} $|\mathcal{C}(x_t)| < C \cdot N(x_t)^\alpha$}
        \WHILE{$t > 0$}
            %\IF{$N(x_t) = 0$}
                 \STATE \textbf{// \textcolor{purple}{Bootstrap Stochastic Interpolant Trajectory Construction}}
                 \STATE \textcolor{blue}{$\epsilon \sim \mathcal{N}(0, I) \quad \xrightarrow{\hat{v}_{\theta}(\epsilon;x_t)} \quad \textcolor{red}{x_0} \sim p_{\hat{v}_{\theta}(\epsilon;x_t)}(x_0 | x_t)$} 
                 \STATE \textcolor{blue}{Stochastic Interpolant: $x_{t-1} = \alpha_{t-1} x_0 + \beta_{t-1} \epsilon_{t-1}$; \text{where,} $\epsilon_{t-1} \sim \mathcal{N}(0, I)$}
                 \STATE \textcolor{blue}{Initialize $V(x_{t-1}) \gets 0$, $N(x_{t-1}) \gets 1$}
                 \STATE \textcolor{blue}{$P \gets P \cup \{x_{t-1}\}$}
                 \STATE \textcolor{blue}{$t \gets t - 1$}
            %\ELSE
                %\ENDWHILE
            %\ENDIF
        \ENDWHILE
    \ENDIF
    
    \STATE \textbf{// Evaluation \& Backup}
    \STATE Evaluate terminal reward: $V(x_0) \gets r(x_0)$
    \FOR{$t = 0, \dots, T$}
        \STATE Soft Bellman backup: 
        \STATE $V(x_{t+1}) \gets \frac{1}{\tau} \log \sum_{x_t \in \mathcal{C}(x_{t+1})} \exp(\tau V(x_t))$
        \STATE Update visit count: $N(x_{t+1}) \gets N(x_{t+1}) + 1$
    \ENDFOR
\ENDFOR

\RETURN Tree $\mathcal{T}$
\end{algorithmic}
\end{algorithm}

\clearpage
%\bibliography{aaai2027}
\end{document}